\documentclass{article}

\usepackage{amsmath,amsfonts,bm}

\def\1{\bm{1}}
\newcommand{\train}{\mathcal{D}_{\mathrm{train}}}
\newcommand{\valid}{\mathcal{D}_{\mathrm{valid}}}
\newcommand{\test}{\mathcal{D}_{\mathrm{test}}}

\def\vu{{\bm{u}}}
\def\vv{{\bm{v}}}
\def\vw{{\bm{w}}}
\def\vx{{\bm{x}}}

\DeclareMathAlphabet{\mathsfit}{\encodingdefault}{\sfdefault}{m}{sl}
\SetMathAlphabet{\mathsfit}{bold}{\encodingdefault}{\sfdefault}{bx}{n}

\def\gC{{\mathcal{C}}}

\def\gE{{\mathcal{E}}}

\def\gG{{\mathcal{G}}}

\def\gK{{\mathcal{K}}}

\def\gN{{\mathcal{N}}}
\def\gO{{\mathcal{O}}}

\def\gS{{\mathcal{S}}}

\def\gV{{\mathcal{V}}}

\def\gX{{\mathcal{X}}}
\def\gY{{\mathcal{Y}}}

\def\sR{{\mathbb{R}}}

\newcommand{\E}{\mathbb{E}}

\usepackage[preprint]{neurips_2026}

\usepackage[utf8]{inputenc} 
\usepackage[T1]{fontenc}    
\usepackage[hyphens]{url}
\usepackage{hyperref}       
\usepackage{booktabs}       
\usepackage{amsfonts}       
\usepackage{nicefrac}       
\usepackage{microtype}      
\usepackage{xcolor}         
\usepackage{setspace}

\title{FedCF: Fair Federated Conformal Prediction}

\usepackage{amsmath}
\usepackage{amssymb}
\usepackage{mathtools}
\usepackage{amsthm}

\usepackage[toc,page,header]{appendix}
\usepackage{minitoc}

\usepackage[export]{adjustbox}
\usepackage{algorithm}
\usepackage{algorithmic}
\usepackage{bm}
\usepackage{dsfont}
\usepackage{etoolbox}
\usepackage{float}
\usepackage[stable,hang,flushmargin]{footmisc}
\usepackage{makecell}
\usepackage{multirow}
\usepackage{natbib}
\usepackage{wrapfig}
\usepackage{ulem}
\usepackage{subcaption}
\usepackage{enumitem}

\usepackage{tikz}
\newcommand*\circled[1]{\tikz[baseline=(char.base)]{
            \node[shape=circle,draw,inner sep=1pt] (char) {#1};}}

\usepackage{thmtools}       
\usepackage{thm-restate}

\theoremstyle{plain}
\newtheorem{theorem}{Theorem}[section]

\newtheorem{lemma}[theorem]{Lemma}

\theoremstyle{definition}

\newcommand{\calib}{\mathcal{D}_{\mathrm{calib}}}

\newcommand{\predict}{\mathcal{C}_\lambda(\vx)}

\newcommand\swapifbranches[3]{#1{#3}{#2}}
\makeatletter
\MHInternalSyntaxOn
\patchcmd{\DeclarePairedDelimiter}{\@ifstar}{\swapifbranches\@ifstar}{}{}
\MHInternalSyntaxOff
\makeatother

\DeclarePairedDelimiter{\ceil}{\lceil}{\rceil}

\DeclarePairedDelimiter{\parens}{(}{)}
\DeclarePairedDelimiter{\brackets}{[}{]}
\DeclarePairedDelimiter{\braces}{\{}{\}}
\DeclarePairedDelimiter{\abs}{\lvert}{\rvert}

\expandafter\def\expandafter\normalsize\expandafter{%
    \normalsize%
    \setlength\abovedisplayskip{4pt}%
    \setlength\belowdisplayskip{4pt}%
    \setlength\abovedisplayshortskip{-8pt}%
    \setlength\belowdisplayshortskip{4pt}%
}

\newcommand\blfootnote[1]{%
  \begin{NoHyper}
  \begingroup
  \renewcommand\thefootnote{}\footnote{#1}%
  \addtocounter{footnote}{-1}%
  \endgroup
  \end{NoHyper}
}

\author{Anutam Srinivasan~\textsuperscript{1,*$\dagger$}, Aditya Vadlamani~\textsuperscript{2,*}, \textbf{Amin Meghrazi~\textsuperscript{2}, Srinivasan Parthasarathy~\textsuperscript{2}}\\ \textsuperscript{1}Georgia Institute of Technology\quad\textsuperscript{2}The Ohio State University \\\texttt{asrinivasan350@gatech.edu, vadlamani.12@osu.edu,}\\\texttt{meghrazi.1@osu.edu, srini@cse.ohio-state.edu} }

\begin{document}
\addtocontents{toc}{\protect\setcounter{tocdepth}{0}}
\maketitle

\begin{abstract}
  Conformal Prediction (CP) is a widely used technique for quantifying uncertainty in machine learning models. In its standard form, CP offers probabilistic guarantees on the coverage of the true label, but it is agnostic to sensitive attributes in the dataset. Several recent works have sought to incorporate fairness into CP by ensuring conditional coverage guarantees across different subgroups. One such method is Conformal Fairness (CF). In this work, we extend the CF framework to the Federated Learning setting and discuss how we can audit a federated model for fairness by analyzing the fairness-related gaps for different demographic groups. We empirically validate our framework by conducting experiments on several datasets spanning multiple domains, fully leveraging the exchangeability assumption.\blfootnote{$^*$Equal Contribution, \textsuperscript{$\dagger$}This work was done while the author was a student at The Ohio State University.}
\end{abstract}
\begin{figure}[H]
    \centering
    \includegraphics[width=\linewidth]{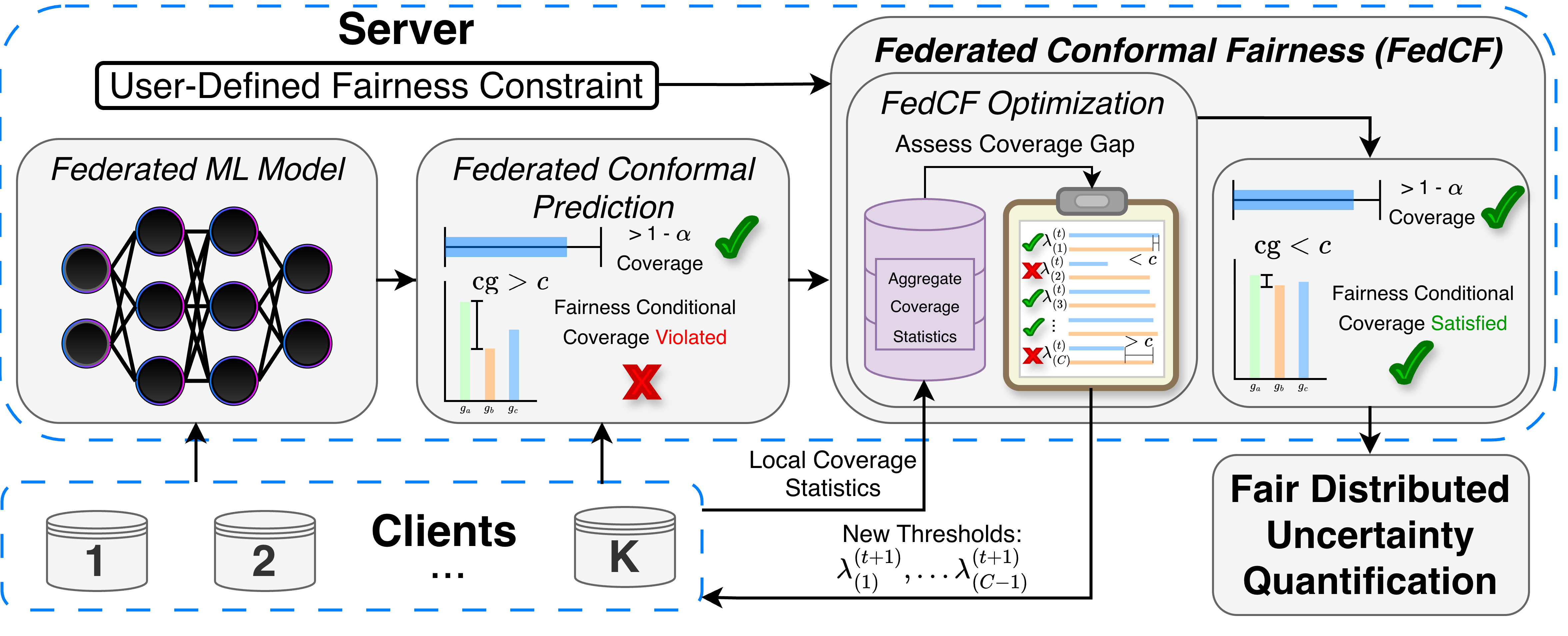}
    \caption{End-to-End FedCF Pipeline.}
    \label{fig:fed_cf:pipeline}
    \vspace{-2mm}
\end{figure}

\section{Introduction}
Ensuring model fairness is an important construct for trustworthy machine learning (ML). ML models, when not calibrated for fairness, are prone to developing biases at each stage of an ML pipeline, as reflected by their predictions \citep{mehrabi2021survey}. We define bias as disparate performance (e.g., classification accuracy) across different sub-populations. During data collection, measurement bias may arise from disproportionate sampling across subpopulations, whereas representation bias may stem from insufficient training data for specific strata. During training, these biases are inductively learned by the model--leading to incorrect predictions in safety-critical tasks. These models are also susceptible to algorithmic bias from regularization and optimization techniques during model training, leading to incorrect generalization for marginalized groups. To mitigate these risks, many ML models must comply with regulations issued by local governing bodies \citep{hirsch2023business}. Towards model compliance \citet{komala2024fair, agrawal2024no, jones2025rethinking} proposed approaches to enhance model fairness in tasks, including federated graph learning and representation learning.

Developing trustworthy ML frameworks with mathematically rigorous guarantees is an important requirement 
for regulated safety-critical tasks. In this context, researchers have explored Conformal Prediction (CP), an uncertainty quantification (UQ) technique that only assumes statistical exchangeability, to develop trustworthy models \citep{vovk2005algorithmic}. 
Practitioners have adopted CP due to its model-free assumption and post-hoc application \citep{cherian2020washington}. Additionally, users can apply CP to structured data (e.g., graphs), which cannot be used with traditional IID-based methods \citep{maneriker2025conformal}. However, vanilla CP is not calibrated for fairness~\citep{cresswell2025conformal}.

Several approaches have been proposed at the intersection of fairness and CP.  
\citet{romano2020malice} developed a CP approach for the regression setting to ensure equalized coverage across protected groups. \citet{lu2022fair} considers equalized coverage in a classification task 
for medical imaging.   
\citet{zhou2024conformal} extend \citet{romano2020malice} and provide an algorithm adaptive to sensitive groups to increase the predictive power of the CP sets when several sensitive attributes are present, and focus on the classification task. 
Lastly, \citet{vadlamani2025a} provides a framework that ensures fair coverage of positive outcomes without requiring protected attributes at inference time, unlike prior work. 
Orthogonally, CP has been used to enhance the fairness of other tasks. To mitigate bias in LLM-based recommender systems, \citet{fayyazi2025facter} explores iteratively using fairness-aware CP. 

The above fair CP methods rely on centralized data, a challenge in domains such as healthcare and finance, where data is often decentralized and data privacy regulations apply.  Federated learning offers a solution by enabling collaborative model training without transferring raw data off local clients.  However, while recent work on federated CP exists~\citep{lu2023federated, humbert2023oneshot}, the shift toward privacy-preserving decentralized training often exacerbates unfairness~\citep{kanaparthy2022, lo2025} in ML models.
Simultaneously balancing fairness and federated requirements in CP is non-trivially challenging due to variability in client requirements and capabilities, cross-client distribution of sensitive attributes, and efficiency.  



\textbf{Key Contributions}: 
We present Federated Conformal Fairness (FedCF), see Figure~\ref{fig:fed_cf:pipeline}, which adapts Conformal Fairness (CF)~\citep{vadlamani2025a} to a federated setting, addressing the above non-trivial challenges while maintaining CF's theoretical guarantees in a distributed setting.

{\noindent \textbf{Extending CF Theory to FL.}} We discuss how to bound conditional coverage according to a user-specified fairness notion when data is decentralized. To facilitate this, we develop a novel \textit{mixture-of-clients formalization}, in which we decompose the conditional coverage into a sufficient set of terms that each client can compute using local data, determine how the server should aggregate these terms to bound the conditional coverage, and theoretically prove the validity of our approach. 

{\noindent \textbf{Descent-Based CF Formulation.}} We revise the original CF algorithm to reduce the number of communication rounds required. We do this by reformulating the original iterative approach presented in~\cite{vadlamani2025a} into a descent-based optimization framework. This allows FedCF to be embedded directly into an FL algorithm and to integrate naturally within existing FL systems.

{\noindent \textbf{Flexible Aggregation Protocols.}} We consider the client-server communication overhead and its tradeoff with preserving data privacy. Specifically, we propose two approaches, with one having less communication overhead and the other being more privacy-preserving of client data. 

{\noindent \textbf{Real-World Empirical Validation.}} We evaluate FedCF on several datasets with naturally induced \textbf{\textit{data heterogeneity}}, across different modalities and  popular fairness metrics. We observe that FedCF can control for a particular coverage gap level while maintaining the original CP guarantee.
\vspace{-5mm}
\section{Background}
\vspace{-3mm}
\subsection{Conformal Prediction}
\label{subsec:cp_background}
Conformal Prediction (CP)~\citep{vovk2005algorithmic} is a widely used framework for quantifying predictive uncertainty in ML. CP provides rigorous statistical guarantees without imposing assumptions on the model, requiring only that the data is \textit{exchangeable}--a broader condition than IID and compatible with non-IID or structured settings (e.g., graphs).

We focus on split (inductive) CP in the classification setting. Let $\bm{x}_i\in\gX = \sR^d$ and $y_i\in\gY = \braces{0, \dots, C - 1}$ denote features and labels. Given a calibration dataset, $\calib = \braces{(\bm{x}_i, y_i)}_{i = 1}^n$, {our goal is} to construct a {set-valued predictor} $\gC$ such that, for an exchangeable test point $(\bm{x}_{\text{test}}, y_{\text{test}})$, 
\begin{equation}\textstyle
    1 - \alpha \leq \Pr\brackets{y_{\text{test}}\in \gC(\vx_{\text{test}})} \leq 1 - \alpha + \frac{1}{\abs{\calib} + 1},
    \label{eq:cp}
\end{equation}
where $1 - \alpha\in (0, 1)$ is the target \textit{coverage level}. We refer to Equation \ref{eq:cp} as the \textit{coverage guarantee}. Concretely, given a non-conformity score $s: \gX\times\gY\to\sR$, define the \textit{conformal quantile} as
{\small
\begin{equation}\label{eq:quantile}
\hat{q}(\alpha) = \text{Quantile}\parens{\frac{\ceil{(n+1)(1-\alpha)}}{n}; \{s(\bm{x}_i, y_i)\}_{i=1}^{n}}.
\end{equation}
}
{The resulting prediction set} $\gC_{\hat{q}(\alpha)}\parens{\bm{x}_{\text{test}}} = \{y \in \gY: s(\bm{x}_{\text{test}}, y) \leq  \hat{q}(\alpha)\,\}$ {satisfies the guarantee in} \ref{eq:cp}.

\noindent{\bf Evaluating CP}: {Two standard metrics are considered (averaged over the test set):} (1) \textit{Coverage}, {the estimated test-time probability}, $\Pr\brackets{y_{\text{test}}\in \gC_{\hat{q}(\alpha)}(\bm{x}_{\text{test}})}$; and (2) \textit{Efficiency}, the prediction set size, $\abs{\gC_{\hat{q}(\alpha)}(\bm{x}_{\text{test}})}$. {These are typically in tension as a higher desired coverage leads to larger sets.}

\subsection{Federated Learning}
A key contributor to developing strong deep learning models is providing a large amount of quality training data~\citep{kaplan2020}. However, in domains including healthcare and finance, collecting large amounts of data may be prohibitive due to privacy concerns. Federated learning (FL)~\cite{mcmahan2017} is a framework for collaborative learning that keeps training data decentralized and private. Given $K$ clients that participate in training, each client has its own training data that it wants to keep private. The goal is to optimize a global loss function $L$ that is the weighted average of local risk functions $\ell_k$. Formally, FL finds weights $\theta^*$ s.t.
{$
    \theta^{*} =
    \underset{\theta}{\text{argmin}} \braces{L(\theta) = \sum_{k = 1}^K w_k \E_{({x^{(k)}, y^{(k)})}\sim P_k} \brackets{\ell_k\parens{\theta; x^{(k)}, y^{(k)}}}},
$}
where $P_k$ is client $k$'s local distribution and $\vw\in\Delta^K$ are weights~\citep{lu2023federated}. 

\subsection{Federated Conformal Prediction (FCP)}
\label{subsec:federated_cp}
\noindent{\bf Setting:}
In FL, the development and calibration datasets are partitioned over
$K$ clients. Meaning, each client $k\in\braces{1, \dots, K} = \gK$ retains a private calibration set
{\small $\calib^{(k)}=\braces{\parens{\vx^{(k)}_i,y^{(k)}_i}}_{i=1}^{n_k}$} drawn from an unknown local
distribution $P_k$. The goal is to still construct a prediction set function $\gC$ such that for any test point $(\vx_{\text{test}},y_{\text{test}})\sim Q_{\text{test}}$, where $Q_{\mathrm{test}}=\sum_{k=1}^{K}\gamma_k P_k$ is the mixture distribution with weights
$\gamma_k\propto(n_k+1)$~\citep{lu2023federated}, Equation \ref{eq:cp} is satisfied. This is done while respecting the communication and privacy constraints of FL.



\noindent{\bf Partial exchangeability and the FCP algorithm.}
FCP~\citep{lu2023federated} relies on the notion of \emph{partial exchangeability}~\citep{deFinetti1971-DEFOTC}, by requiring that within each client, with probability $\gamma_k$, \(
\{{s({\vx^{(k)}_1,y^{(k)}_1}),\dots,s({\vx^{(k)}_{n_k},y^{(k)}_{n_k}}),
  s\parens{\vx_{\mathrm{test}},y_{\mathrm{test}}}}\}
\) is exchangeable. This assumption helps account for real-world label skew and data heterogeneity amongst clients.
Under this assumption, FCP aggregates all non-conformity scores and
selects the $\parens{(1-\alpha)(N+K)}$-th statistic as follows
{\small
\begin{equation}
\hat{q}(\alpha) = \text{Quantile}{\parens{\frac{\ceil{(N + K)(1-\alpha)}}{N}; \braces{\parens{\vx^{(k)}_i,y^{(k)}_i}}_{k,i}}}
\label{eq:fedcf:fedcp_quantile}
\end{equation}
}
where $N=\sum_{k=1}^{K}n_k$. The prediction set $C_{\alpha}(\vx)=\{y:s(\vx,y)\le\hat q_{\alpha}\}$ then
satisfies,
\begin{equation}\textstyle
1-\alpha
\;\le\;
\Pr\brackets{y_{\mathrm{test}}\in C_{\alpha}(\vx_{\mathrm{test}})}
\;\le\;
1-\alpha+\frac{K}{N+K}.
\label{eq:fed_cp_guarantee}
\end{equation}
\noindent{\bf Communication-efficient quantile sketches.}
To preserve privacy, instead of transmitting all $N$ scores to the server, each client can send a mergeable sketch (e.g., T-Digest~\citep{dunning2021}, DDSketch~\citep{masson2019}). Doing so loosens the guarantee given in Eq.~\eqref{eq:fed_cp_guarantee}~\citep{lu2023federated}.


\subsection{Conformal Fairness}
\label{subsec:cf_background}
The \emph{Conformal Fairness (CF)} framework \citep{vadlamani2025a} formalizes the notion of fair prediction sets by considering the disparity in \emph{conditional} coverage between sensitive groups and improves upon the marginal coverage guarantees of traditional CP, where different groups face disparate coverage.
Formally, let $M$ denote a fairness metric (e.g. Equal Opportunity) and defines the filter function $F_M:\mathcal{X}\times\mathcal{Y}\times\mathcal{G}\times\mathcal{Y}^{+}\to \{0, 1\}$, where $\mathcal{G}$ is a set of sensitive attributes, and $\mathcal{Y}^+ \subseteq \mathcal{Y}$ is the set of advantaged outcomes, to assess conditional coverages. 
For example, $F_M$ for Equal Opportunity is defined as $F_M(\vx_\mathrm{test}, y_\mathrm{test}, g, y):=\mathbf{1}[\vx_\mathrm{test} \in g \wedge y_\mathrm{test}=y]$. Then, CF finds class-wise thresholds, $\lambda_\text{opt} \in \mathbb{R}^{|\mathcal{Y}|}$ and defines $\gC_{\lambda_\text{opt}}\parens{\bm{x}_{\text{test}}} = \{y \in \gY: s(\bm{x}_{\text{test}}, y) \leq  \lambda_{\text{opt};(y)}\,\}$ where ${\lambda_{\text{opt};(y)}}$ is the threshold for class $y$, such that, $\forall\, g_a,g_b\in\mathcal{G},\ \tilde y\in\gY^+$:
{\small
\begin{equation}\label{eq:cf_guarantee}
\left|\Pr\!\big[\tilde y\in \gC_{\lambda_\text{opt}}(\bm{x}_\text{test})\mid F_M(\bm{x}_\text{test}, y_\text{test}, g_a, \tilde y) = 1\big]
\hspace{-1.2mm}-\hspace{-1.2mm}
\Pr\!\big[\tilde y\in \gC_{\lambda_\text{opt}}(\bm{x}_\text{test})\mid F_M(\bm{x}_\text{test}, y_\text{test}, g_b, \tilde y) = 1\big]\right|
\;\hspace{-1.5mm}\le\;\hspace{-1.5mm} c,
\end{equation}
}

where $c$ is a user-specified closeness criterion dictating how strictly fairness must be satisfied. CF searches over $\Lambda = [\lambda_0, \lambda_\text{max}]$ to minimize ${\lambda_{\text{opt};(y)}}$ (thereby the prediction set size) while satisfying Equation \ref{eq:cf_guarantee}. CF sets $\lambda_0 = \hat{q}(\alpha)$ to retain the coverage guarantee $\Pr\!\big[ y\in \gC_{\lambda_\text{opt}}(\bm{x}_\text{test})\big] \ge 1-\alpha$, while $\lambda_\text{max}$ is the largest value the non-conformity score may take.
To determine if a threshold $\lambda$ satisfies \eqref{eq:cf_guarantee} under exchangeability, CF finds tight upper and lower bounds $U_\text{cov}(\lambda, F_M, g, \tilde{y})$ and $L_\text{cov}(\lambda, F_M, g, \tilde{y})$ such that $L_\text{cov}(\lambda, F_M, g, \tilde{y}) \le \Pr\!\big[\tilde y\in \gC_{\lambda_\text{opt}}(\vx_\mathrm{test})\mid F_M(\vx_\mathrm{test}, y_\mathrm{test}, g, \tilde y) = 1\big] \le U_\text{cov}(\lambda, F_M, g, \tilde{y})$ and formally computes the coverage gap, 
\begin{equation}\label{eq:cg_form_1}
    \textstyle\text{cg}(\lambda, F_M, \tilde y, \mathcal{G}):= \max\limits_{g_a \in \mathcal{G}}\braces{U_\text{cov}(\lambda, F_M, g_a, \tilde{y})}-\min\limits_{g_b \in \mathcal{G}}\braces{L_\text{cov}(\lambda, F_M, g_b, \tilde{y})}.
\end{equation}
For a given positive label $\tilde y$, CF returns the minimum $\lambda$ that satisfies $\text{cg}(\lambda, F_M, \tilde y, \mathcal{G}) \le c$ as ${\lambda_{\text{opt};(\tilde y)}}$. For labels $y \notin \mathcal{Y}^{+}$, CF returns ${\lambda_{\text{opt};( y)}} = \lambda_0$. With this procedure, CF offers a theoretically grounded way to bound fairness disparities (Eq. \eqref{eq:cf_guarantee}) without sacrificing CP’s finite-sample coverage. Note that CF does not aim to equalize coverage across groups or classes (like other works~\citep{romano2020malice}), but instead controls the \textit{coverage gap} across groups for different labels. More details on Conformal Fairness can be found in Appendix~\ref{app:fed_cf:descent_analysis}.

\section{FedCF Theory and Methodology}\label{sec:methods}
In this section, we begin by establishing the theoretical and methodological foundations. We introduce a descent-based reformulation of the CF Framework (\ref{subsec:cf_descent}) that is conducive to the FL setting. Next, we extend the notion of the \textit{coverage gap} (\ref{subsec:fed_cov_gap}) and clarify our assumptions and the theoretical guarantees we can derive. Finally, we present the \emph{Fed}erated \emph{C}onformal \emph{F}airness (FedCF) Framework.

\subsection{Revisiting the Conformal Fairness Algorithm}
\label{subsec:cf_descent}
One drawback of the original CF algorithm is that it performs a grid search over a \textit{discretized space} to find the minimal $\lambda$ satisfying the fairness specification. This requires computing the coverage gap for each positive label every iteration, which is inefficient in the FL setting due to client-server communication. We introduce a \textit{descent-based}\footnote{We use the term descent-based since we are adaptively searching for a minimal $\lambda$ that satisfies the fairness constraint. The term descent-based does not imply the use of gradients~\citep{liu2020descent, Golovin2020Gradientless}.} CF algorithm also seen in Algorithm~\ref{alg:descent_cf_algo_extended}.

\noindent{\bf Input:} The core optimization step requires a positive label, $\tilde{y}\in\gY^+$, the clients' calibration data, $\{\calib^{(k)}\}_{k\in\gK}$ and a filter function, $F_M$ to compute the coverage gap across the set of groups $\gG$. For fairness evaluation, we specify a closeness criterion, $c$. For the optimization hyperparameters, we specify the number of optimization rounds, initial learning rate $\eta$, and momentum $\mu$. We initialize $\lambda_0 = \hat{q}(\alpha)$ as defined in \eqref{eq:fedcf:fedcp_quantile} to ensure the original coverage guarantee~\eqref{eq:fed_cp_guarantee} is satisfied.

\noindent{\bf Problem:} Given the inputs, let $\lambda\in\Lambda$ be a threshold and $\text{cg}(\lambda, F_M, \tilde y, \mathcal{G})$ be the coverage gap evaluated at $\lambda$. The optimization problem for CF can be formulated as
\begin{equation}\label{eq:classwise_formulation}
    \forall\tilde{y} \in \mathcal{Y}^{+}:  {\lambda_{\text{opt};(\tilde y)}} = \min_{\lambda\in\Lambda} \lambda\quad \text{ subject to } \text{cg}(\lambda, F_M, \tilde y, \mathcal{G}) - c \leq 0.
\end{equation}
\noindent{\bf Procedure:} In round $t$, for each $\tilde y\in\gY^+$, given the current threshold {\small$\lambda^{(t)}_{(\tilde y)}$} and current optimal (and satisfactory) threshold {\small$ \bar\lambda_{\text{opt};(\tilde y)}^{(t)}$}, we can compute the coverage gap {\small$\text{cg}_t\coloneqq\text{cg}(\lambda^{(t)}_{(\tilde y)}, F_M, \tilde y, \mathcal{G})$} and evaluate whether it adheres to our fairness constraint. The update rule for {\small$\lambda^{(t)}_{(\tilde{y})}$} is,
{\small
\begin{equation}
\lambda^{(t+1)}_{(\tilde y)}
= \lambda^{(t)}_{(\tilde y)} + \eta\,2^{-p_t}\, b_{t+1},
\quad
b_{t+1} = \mu b_t + (\mathrm{cg}_t - c),
\quad
p_t = \max\!\left\{\left\lceil \log_2 ({\eta}/{a_{\max}}) \right\rceil, 0\right\},
\end{equation}
\begin{equation}
a_{\max} = \min\!\left\{\eta,\; \frac{\Delta\lambda}{|b_{t+1}|+\epsilon}\right\},
\quad
\Delta\lambda =
\begin{cases}
\bar{\lambda}^{(t)}_{\mathrm{opt};(\tilde y)} - \lambda^{(t)}_{(\tilde y)}, & b_{t+1} \ge 0,\\
\lambda^{(t)}_{(\tilde y)} - \lambda_0, & b_{t+1} < 0,
\end{cases}
\end{equation}
}
Our approach does not stop once a satisfactory $\lambda$ is found. We continue exploring smaller values (unlike in a discretized grid). In the case where our modified step size is $b_{t+1} > 0$, meaning the update rule takes us to a higher $\lambda_{t+1}$, and satisfies the fairness constraint, we perform a random restart 
to promote further exploration and escape local minima, since $\text{cg}(\lambda, F_M, \tilde{y}, \mathcal{G})$ is piecewise-monotonic \cite{milnor2006iterated}. For more details, see Appendix~\ref{app:fed_cf:descent_analysis} and Algorithm~\ref{alg:descent_cf_algo_extended} (lines~\ref{alg_step:random_restart_start} -\ref{alg_step:random_restart_end}). We note that this algorithm directly applies to the federated setting, with one important consideration: the coverage gap computation.
\subsection{Extending Coverage Gap to the Federated Setting}
\begin{table}[htbp!]
    \vspace{-5mm}
    \small
    \centering
    \caption{Important notation used for coverage gap calculation.}
    \label{tab:fed_cf:notation_table_short}
    \begin{tabular}{rp{3.25cm}|rp{6cm}}
        \toprule
        \textbf{Notation} & \textbf{Definition} & \textbf{Notation} & \textbf{Definition} \\
        \midrule
        $n_k$ & $\abs{\calib^{(k)}}$
        & $\calib^{(k)}$ & Client $k$'s calibration dataset. \\
        
        $n_k^{(g,\tilde y)}$ & $\abs{\gS_{k}^{(g,\tilde y)}}$
        & $\gS_{k}^{(g,\tilde y)}$ & $\scriptsize \braces{\textstyle (\vx_i, y_i) \in \calib^{(k)} \mid F_M(\vx_i, y_i, g, \tilde{y})=1}$ \\
        
        $\gamma_k$ & $\Pr[E_k]$
        & $E_k$ & The event $\bm x_\text{test}$ is exchangeable with $\calib^{(k)}$. \\
        
        $\pi^{(g, \tilde{y})}$ & \multirow{2}{*}{\parbox{\linewidth}{Point estimate for\\Term {\scriptsize \circled{IV}} in Equation \ref{eq:dissected_fair_cov}.}}
        & $L^{(g, \tilde{y})}, U^{(g, \tilde{y})}$ & Bounds for Term {\scriptsize \circled{IV}} in Equation \ref{eq:dissected_fair_cov}.\\
        
        &  & $\alpha_k^{(g,\tilde y); \lambda}$ & $\sum_{(\vx_i, \_) \in \gS_{k}^{(g,\tilde y)} }\1\brackets{s(\vx_i,\tilde y) \leq \lambda}$
        \\
        \bottomrule
    \end{tabular}
\end{table}
\label{subsec:fed_cov_gap}
\noindent\textbf{Assumptions \& Notation:} 
To account for heterogeneity in client data in FL, the presented theory assumes \emph{partial exchangeability}. It requires only that one instance of a group-label pair be observed across all clients (cf. Remark \ref{rem:data_req}). No further distributional assumptions are made here. For clarity, Table \ref{tab:fed_cf:notation_table_short} summarizes the notation used in our main theorem, with a complete table in Appendix \ref{app:fed_cf:notation}.

\noindent{\textbf{Mixture-of-Clients Coverage Gap Formulation.}} In the federated setting, since the data is decentralized, we cannot directly estimate the terms in \eqref{eq:cf_guarantee}. Thus, we rewrite the quantity as: 
\begin{small}
    \begin{align}
    &\underbrace{\Pr\brackets{s(\vx_{\text{test}},\tilde{y})\leq \lambda\mid F_M(\vx_{\text{test}}, y_{\text{test}}, g, \tilde{y})=1}}_{\Pi_{cov}}=\sum\limits_{k=1}^{K} \Big(\underbrace{\Pr\brackets{s(\vx_{\text{test}},\tilde{y})\leq \lambda\mid F_M(\vx_{\text{test}}, y_{\text{test}}, g, \tilde{y})=1, E_k}}_\text{\circled{I}}\nonumber\\&\textstyle\qquad\cdot\underbrace{\Pr\brackets{F_M(\vx_{\text{test}}, y_{\text{test}}, g, \tilde{y}) = 1\mid  E_k}}_\text{\circled{II}}\cdot\underbrace{\Pr\brackets{E_k}}_\text{\circled{III}}\Big)\cdot \Big(\underbrace{\Pr\brackets{F_M(\vx_{\text{test}}, y_{\text{test}}, g, \tilde{y}) = 1}}_\text{\circled{IV}}\Big)^{-1}
    \label{eq:dissected_fair_cov}
    \end{align}
\end{small}
With this reformulation, we can estimate individual terms--either locally on each client or globally on the server. The following theorem presents bounds for the fairness-specific coverage level. Bounds for the individual terms are given in Lemmas \ref{lem:client_fair_cov}, \ref{lem:F_M_bound}, and \ref{lem:prior_term}. The corresponding proofs for the lemmas and theorems are provided in Appendix \ref{app:fed_cf:proofs}. 
\begin{restatable}[Interval Bounds]{theorem}{thmMainCovThm}\label{thm:main_cov_thm}
    The fairness-specific coverage level, $\Pi_\text{cov}$, given by Equation \ref{eq:dissected_fair_cov}, can be bounded as
    {$
        L_\text{cov}(\lambda, F_M, g, \tilde{y}) \leq \Pi_\text{cov} \leq U_\text{cov}(\lambda, F_M, g, \tilde{y}),
    $}
    where
    \begin{align}
        &L_\text{cov}(\lambda, F_M, g, \tilde{y}) = \sum_{k=1}^{K} \frac{\gamma_k\alpha_k^{(g,\tilde y); \lambda}}{(n_k+1)U^{(g,\tilde y)}}~\text{and}~U_\text{cov}(\lambda, F_M, g, \tilde{y}) = \sum_{k=1}^{K} \frac{\gamma_k(\alpha_k^{(g,\tilde y); \lambda}+1)}{(n_k+1)L^{(g,\tilde y)}}.
    \end{align}
    and $L^{(g,\tilde y)} = \sum_{k=1}^{K} \gamma_k\frac{n_k^{(g,\tilde y)}}{n_k+1}$ and $U^{(g,\tilde y)} =  \sum_{k=1}^{K} \gamma_k\frac{n_k^{(g,\tilde y)}+1}{n_k+1}$
\end{restatable}

Theorem \ref{thm:main_cov_thm} gives us bounds for the coverage level, which translate into bounds for the coverage gap between groups for fairness evaluation. We note that $\lambda_\text{opt} = \lambda_\text{max}$ is a degenerate solution satisfying the closeness criterion, and formally show that the condition to ensure entries of $\lambda_\text{opt} < \lambda_\text{max}$:


\begin{restatable}{theorem}{thmIntervalWidthBound}\label{thm:iw_bound}
Given a closeness criterion, $c$ and $\gamma_k = \frac{n_k+1}{N+K}$, in order to have non-vacuous prediction sets (i.e., $\lambda < \lambda_\textrm{max}$), we need $N^{(g,\tilde y)}:= \sum_{k=1}^{K}n_k^{(g,\tilde{y})} \geq \frac{2K}{c}$ for all $(g,\tilde y) \in \mathcal{G}\times \mathcal{Y}^{+}$.
\end{restatable}

\begin{restatable}{remark}{remDataReq}
\label{rem:data_req}
We only require a certain amount of data \textbf{across} the federation--not per-client.
\end{restatable}

\subsubsection{Approaches to Improve Prediction Set Size (Efficiency)}\label{subsubseq:improve_eff}
As seen in \citet{vadlamani2025a}, improving fairness inevitably comes at the cost of utility (efficiency). To address this tension, we explore two approaches.

\noindent\textbf{I. MLE:} FedCF relies on the broad assumption of partial exchangeability, similar to FCP, resulting in more conservative bounds. With stronger assumptions, we can get a less conservative bound. For example, suppose the data satisfies a notion of \textit{partial IID}, that is, for each client $k$,  with probability $\gamma_k$, \(
\braces{{s({\vx^{(k)}_1,y^{(k)}_1}),\dots,s({\vx^{(k)}_{n_k},y^{(k)}_{n_k}}),
  s\parens{\vx_{\mathrm{test}},y_{\mathrm{test}}}}}
\) is \textit{IID}. Note this is \textbf{not} the same as assuming IID across the federation. Then we can construct a point estimate of our coverage bound using maximum likelihood estimation (MLE). This gives us a tighter estimate, but may violate the earlier finite-sample guarantees. We formalize this in the following theorem with the proof in Appendix~\ref{app:fed_cf:proofs}. 
\begin{restatable}[Point (MLE) Estimate]{theorem}{thmMainCovThmMLE}\label{thm:main_cov_thm_mle}
    If we assume partial-IID, using MLE estimates for each term, we get the following estimate for the fairness-specific coverage level,
        $\Pi_\text{cov} =\sum_{k=1}^{K} \frac{\gamma_k\alpha_k^{(g,\tilde y); \lambda}}{n_k\pi^{(g,\tilde y)}}$.
\end{restatable}

\noindent\textbf{II. Using Group Information} \label{subsec:group_info} FedCF \emph{does not require group-information} at inference time. However, FedCF can be modified to use group information, such that $\lambda_\text{opt} \in \Lambda^{|\mathcal{Y}|\times|\mathcal{G}|}$. For a label $y$ the threshold vector is $\lambda_{\text{opt};(y)} \in \Lambda^{|\mathcal{G}|}$, the problem becomes,
\begin{equation}\label{eq:groupwise_cf}
    \min_{\lambda\in \Lambda^{|\mathcal{G}|}}f(\lambda) \quad \text{subject to }\quad \text{cg}(\lambda, F_M, \tilde y, \mathcal{G}) - c \leq 0,
\end{equation}
where $f$ is the scalar objective, and $\text{cg}$ is redefined to accommodate different thresholds for each group. In this work, we take $f(\lambda) = \lambda^{\top}1$ to be the elementwise sum and approximately solve \eqref{eq:groupwise_cf}. Theoretical details of problem \eqref{eq:groupwise_cf} and how we solve it are in Appendix~\ref{app:additional_theory}.
\vspace{-2mm}
\subsection{FedCF: Federated Conformal Fairness framework}

Having established the sufficient terms to compute the fairness-specific coverage gap, we now present the FedCF framework, which is depicted in Figure~\ref{fig:fed_cf:pipeline}. We discuss FedCF in the context of the interval-bounds estimates from Theorem \ref{thm:main_cov_thm}, noting the discussion also applies to the point-estimate case by setting $L_\text{cov} = U_\text{cov} = \Pi_\text{cov}$. Recall, the coverage gap is given by Equation \ref{eq:cg_form_1} or equivalently,
{\small
\begin{equation}
    \label{eq:cg_form_2}
    \textstyle
    \text{cg}(\lambda, F_M, \tilde y, \mathcal{G}) = \max\limits_{g_a, g_b \in \mathcal{G}}\braces{U_\text{cov}(\lambda, F_M, g_a, \tilde{y})-L_\text{cov}(\lambda, F_M, g_b, \tilde{y})}.
\end{equation}}
While Equations \ref{eq:cg_form_1} and \ref{eq:cg_form_2} are mathematically equivalent, their formulations lead to two different communication and aggregation strategies, demonstrating the tradeoff between \textbf{communication overhead} and \textbf{privacy}. We present the \textit{communication efficient} protocol in the main paper and the \textit{enhanced privacy} protocol in Appendix \ref{app:fed_cf:privacy}. In Appendix \ref{app:fed_cf:privacy}, we also present a hybrid protocol, where clients select whether to use the \textit{communication efficient} or \textit{enhanced privacy} protocol.
We include FedCF extensions concerning differential privacy in Appendix \ref{app:fedcf:diff_privacy}. 

Note that $U_\text{cov}$ and $L_\text{cov}$ depend on $L^{(g, \tilde y)}$ and $U^{(g, \tilde y)}$, respectively. Since $L^{(g, \tilde y)}$ and $U^{(g, \tilde y)}$ are also computed in a federated manner on the server, we compute them \textit{prior} to computing the coverage gap for any particular $\lambda$ value\footnote{The prior term is computed analogously to the coverage gap.}.
Given that these priors are available on the server, we can compute the fairness-specific coverage gap. From Theorem \ref{thm:main_cov_thm}, each client computes and sends two values for each $(g, \tilde{y})\in \gG\times\gY^+$ pair: $\frac{\alpha_k^{(g,\tilde y); \lambda}}{n_k + 1}$ and $\frac{\alpha_k^{(g,\tilde y); \lambda} + 1}{n_k+1}$. Once the server receives these pairs from each client, it aggregates these quantities to derive $L_\text{cov}$ and $U_\text{cov}$ for each $(g, \tilde{y})$ and computes the coverage gap (Equation \ref{eq:cg_form_1}). $U_\text{cov}$ is limited to $1$ to reconcile $\Pr[\,\cdot\,] \leq 1$ for any event.
\vspace{-2mm}
\paragraph{Communication Complexity and Privacy Implications.} Each client is responsible for sending messages of size totaling $\gO\parens{2\cdot\abs{\gG}\abs{\gY^{+}}}$ to the server per server round. While this is linear in terms of the number of $(g, \tilde{y})$ pairs, we note that with enough $\lambda$s, the server can learn the distribution of $\Pr\brackets{s(\bm{x}_{\text{test}},\tilde{y})\leq \lambda \mid F_M(\bm{x}_{\text{test}}, y_{\text{test}}, g, \tilde{y})=1, E_k}$. 
\setlength{\tabcolsep}{1mm}

\section{Experiments}\label{sec:experiments}
\subsection{Setup.}
\paragraph{Datasets.} We evaluate the FedCF framework on four multi-class datasets: (1, 2) ACSIncome and ACSEducation \citep{ding2021retiring}, (3) Pokec-\{n, z\} \citep{takac2012data}, (4) Fitzpatrick \cite{groh2021evaluating}. Since these datasets were not originally designed for FL, we partition them into clients. 
For ACS, we use U.S. state/territory data with \textbf{six} regional schemes, yielding 4 (small), 8 (large), or 51 (all) clients, plus continental-only variants ($4$, $8$, and $48$ clients, respectively). For Pokec-\{n, z\}, each graph is treated as a client, as they are from distinct partitions of the larger \textit{Pokec} social network. Finally, for Fitzpatrick, since there is no natural partitioning scheme, we apply a Dirichlet partitioner~\citep{yurochkin2019} with $\alpha = 0.5$ to create $K\in\braces{2, 4, 8}$ clients. We use a $30/20/25/25$ stratified split for the full $\train/\valid/\calib/\test$.

\paragraph{Base Models.} For the ACS datasets, we use XGBoost~\citep{chen2016}. For Pokec-{n,z}, we use GraphSAGE with GCN aggregation~\citep{hamilton2017}. For Fitzpatrick, we use ResNet-18~\citep{he2016}. GraphSAGE and ResNet are trained with FedAvg~\citep{mcmahan2017}, while XGBoost uses FedXgbBagging from Flower~\citep{beutel2020flower}.

\paragraph{Baseline.} We construct a federated (fairness-agnostic) conformal predictor targeting a coverage level of $1 - \alpha = 0.9$ using FCP with T-Digest. We set $\gamma_k\propto (n_k + 1)$. For the non-conformity score, we use APS~\citep{romano2020classification} and RAPS~\citep{angelopoulos2022uncertaintysetsimageclassifiers} for all datasets, and DAPS~\citep{zargarbashi23conformal}, a graph-specific method, for Pokec-\{n,z\}. Descriptions of the scores are in Appendix \ref{app:scores}. We then assess fairness using $\lambda = \hat{q}(\alpha)$ for three popular group-fairness metrics, reformulated in Table \ref{tab:fed_cf:conf_fairness_defs}--Demographic Parity, Equal Opportunity, and Predictive Equality.

\paragraph{Evaluation Metrics:} We report two key metrics: (1) \textit{efficiency}, and (2) \textit{worst-case fairness disparity}. The latter captures the largest difference in conditional coverage across groups under the chosen fairness metric. For example, under \textit{Demographic Parity}, we report:
\begin{align}
\max_{\tilde{y}\in\gY^+}\max\limits_{g_a,g_b\in\mathcal{G}}\big|&\Pr\!\big[\tilde y\in \gC_{\lambda}(\vx_{test})\mid \vx_{test}\in g_a\big]-\Pr\!\big[\tilde y\in \gC_{\lambda}(\vx_{test})\mid \vx_{test}\in g_b\big]\big|.
\end{align}
All methods achieve the desired $1 - \alpha$ coverage since FedCF only ever increases $\lambda$. Coverage results are included in Appendix~\ref{app:fedcf:more_results}. Additional details on datasets and experimental setup are in Appendix \ref{app:fed_cf:experimental_details}.
\vspace{-2mm}
\subsection{Results}
In each figure, we use a \textbf{solid} line to represent the \textit{average} efficiency of the \textbf{base federated conformal predictors} across different thresholds and a \textbf{dashed} line to represent the corresponding \textit{average} worst-case fairness disparity. The bar plot shows the efficiency and worst-case fairness disparity using FedCF, while the \textbf{dots} indicate the \textit{desired} fairness disparity. We report the average base performance for clarity and readability. In all experiments, FedCF meets the fairness disparity bound under the closeness criterion $c$ (up to minor finite-sample variation), unlike the baseline FCP.
\vspace{-1.5mm}
\paragraph{Preserves Key Characteristics of CF.} Two important characteristics of the CF framework are that it is (1) agnostic to the specific non-conformity score function and (2) supports intersectional fairness. We demonstrate that our FedCF framework preserves these two characteristics via the Pokec-\{n, z\} dataset. Pokec-\{n, z\} each have two sensitive attributes: \textit{region} and \textit{gender}. In addition to considering each attribute individually, we can treat each \textit{pair} of attributes as distinct and apply FedCF. Furthermore, Pokec-\{n, z\} is a graph dataset. Recently, several developments have been made in graph CP research on non-conformity scores that utilize the graph structure. In addition to two standard CP methods--APS and RAPS--we also provide results using DAPS. Figure \ref{fig:pokec_intersec} shows how the FedCF framework can achieve the desired fairness criterion with minimal cost to efficiency for different non-conformity scores and when considering multiple groups.
\begin{figure}[ht!]
    \centering
    \begin{subfigure}{\textwidth}
    \centering
        \includegraphics[width=0.75\linewidth]{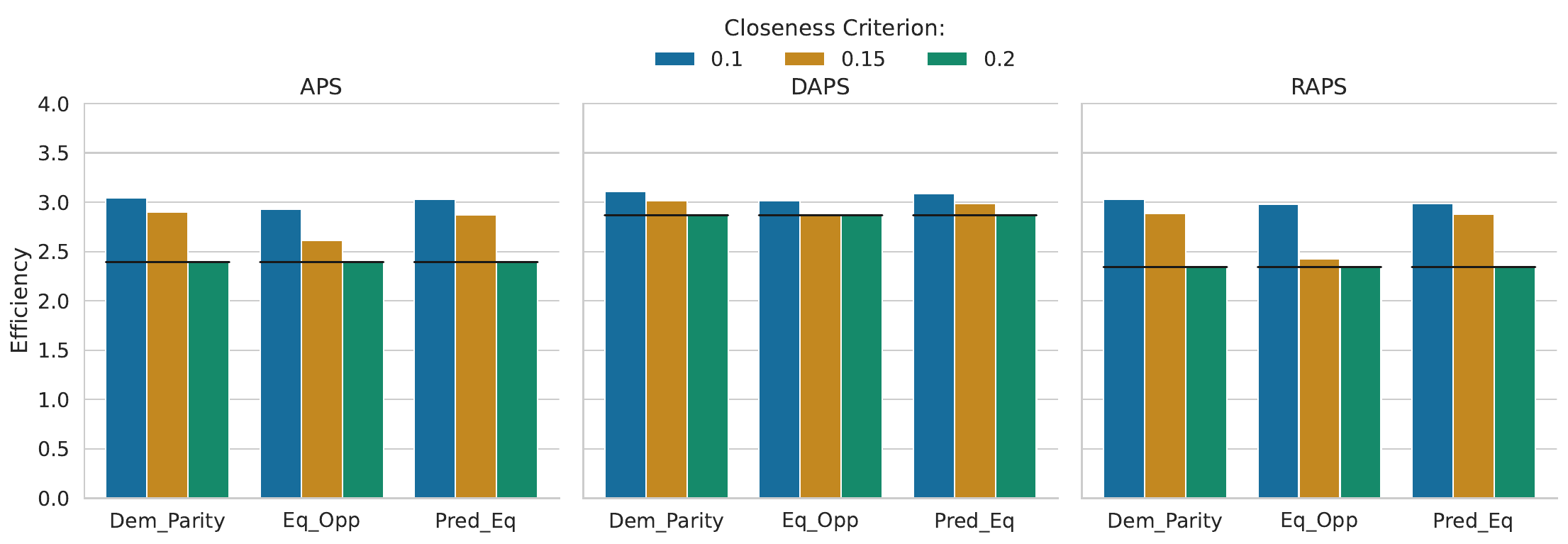}
    \end{subfigure}
    \begin{subfigure}{\textwidth}
    \centering
    \includegraphics[width=0.75\linewidth]{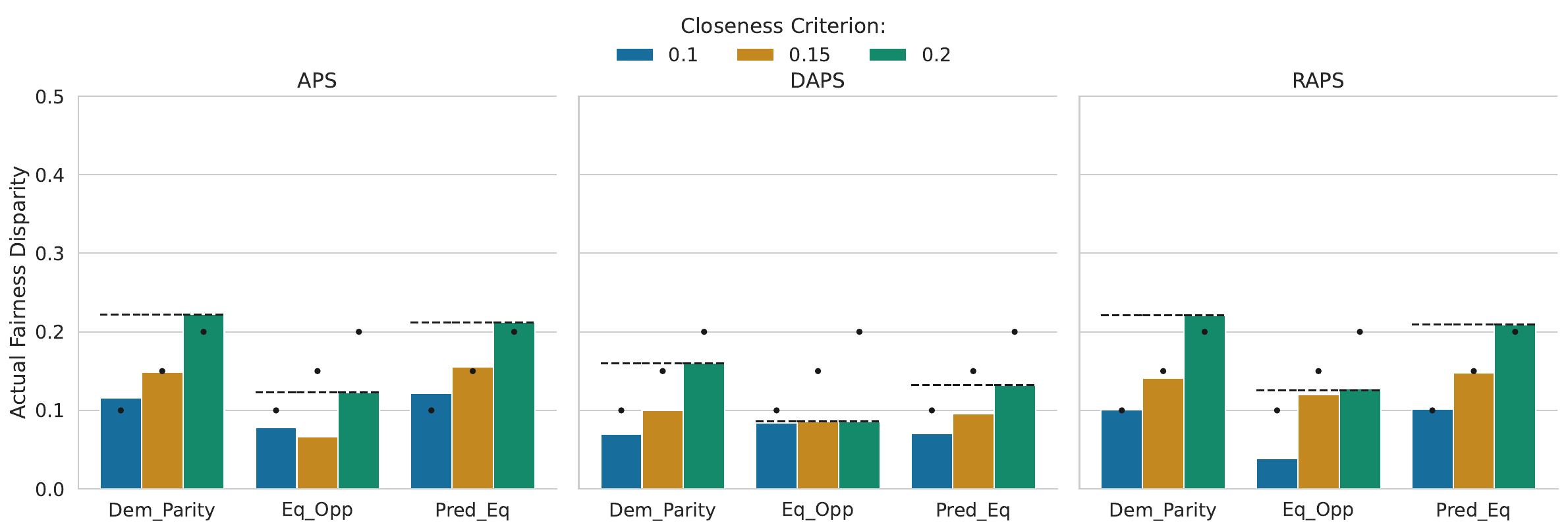}
    \end{subfigure}
    \captionsetup{font=small,labelfont=small}
    \caption{\textbf{Pokec-\{n, z\}} using \textbf{both} sensitive attributes. The top plots present the efficiency results, while the bottom plots are for the fairness disparities for (a) APS, (b) DAPS, and (c) RAPS. In all cases, FedCF achieves the desired closeness criteria better than the base federated conformal predictors.}
    \vspace{-5mm}
    \label{fig:pokec_intersec}
\end{figure}
\vspace{-1.5mm}
\paragraph{Robust Performance with Different Numbers of Clients.} An important trade-off in trustworthy FL is between predictive utility and maintaining fairness/privacy guarantees for each of its clients, which is increasingly challenging as the number of clients grows~\citep{wen2023survey}. To demonstrate how our framework can adapt to the number of clients, we use a Dirichlet partitioner with the Fitzpatrick dataset to evaluate performance with $K\in\braces{2, 4, 8}$ clients in addition to a centralized setup with a single client. In Table~\ref{tab:client_num_expt}, as the number of clients increases, the baseline performance worsens, but the FedCF framework can still control for the necessary closeness criterion. We omit Equal Opportunity for Fitzpatrick as it is not meaningful in the context of this dataset, which aims to predict a skin condition with the sensitive attribute being skin type. People with certain skin types are more likely to develop certain skin conditions, so the true positive rates of a classifier will typically not equalize, resulting in a degenerate (meaningless) solution. 
\begin{table*}[ht]
\centering
\scriptsize
\captionsetup{font=small,labelfont=small}
\caption{\textbf{Fitzpatrick using RAPS}. Each entry is of the form, \textbf{efficiency/fairness disparity}. We bold the lower fairness disparity value.  This table contains the results for $K\in\{1, 2, 4, 8\}$ clients. We observe that FedCF consistently matches or outperforms the base federated conformal predictor with disparity less than $c$.}
\label{tab:client_num_expt}
\vspace{-0.2cm}
\subfloat[$c = 0.1$]{%
\begin{tabular}{lcccccccc}
\toprule
 & \multicolumn{2}{c}{1 client} & \multicolumn{2}{c}{2 clients} & \multicolumn{2}{c}{4 clients} & \multicolumn{2}{c}{8 clients} \\
\cmidrule(lr){2-3} \cmidrule(lr){4-5} \cmidrule(lr){6-7} \cmidrule(lr){8-9}
Metric & Base & Ours & Base & Ours & Base & Ours & Base & Ours \\
\midrule
\textbf{Dem\_Parity} & 2.356 / 0.151 & 3.647 / \textbf{0.103} & 2.565 / 0.163 & 4.149 / \textbf{0.153} & 2.844 / 0.293 & 4.684 / \textbf{0.099} & 3.502 / 0.166 & 5.214 / \textbf{0.089} \\
\textbf{Pred\_Eq} & 2.356 / 0.111 & 3.647 / \textbf{0.109} & 2.543 / 0.177 & 4.140 / \textbf{0.177} & 2.837 / 0.287 & 4.564 / \textbf{0.102} & 3.502 / 0.171 & 5.293 / \textbf{0.083} \\
\bottomrule
\end{tabular}}

\vspace{0.1cm}

\subfloat[$c = 0.15$]{%
\begin{tabular}{lcccccccc}
\toprule
 & \multicolumn{2}{c}{1 client} & \multicolumn{2}{c}{2 clients} & \multicolumn{2}{c}{4 clients} & \multicolumn{2}{c}{8 clients} \\
\cmidrule(lr){2-3} \cmidrule(lr){4-5} \cmidrule(lr){6-7} \cmidrule(lr){8-9}
Metric & Base & Ours & Base & Ours & Base & Ours & Base & Ours \\
\midrule
\textbf{Dem\_Parity} & 2.356 / \textbf{0.151} & 2.356 / \textbf{0.151} & 2.565 / 0.163 & 2.793 / \textbf{0.163} & 2.837 / 0.291 & 3.347 / \textbf{0.114} & 3.498 / 0.166 & 3.859 / \textbf{0.088} \\
\textbf{Pred\_Eq} & 2.356 / \textbf{0.111} & 2.356 / \textbf{0.111} & 2.543 / 0.177 & 2.778 / \textbf{0.177} & 2.844 / 0.289 & 3.296 / \textbf{0.144} & 3.496 / 0.170 & 3.867 / \textbf{0.087} \\
\bottomrule
\end{tabular}}

\vspace{0.1cm}

\subfloat[$c = 0.2$]{%
\begin{tabular}{lcccccccc}
\toprule
 & \multicolumn{2}{c}{1 client} & \multicolumn{2}{c}{2 clients} & \multicolumn{2}{c}{4 clients} & \multicolumn{2}{c}{8 clients} \\
\cmidrule(lr){2-3} \cmidrule(lr){4-5} \cmidrule(lr){6-7} \cmidrule(lr){8-9}
Metric & Base & Ours & Base & Ours & Base & Ours & Base & Ours \\
\midrule
\textbf{Dem\_Parity} & 2.356 / \textbf{0.151} & 2.356 / \textbf{0.151} & 2.541 / \textbf{0.161} & 2.541 / \textbf{0.161} & 2.844 / 0.293 & 3.260 / \textbf{0.164} & 3.500 / 0.166 & 3.528 / \textbf{0.146} \\
\textbf{Pred\_Eq} & 2.356 / \textbf{0.111} & 2.356 / \textbf{0.111} & 2.541 / \textbf{0.177} & 2.541 / \textbf{0.177} & 2.837 / 0.287 & 3.256 / \textbf{0.167} & 3.501 / 0.171 & 3.524 / \textbf{0.150} \\
\bottomrule
\vspace{-1cm}
\end{tabular}}
\end{table*}
\vspace{-2mm}
\paragraph{Efficiency vs Fairness Trade-Off.} To make FedCF an actionable framework, it is essential to understand the utility trade-off when imposing fairness constraints. Using the interval-based approach to estimate the coverage gap gives a finite-sample guarantee for controlling fairness gaps. However, sometimes imposing fairness results in a severe cost to utility. For example, a degenerate conformal predictor (one with near-full efficiency) is ``fair,'' but completely impractical for use. By relaxing theoretical guarantees, we can improve efficiency by using tighter estimates of the coverage gap, such as MLE point estimates. Figure \ref{fig:mle_expt} compares the efficiency and fairness disparities when using interval bounds vs the MLE estimate on the ACSEducation dataset. We observe that with the interval bounds, we are always within the closeness criterion, but the efficiencies are quite high. Alternatively, using MLE estimates may exceed the desired closeness criterion, but be fairer than the baseline and not sacrifice as much efficiency. 
\begin{wraptable}{r}{0.4\textwidth}
\vspace{-1em}
    \centering
    \footnotesize
\captionsetup{font=scriptsize,labelfont=footnotesize}
    \caption{\textbf{ACSEducation (continental small, $4$ clients), \textbf{RAPS}} for \textbf{Demographic Parity} \textit{Average} Efficiency.}
    \label{tab:fedcf:classwise_and_groupwise}
    \vspace{-2mm}
    \begin{tabular}{c c c}
        \toprule
        $c$ & Classwise & (Group, Class)-wise \\
        \midrule
        0.1  & 5.006 & \textbf{4.791} \\
        0.15 & 4.823 & \textbf{4.391} \\
        0.2  & 3.975 & \textbf{3.613} \\
        \bottomrule
    \end{tabular}
    \vspace{-2mm}
\end{wraptable}
If available, using group information at inference time is another way to improve efficiency. Table~\ref{tab:fedcf:classwise_and_groupwise} compares the efficiency of FedCF using the classwise $\lambda_{\text{opt}; (y)}$ versus using a (group, class)-wise $\lambda_{\text{opt}; (y, g)}$. We observe that FedCF can flexibly use available group information, when available, to improve efficiency, though it can still provide guarantees even when such information is unavailable.
\begin{figure}[htbp!]
    \centering
    \begin{subfigure}{0.475\textwidth}
    \centering
        \includegraphics[width=\linewidth]{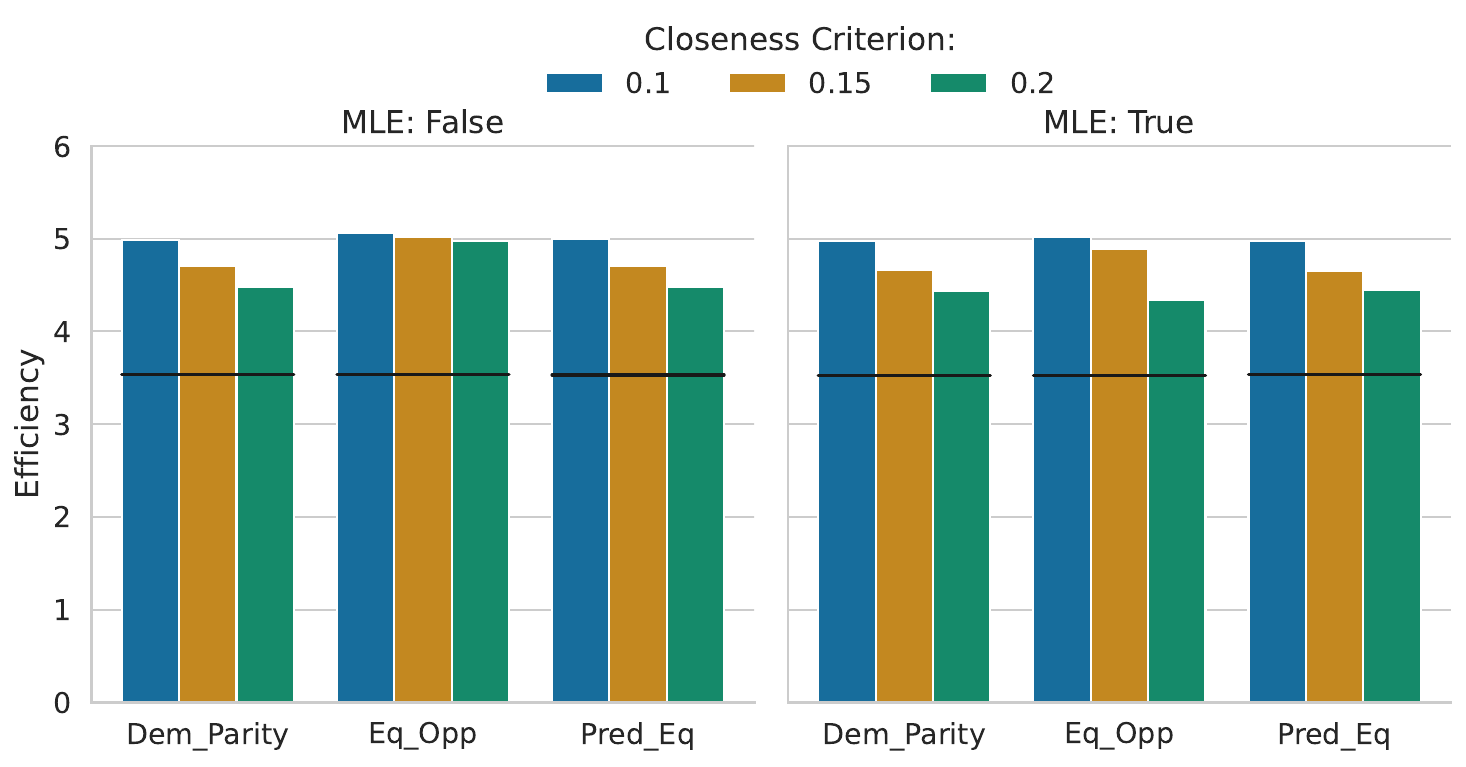}
    \end{subfigure}%
    \begin{subfigure}{0.5\textwidth}
    \centering
        \includegraphics[width=\linewidth]{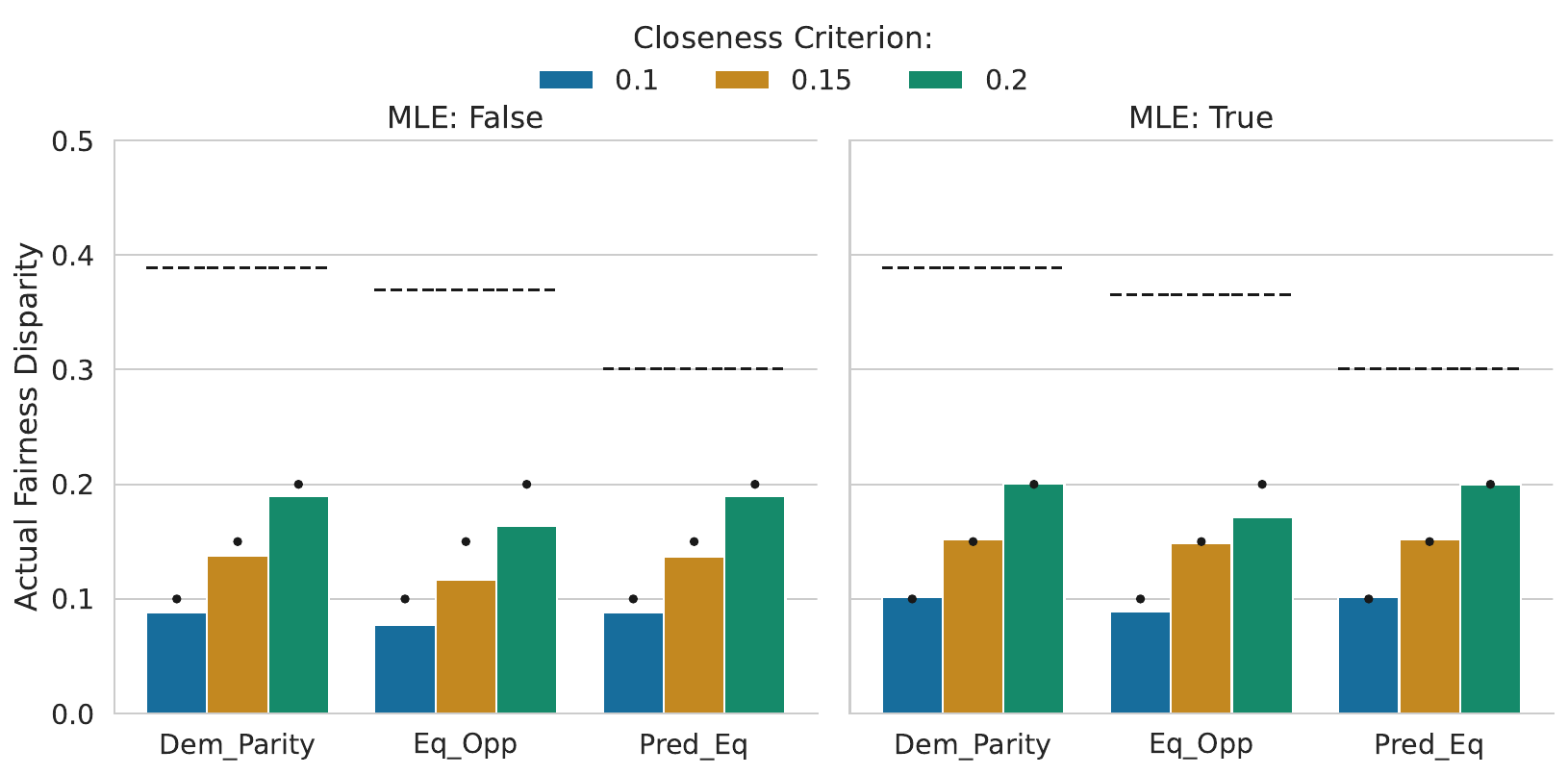}
    \end{subfigure}
    \caption{\textbf{ACSEducation (continental all, $48$ clients) using RAPS.} The left two plots show the efficiency for (a) using the interval bounds and (b) using the MLE estimate. Similarly, the two right plots show the fairness disparity (c) using the interval bounds and (d) using the MLE estimate. We observe that with the MLE estimates, FedCF achieves better efficiency at the cost of a higher worst-case fairness disparity. Both outperform the baseline in terms of fairness disparity.}
    \label{fig:mle_expt}
    \vspace{-2mm}
\end{figure}

\vspace{-3mm}
\section{Discussion}
\paragraph{On Data Heterogeneity.} In a practical federated setting, the data distribution varies across clients--resulting in data heterogeneity, which can affect performance at inference time~\citep{wen2023survey}. To address these concerns, we evaluate FedCF on varying partitioning schemes. For Fitzpatrick, we use a probabilistic partitioning scheme to ensure the data is distributed in a particular manner. For the ACS and Pokec-\{n, z\} datasets, the partitioning is \emph{naturally induced} by state and region properties, resulting in real-world data heterogeneity. Across all partitioning schemes, we show FedCF controls for the desired closeness criterion, $c$, and adapts to data heterogeneity.  

\vspace{-2mm}
\paragraph{On Data Requirements.} A limitation of CP is that, to achieve a desired coverage rate, practically, you require a large enough calibration dataset so that the interval width for the CP guarantee is tight enough. This is true in both the CF and FedCF framework as it requires sufficient calibration data for each group-positive label pair (across the federation for FedCF). For intersectional fairness, the multiplicative increase in the number of groups further increases the data requirements. 
\vspace{-2mm}
\paragraph{On Framework Extensibility.} In this work, we've presented the base FedCF framework and introduced several key axes along which FedCF can be extended. 
First, enforcing fairness can have utility (efficiency) costs, and FedCF can be extended to mitigate these costs using tighter MLE estimates instead of interval bounds or by leveraging group information at test time (Appendix~\ref{app:additional_theory}). 
Second, the data privacy can be improved in FedCF using the \textit{Enhanced Privacy} communication approach in Appendix \ref{app:fed_cf:privacy} or by adding statistical noise via Gaussian or Exponential mechanisms for formal differential privacy guarantees (Appendix \ref{app:fedcf:diff_privacy}).   
Third, we discuss how FedCF can assist with ML fairness auditing in Appendix~\ref{app:fed_cf:audit}, thereby closing the ML development loop.

\vspace{-2mm}
\section{Conclusion}
\vspace{-2mm}
In this work, we extended the Conformal Fairness framework to a federated setting, introducing the novel and comprehensive FedCF framework. We reformulated the CF framework to use a descent-based approach to make it more efficient for FL applications. Additionally, we developed theoretically grounded protocols to enable coverage gap calculations in a federated manner. The FedCF framework offers clients a choice of participation protocols, including \textit{communication efficient} and \textit{enhanced privacy} options. We conducted experiments on various non-conformity scores and datasets--including graph data, where we leverage the exchangeability assumption from CP. 

\vspace{-3mm}
\paragraph{Limitations.} This work assumes the notion of partial exchangeability, similar to FCP, which we build on top of; however, in real-world settings, this may be violated, for example, in settings where there is dynamic drift~\citep{jiang2022harmoflharmonizinglocalglobal,ramezanikebrya2023federatedlearningcovariateshifts,zec2025overcominglabelshifttargetaware}. To redress, one option we hope to explore is the use of non-exchangeable CP~\citep{Tibshirani2019,barber2023conformal} to develop fair and federated solutions~\citep{plassier2024efficientconformalpredictiondata}. 
\vspace{-3mm}
\paragraph{Future Work.}  We plan to explore using Robust FCP \citep{kang2024certifiably} in lieu of FCP and also apply the FedCF method post-hoc to work under adversarial conditions.
We will also explore how FedCF can be extended to split learning~\citep{gupta2018}.
Unlike FL, which trains full models locally and aggregates updates, split learning divides the model across clients and server, sharing only partial computations (enhanced privacy, reduced compute).

\bibliography{references}
\bibliographystyle{plainnat}
\clearpage
\appendix
\counterwithin{figure}{section}
\counterwithin{table}{section}
\counterwithin{algorithm}{section}
\renewcommand\thefigure{\thesection.\arabic{figure}}
\renewcommand\thetable{\thesection.\arabic{table}}
\renewcommand\thealgorithm{\thesection.\arabic{algorithm}}

\addtocontents{toc}{\protect\setcounter{tocdepth}{2}}

\renewcommand{\contentsname}{Appendix Table of Contents}

\begin{spacing}{0.8}
\tableofcontents
\end{spacing}
\newpage
\section{Impact Statement}\label{app:fedcf:impact_statement}
This work develops a trustworthy machine learning framework for fair uncertainty quantification and decision-making in federated settings. Using a principled mixture of clients' formulations in the federated, we can provide theoretical bounds and guarantees that translate into both rigorous fairness evaluation under uncertainty as well as guidance for ML fairness auditing. By accounting for the privacy constraints of FL, our method is suitable for use and will have positive impacts across several domains (e.g., finance and healthcare).
\section{Notation Table}
\label{app:fed_cf:notation}
\begin{table}[htbp]
    \centering
    \caption{Common notation used in FedCF.}
    \label{tab:fed_cf:notation_table}
    \begin{tabular}{cc}
        \toprule
        \textbf{Notation}  & \textbf{Definition} \\
        \midrule
         $\gG$ & The set of all demographic groups. \\
        $\gY, \gY^+$ & The set of labels and positive/advantaged labels, respectively.\\
        $g \text{ and } \tilde y$ & The group $g \in \mathcal{G}$ and positive label $\tilde y \in \mathcal{Y}^{+}$.\\ 
        $F_M$ & Filter function for fairness metric $M$.\\
        $c$ & Closeness criterion for a fairness specification.\\
        $\lambda$ & Threshold used for constructing test prediction sets.\\
        \midrule
        $\gK$ & The set of clients, $\braces{1,\dots, K}$.\\
        $\train^{(k)}/\valid^{(k)}/\calib^{(k)}$ & Client $k$'s train/validation/calibration dataset.\\
        $n_k$ and $N$ & $|\calib^{(k)}|$ and $\sum\limits_{k=1}^{K}n_k$, respectively. \\
        \midrule
        $\gS_k^{(g,\tilde y)}$ and $n_k^{(g,\tilde y)}$ & $\braces{(\vx_i, y_i) \in \calib^{(k)} \mid F_M(\vx_i, y_i, g, \tilde{y})=1}$ and $\abs{\gS_{k}^{(g,\tilde y)}}$ \\
        $\alpha_k^{(g,\tilde y); \lambda} $ & $\sum\limits_{(\vx_i, \_) \in \gS_{k}^{(g,\tilde y)} }\1\brackets{s(\vx_i,\tilde y) \leq \lambda}$ \\
        $E_k$ and $\gamma_k$ & The event $\bm x_\text{test}$ is exchangeable with data from client $k$ and $\Pr[E_k]$.\\
        \midrule
        $L^{(g, \tilde y)}, U^{(g, \tilde y)}$ & Bounds for prior (term {\scriptsize\circled{IV}}).\\
        $\pi^{(g, \tilde y)}$ & Point estimate for prior (term {\scriptsize{\circled{IV}}}).\\
        $L_\text{cov}, U_\text{cov}$ & Bounds for fairness-specific coverage level.\\
        $\Pi_\text{cov}$ & Point estimate for fairness-specific coverage level.\\
        \bottomrule
    \end{tabular}
\end{table}

\section{Conformal Fairness Background}\label{app:fed_cf:cf_background}

\noindent\textbf{From point predictions to set-based fairness.}
Let \(\gC_{\lambda}(\vx)=\{y\in\gY:\, s(\vx,y)\le \lambda\}\) denote the CP prediction set at score threshold \(\lambda\).
CF adapts classical group-fairness metrics by replacing point-prediction events (i.e., the predicted label equals the advantaged label: \(\tilde{y} = \hat{y}\)) with set-membership events (i.e., the advantaged label is in the prediction set: \(\tilde y\in \gC_{\lambda}(\vx)\)) and evaluating disparities across groups \(\mathcal{G}\) and, when appropriate, advantaged labels \( \gY^+ \). For example, a set-based Demographic Parity-style constraint can be written as,
{\small $
\left|\Pr\!\big[\tilde y\in \gC_{\lambda}(\vx)\mid \vx\in g_a\big]
-
\Pr\!\big[\tilde y\in \gC_{\lambda}(\vx)\mid \vx\in g_b\big]\right|
\;\le\; c$},  $\forall\, g_a,g_b\in\mathcal{G},\ \tilde y\in\gY^+$ with analogous set-based forms for other common group-fairness metrics as seen in Table~\ref{tab:fed_cf:conf_fairness_defs}.

\noindent\textbf{Conditional coverage as the fairness control knob.}
To evaluate a chosen fairness notion, CF filters the calibration data to the relevant subpopulation (e.g., a group or a group-and-label slice) via a filter function, $F_M$, and uses conditional coverage estimates under that filter. (For example, $F_M$ for Equal Opportunity is defined as $F_M(\vx_\text{test}, y_\text{test}, g, y):=\mathbf{1}[\vx_\text{test} \in g \cap y_\text{test}=y]$.) It then searches a threshold space \(\Lambda\) to identify \(\lambda_{\mathrm{opt}}\) that satisfies the closeness criterion across groups (and labels, if required) while maintaining CP validity. A key ingredient is that CP coverage holds when labels are fixed to a particular \(\tilde y\), which underpins group- and class-conditional control in CF.

\noindent\textbf{Guarantees and trade-offs.}
CF provides a theoretically grounded and empirically validated procedure to bound fairness disparities (Equation \eqref{eq:cf_guarantee}) without sacrificing CP’s finite-sample coverage guarantees. In practice, satisfying stricter fairness requirements (smaller closeness criterion \(c\)) increases the average prediction set size, reflecting the fairness–efficiency trade-off. 

\noindent\textbf{Practical advantages.} Unlike many conditional-CP methods that require group membership at inference time or are model-specific, CF’s set-based metrics and thresholding procedure do not require access to protected attributes at test time and apply to different non-conformity scores and data modalities, making it compatible with downstream deployment constraints.

\begin{table*}[htbp]
    \tiny
    \centering
    \caption{Formulations for Conformal Fairness Metrics. Definitions hold $\forall g_a, g_b\in\gG$ and $\forall \tilde y\in \gY^{+}$}
    \begin{tabular}{cc} 
        \toprule
        \textbf{Metric} & \textbf{Definition} \\ 
        \midrule
         Demographic (or Statistical) Parity & $\abs{\Pr\brackets{\tilde{y}\in \predict ~\Big| ~ \vx\in g_a} - \Pr\brackets{\tilde{y}\in \predict ~\Big| ~ \vx\in g_b}} < c$ \\ [2ex]
        Equal Opportunity & $\abs{\Pr\brackets{\tilde{y}\in \predict ~\Big|~ y = \tilde{y}, \vx\in g_a} - \Pr\brackets{\tilde{y}\in \predict ~\Big|~ y = \tilde{y}, \vx\in g_b}} < c$ \\ [2ex]
        Predictive Equality & $\abs{\Pr\brackets{\tilde{y}\in \predict ~\Big|~ y \neq \tilde{y}, \vx\in g_a} - \Pr\brackets{\tilde{y}\in \predict ~\Big|~ y\neq \tilde{y}, \vx\in g_b}} < c$ \\
        \bottomrule
    \end{tabular}
    \label{tab:fed_cf:conf_fairness_defs}
\end{table*} 

\section{CF Optimization Approach}\label{app:fed_cf:descent_analysis}

\begin{algorithm}[H]
\caption{Descent-Based\protect\footnotemark CF Optimization}
\label{alg:descent_cf_algo_extended}
\begin{algorithmic}[1]
\FUNCTION{\textsc{Fair\_Opt\_Descent}($\tilde y$, $\lambda_0$, $\mathcal{G}$, $c$, $F_M$, $\text{num\_rounds}$, $\eta$, $\mu$)}
    \STATE $\bar\lambda_{\text{opt};(\tilde y)} \gets 1$
    \STATE $b_0 \gets 0$
    \FOR{$t = 0$ to $\text{num\_rounds}-1$}
        \STATE $\text{cg}_t \gets \text{cg}(\lambda^{(t)}_{(\tilde y)}, F_M, \tilde y, \mathcal{G})$\label{alg:descent_cf_cg_calc}
        \STATE $u_t \gets \mathbf{1}\!\left[\text{cg}_t \le c \ \land\ \lambda^{(t)}_{(\tilde y)} < \bar\lambda_{\text{opt};(\tilde y)}\right]$
        
        \IF{$\text{cg}_t \le c$ \AND $t = 0$}
            \STATE \textbf{return} $\lambda_0$
        \ELSIF{$u_t$}
            \STATE $\bar\lambda_{\text{opt};(\tilde y)} \gets \lambda^{(t)}_{(\tilde y)}$
        \ENDIF
        
        \STATE $b_{t+1} \gets \mu\, b_t + (\text{cg}_t - c)$
        
        \IF{$u_t$ \AND $b_{t+1} > 0$}
        \label{alg_step:random_restart_start}
            \STATE $b_{t+1} \gets 0$
            \STATE $\lambda^{(t+1)}_{(\tilde y)} \sim \text{Uniform}(\lambda_0,\ \bar\lambda_{\text{opt};(\tilde y)})$
            \label{alg_step:random_restart_end}
        \ELSE
            \STATE $\Delta\lambda \gets (\bar\lambda_{\text{opt};(\tilde y)} - \lambda^{(t)}_{(\tilde y)})\mathbf{1}[b_{t+1} \ge 0] + (\lambda^{(t)}_{(\tilde y)} - \lambda_0)\mathbf{1}[b_{t+1} < 0]$ \label{alg_step:delta_lambda}
            \STATE $a_{\max} \gets \min\!\left(\eta,\; \frac{\Delta\lambda}{|b_{t+1}| + \epsilon}\right)$\label{alg_step:a_max}
            \STATE $p_t \gets \max\!\left\{\left\lceil \log_2\!\left(\frac{\eta}{a_{\max}}\right)\right\rceil,\ 0\right\}$
            \STATE $\eta_t \gets \min\!\left(\eta \cdot 2^{-p_t},\ a_{\max}\right)$
            \STATE $\lambda^{(t+1)}_{(\tilde y)} \gets \lambda^{(t)}_{(\tilde y)} + \eta_t\, b_{t+1}$ \label{alg_step:update}
        \ENDIF
    \ENDFOR
    \STATE \textbf{return} $\bar\lambda_{\text{opt};(\tilde y)}$
\ENDFUNCTION
\end{algorithmic}
\end{algorithm}
\footnotetext{We use the term descent-based since we are adaptively searching for a minimal $\lambda$ that satisfies the fairness constraint. The term descent-based does not imply the use of gradients~\citep{liu2020descent, Golovin2020Gradientless}.}

\subsection{Existence of a Solution}
By construction, our approach guarantees an admissible solution. We begin with the trivial solution where $\lambda$ is at its maximum (ensuring zero disparity), then use our coverage gap computation to verify if a new $\lambda$ satisfies the closeness criterion. To find the minimal satisfying $\lambda$, we employ a descent-based approach. Since the objective is non-convex, we mitigate local minima by performing random restarts within the admissible set whenever we reach the boundary of the admissible set (i.e., the current $\lambda_\text{opt}$). 

\subsection{Convergence Analysis}

For simplicity, suppose we remove the random restart step in Algorithm~\ref{alg:descent_cf_algo_extended}. We will show that our algorithm finds and halts at a local minimum if we remove the random restart (Lines~\ref{alg_step:random_restart_start}-\ref{alg_step:random_restart_end}). This would imply that with random restarts, the algorithm can escape and explore other local minima, (heuristically) improving performance.

\begin{theorem}
WLOG, fix some $\tilde y\in \gY^+$. Suppose at iteration $T$, the algorithm finds a valid upper-bound threshold $\bar\lambda_{\text{opt}; (\tilde y)}^{(T)}$ that satisfies $\text{cg}(\bar\lambda_{\text{opt}; (\tilde y)}^{(T)}) \leq c$. Suppose further that at a subsequent iteration $t > T$, the explored threshold $\lambda^{(t)}_{(\tilde y)} < \bar\lambda_{\text{opt}; (\tilde y)}^{(T)}$ falls into a local valley where $\text{cg}(\lambda^{(t)}_{(\tilde y)}) > c$. Then, without the random restart mechanism, the algorithm actively bounds the subsequent explored thresholds by $\bar\lambda_{\text{opt}; (\tilde y)}^{(T)}$, converging asymptotically, and permanently traps itself in the local valley.
\end{theorem}
\begin{proof}
For brevity, we will drop the $(\tilde y)$ subscript and use $\lambda_t \coloneqq \lambda^{(t)}_{(\tilde y)}$ and $\lambda_{opt} \coloneqq \bar\lambda_{\text{opt}; (\tilde y)}^{(T)}$. Because we are in a region where the constraint is violated, the error term is $e_t = \text{cg}(\lambda_t) - c > 0$. Per the update rule, the momentum $b_{t + 1} = \mu b_t + e_t$ accumulates this positive error term. Since $e_t > 0$ and $\mu \in (0, 1)$, $b_{t + 1}$ will eventually become strictly positive ($b_{t+1} > 0$). This positive momentum term indicates that the algorithm intends to step upward to satisfy the coverage gap constraint. 

However, in Lines~\ref{alg_step:delta_lambda}-\ref{alg_step:update}, the algorithm performs spatial clipping to prevent overshooting the known valid bound, enforcing $\lambda_{t+1} \leq \lambda_\text{opt}$. Let $D_t = \lambda_{opt} - \lambda_t$ be the available distance to the ceiling of our search space. In Line~\ref{alg_step:a_max}, we bound the maximum allowable step scale as:
\[
    a_{\max} = \min\braces{\eta, \frac{D_t}{b_{t+1} + \epsilon}},
\]
where $\epsilon > 0$ is a small stability term. To find the precise step scale $\eta_t$, the algorithm applies a base-2 logarithmic scaling to ensure it is the largest power of $1/2$ scaling of the base learning rate that remains less than or equal to $a_{\max}$. This discretization mathematically guarantees that:
\[
    \frac{1}{2}a_{\max} < \eta_t \leq a_{\max}.
\]
The parameter update rule is $\lambda_{t + 1} = \lambda_t + \eta_t b_{t + 1}$. Subtracting both sides from $\lambda_{opt}$, we express this in terms of the distance $D_t$:
\[
    D_{t + 1} = D_t - \eta_t b_{t + 1}.
\]
Substituting the lower bound for $\eta_t$ and the definition of $a_{\max}$, we obtain two cases: (1) $a_{\max} = \eta$ and (2) $a_{\max} = \frac{D_t}{b_{t+1} + \epsilon}$. In the first case, we get that 
\[
    D_{t + 1} = D_t - \eta\cdot b_{t + 1},
\]
meaning $\lambda_t\to\lambda_\text{opt}$ \textit{linearly}. In the second case, using the lower bound for $\eta_t$, we get
\[
    D_{t + 1} \leq D_t - \frac{1}{2}\parens{\frac{D_t}{b_{t + 1} + \epsilon}} b_{t + 1} = D_t \parens{1 - \frac{b_{t + 1}}{2(b_{t + 1} + \epsilon)}}.
\]
Let $\kappa = 1 - \frac{b_{t + 1}}{2(b_{t + 1} + \epsilon)}$. Since $b_{t + 1} > 0$ and $\epsilon > 0$, it strictly holds that $0 < \kappa < 1$. Therefore, $D_{t + 1} \leq \kappa D_t$. Thus, $\lambda_t\to \lambda_\text{opt}$ at a \textit{geometric} rate. Thus, without the random restart, in either case, the algorithm eventually gets stuck at some (local) valley.
\end{proof}

\subsection{Empirical Justification of Design Choice}
To justify the novel descent-based approach for the federated setting in terms of reducing the number of communications rounds, we define the discrete search space for the iterative method as $\Lambda = \texttt{linspace}(\hat{q}(\alpha), 2, \text{num\_rounds})$, where $\hat{q}(\alpha)$ is the minimal lambda to satisfy $1-\alpha$ coverage of standard federated CP, and 2 is the upper bound of the RAPS non-conformity score. Table~\ref{tab:acs_income_raps_opt_vs_iter} shows that the iterative (grid-search) approach requires more communication rounds to achieve performance comparable to our new descent-based approach. This is because discretizing the space limits the precision of the $\lambda$ values, so, by using fewer rounds, we find a less precise, and larger (more conservative) $\lambda$. This is not a limitation for our descent-based approach, which converges to a constraint-satisfying $\lambda$ using our adaptive update rule. The results in Table~\ref{tab:acs_income_raps_opt_vs_iter} show that the Iterative approach requires 1000 rounds to achieve sufficient granularity and match the performance of the Descent-Based approach, which only used 100 rounds. Thus, we demonstrate approximately $10\times$ speedup to achieve comparable performance, justifying the descent-based adaptation for the federated learning setting.

\begin{table}[htbp]
    \centering
    \caption{\textbf{ACSIncome (small, $4$ clients) using RAPS.} Each entry is of the form, \textbf{efficiency/fairness disparity}. We compare the Descent-Based (with 100 communication rounds) and the Iterative (grid-search) method (with 100 and 1000 communication rounds) across different closeness criteria ($c$).}
    \label{tab:acs_income_raps_opt_vs_iter}
    \begin{tabular}{lccc}
        \toprule
        \multirow{2}{*}{Method}& \multicolumn{1}{c}{$c = 0.1$} & \multicolumn{1}{c}{$c = 0.15$} & \multicolumn{1}{c}{$c = 0.2$} \\
        \cmidrule(lr){2-2} \cmidrule(lr){3-3} \cmidrule(lr){4-4}
         (Number of rounds) & Eff / Disp & Eff / Disp & Eff / Disp \\
        \midrule
        \textbf{Descent-Based} (100) & 3.127 / {0.096} & 2.977 / {0.179} & 2.816 /{0.257} \\
        \textbf{Iterative} (100)     & 3.239 / {0.119} & 3.137 / {0.139} & 2.890 / {0.213} \\
        \textbf{Iterative} (1000)    & 3.120 / {0.137} & 2.989 / {0.170} & 2.821 / {0.251} \\
        \bottomrule
    \end{tabular}
    \vspace{-1em}
\end{table}

\subsection{Example: Synthetic Adversarial Coverage Gap Landscape}
The coverage gap values that we encountered in our experiments demonstrated a piecewise-monotonic behavior~\citep{milnor2006iterated, guckenheimer1979sensitive, zhu2021piecewise}, meaning there are a small, but finite number of turning points (see Figure~\ref{fig:quasi_monotonic}). To demonstrate that our algorithm works in more adversarial settings, we provide an experiment with a synthetic, high-oscillation, coverage gap function and show that Algorithm~\ref{alg:descent_cf_algo_extended} finds a satisfactory solution (see Figure~\ref{fig:random_restart}).

\begin{figure}[htbp]
    \centering
        \includegraphics[width=\linewidth]{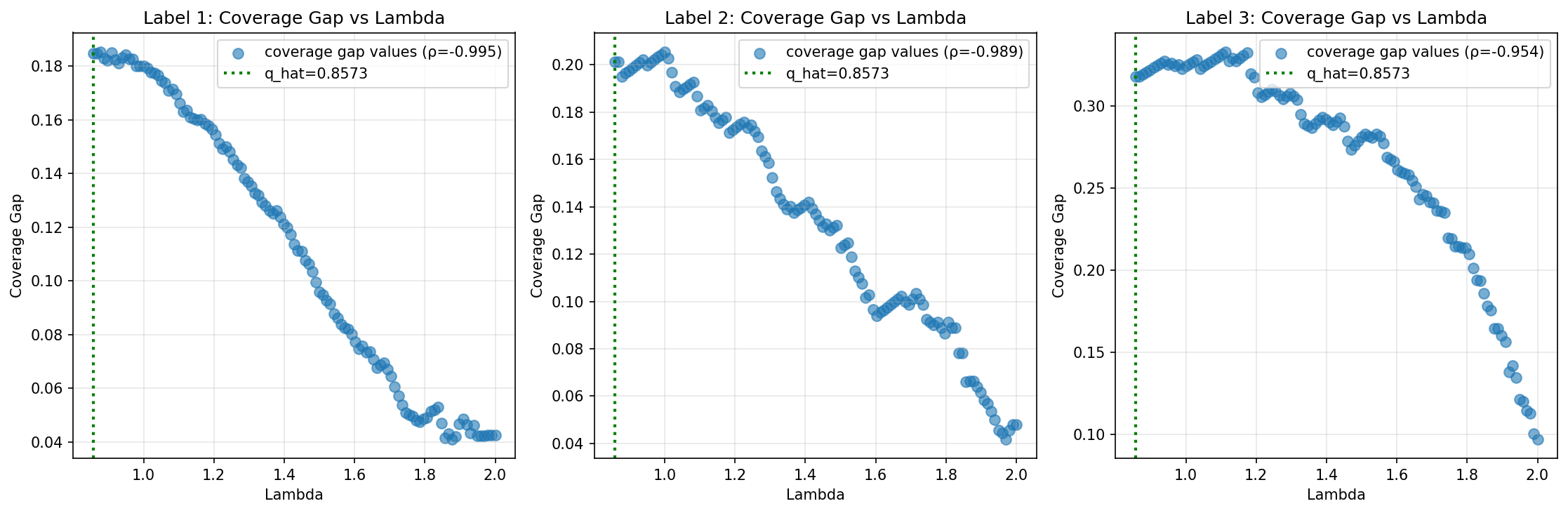}
        \caption{\textbf{ACSIncome (small, $4$ clients).} These plots show how the coverage gap (cg\_curr) varies as we increase $\lambda$ starting from $\hat{q}$. As we can see, the relationship is nearly monotonic with $|\rho| > 0.95$ for all labels.}
    \label{fig:quasi_monotonic}
\end{figure}


\begin{figure}[htbp!]
    \centering
    \includegraphics[width=\linewidth]{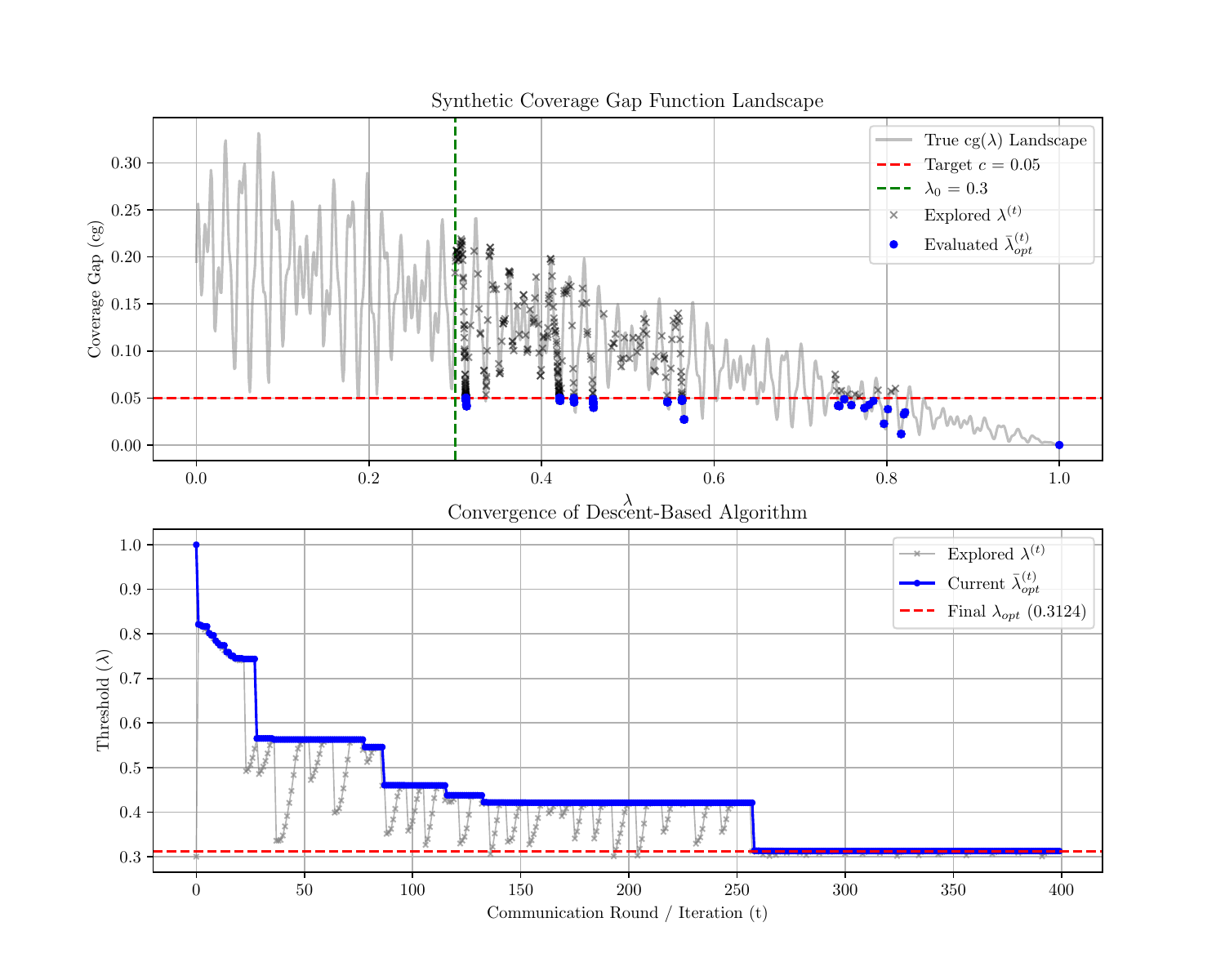}
    \caption{We run Algorithm~\ref{alg:descent_cf_algo_extended}, which includes \textit{random restarts}--the mechanism we use to continue searching once a new $\bar\lambda_{\text{opt}}$ is found--on a synthetic, jagged function, representing an adversarial coverage gap function. Specifically, we use $\mu = 0.9$, $\eta = 0.1$. \textbf{Note:} we use $\text{cg}(\lambda)$ for notational brevity instead of $\text{cg}(\lambda, F_M, \tilde{y}, \mathcal{G})$ as $F_M, \tilde{y},~\text{and}~ \mathcal{G}$ are implicitly fixed in the synthetic function.}
    \label{fig:random_restart}
\end{figure}
\clearpage
\section{Proofs}
\label{app:fed_cf:proofs}
\subsection{Proof of Theorem \ref{thm:main_cov_thm}}
Recall, since the data is distributed across clients in the federated setting, we reformulated the fairness-specific coverage level as Equation \ref{eq:dissected_fair_cov}. In doing so, the computation of the coverage level is split between the clients and the server. We present bounds and point estimates for each of the terms in Equation \ref{eq:dissected_fair_cov} across Lemmas \ref{lem:client_fair_cov}, \ref{lem:F_M_bound}, and \ref{lem:prior_term}, leading to a proof of Theorem \ref{thm:main_cov_thm}.



\subsubsection{Client-Side Estimates}
Since each client operates independently with its own dataset, we can derive interval bounds for terms $\circled{I}$ and $\circled{II}$ assuming \emph{partial exchangeability}. For the point estimates approach, we use maximum likelihood estimators (MLEs) for each term, providing the tightest estimates assuming \emph{partial IID}.

\begin{restatable}{lemma}{lemClientFairCov}\label{lem:client_fair_cov}
For each client $k$, group $g$, positive label $\tilde{y}$, and threshold $\lambda$, we get the following interval bounds:
    \begin{equation}\textstyle
        \frac{\alpha_k^{(g,\tilde y); \lambda}}{n_k^{(g,\tilde y)}+1} \leq \Pr\brackets{s(\vx_{\text{test}},\tilde{y})\leq \lambda~|~F_M(\vx_{\text{test}}, y_{\text{test}}, g, \tilde{y})=1,E_k} \leq \frac{\alpha_k^{(g,\tilde y); \lambda} + 1}{n_k^{(g,\tilde y)}+1}
    \end{equation}
If the data satisfy the partial-IID assumption, then we can use an MLE point estimate, given by the following:
    \begin{equation}\textstyle
        \Pr\brackets{s(\vx_{\text{test}},\tilde{y})\leq \lambda~|~F_M(\vx_{\text{test}}, y_{\text{test}}, g, \tilde{y})=1,E_k} = \frac{\alpha_k^{(g,\tilde y); \lambda}}{n_k^{(g,\tilde y)}}
    \end{equation}
\end{restatable}
The proof of Lemma \ref{lem:client_fair_cov} is as follows: 
\begin{proof}
    We first observe that,
    \begin{align*}
        &\Pr\brackets{s(\bm{x}_{\text{test}},\tilde{y})\leq \lambda~|~F_M(\bm{x}_{\text{test}}, y_{\text{test}}, g, \tilde{y})=1,~E_k} \\
        &= \underset{\bm{x}_{\text{test}} \sim {P}_k}{\Pr}\brackets{s(\bm{x}_{\text{test}},\tilde{y})\leq \lambda~|~F_M(\bm{x}_{\text{test}}, y_{\text{test}}, g, \tilde{y})=1}, 
    \end{align*}
    since exchangeability with the elements in $k$ is true iff $\bm{x}_{\text{test}}$ is sampled from $k$'s local distribution, $P_k$. The interval bounds follow from the conditional coverage guarantees given in CF~\citep{vadlamani2025a}.    


    For the point estimate, we can model the event that the predicted score $s(\vx_{test}, \tilde{y})$ falls below $\lambda$ as a Bernoulli random variable with success probability $p$. We can treat the $n_k^{(g,\tilde{y})}$ calibration points as individual Bernoulli trials, to then construct a maximum likelihood estimate (MLE) for $p$, which will be $\hat{p} = \frac{\alpha_k^{(g,\tilde{y}); \lambda}}{n_k^{(g,\tilde{y})}}$.
\end{proof}

Lemma \ref{lem:client_fair_cov} bounds the fair-conditional coverage for a particular group-label pair for the test covariate $(\bm x_\text{test}, y_\text{test})$. We next bound the coverage of the test covariate satisfying the Fairness Metric ($F_M$), conditioned on the test point being exchangeable with data from client $k$ using Lemma \ref{lem:F_M_bound}, and provide the proof below.

\subsubsection{Server-Side Estimates}
Terms {\scriptsize \circled{III}} and {\scriptsize \circled{IV}} require a global view of the clients' data, so they are handled on the server.

For term {\scriptsize \circled{III}}, we follow the setup by \citet{lu2023federated}, where given $n_k = \abs{\calib^{(k)}}$, $\gamma_k\coloneqq\Pr[E_k] \propto n_k + 1$ and $\sum\limits_{k = 1}^K \gamma_k = 1$. Finally, for term {\scriptsize \circled{IV}}, we have that
\[\textstyle
    \Pr\brackets{F_M(\vx_{\text{test}}, y_{\text{test}}, g, \tilde{y}) = 1} = \sum\limits_{k = 1}^K \Pr\brackets{F_M(\vx_{\text{test}}, y_{\text{test}}, g, \tilde{y}) = 1 \mid E_k}\cdot \Pr[E_k].
\]
Using Lemma \ref{lem:F_M_bound}, we can get an interval-bound and point-estimate as shown in the following lemma.

\begin{restatable}{lemma}{lemFMBound}\label{lem:F_M_bound}
For each client $k$, group $g$, and positive label $\tilde{y}$, we get the following interval bounds:
    \begin{equation}\textstyle
        \frac{n_k^{(g,\tilde y)}}{n_k+1} \leq \Pr\brackets{F_M(\vx_{\text{test}}, y_{\text{test}}, g, \tilde{y})=1\mid E_k} \leq \frac{n_k^{(g,\tilde y)}+1}{n_k+1}
    \end{equation}
If the data satisfy the partial-IID assumption, then we can use an MLE point estimate, given by the following:
    \begin{equation}\textstyle
        \Pr\brackets{F_M(\vx_{\text{test}}, y_{\text{test}}, g, \tilde{y})=1\mid E_k} = \frac{n_k^{(g,\tilde y)}}{n_k}
    \end{equation}
\end{restatable}
\begin{proof}
    To demonstrate the finite sample guarantee, we note that $F_M(\vx, y, g, \tilde{y})=1$ for all $(\vx, y)\in\calib^{(k)}$ are all exchangeable Bernoulli trials. Observe that conditioning on $E_k$ implies ${\calib}_{+}^{(k)} \coloneqq \calib^{(k)}\cup\braces{\bm{x}_{\text{test}}}$ is an exchangeable sequence of length $n_k + 1$. Treating this as a finite `bag' of covariates, we have 
   $$\forall_{(\vx, y)\in {\calib}_+^{(k)}},~~\Pr\brackets{F_M(\bm{x}, y, g, \tilde{y})=1~|~ E_k} =\frac{\sum_{(\bm x_i, y_i)\in {\calib}_+^{(k)}}F_M(\bm{x_i}, y_i, g, \tilde{y})}{n_k+1}.$$
   In other words, we have defined the probability of randomly selecting a covariate with $F_M(\vx, y, g, \tilde{y})=1$. Since this applies to all points we know, 
    \begin{equation}\label{eq:intermediate_expression}
        \Pr\brackets{F_M(\bm{x}_{\text{test}}, y_{\text{test}}, g, \tilde{y})=1~|~ E_k} =\frac{ \sum\limits_{(\bm x_i, y_i)\in {\calib}_{+}^{(k)}}F_M(\bm{x}_i, y_i, g, \tilde{y})}{n_k+1}.
    \end{equation} 
    Since we implicitly condition on $F_M(\bm{x}, y, g, \tilde{y}),~~\forall (\bm{x}, y) \in k$ (effectively making them deterministic), we can calculate the following bounds,
    \begin{footnotesize}
    \begin{equation}
        \frac{\underset{(\bm x_i, y_i)\in \calib^{(k)}}{\sum}F_M(\bm{x_i}, y_i, g, \tilde{y})}{n_k+1} \leq \Pr\brackets{F_M(\bm{x}_{\text{test}}, y_{\text{test}}, g, \tilde{y})=1~|~ E_k} \leq \frac{\underset{(\bm x_i, y_i)\in \calib^{(k)}}{\sum}F_M(\bm{x_i}, y_i, g, \tilde{y}) + 1}{n_k+1}, 
    \end{equation}
    \end{footnotesize}

    where the $+1$ term comes from the unknown value of $F_M(\bm{x}_{\text{test}}, y_{\text{test}}, g, \tilde{y})$. Substituting the sums with $n_k^{(g,\tilde y)}$, proves the interval bounds.

    For the point estimate, the event that $F_M = 1$ can be modeled as a Bernoulli random variable with success probability $p$. We can use the full $n_k$ calibration points as $n_k$ Bernoulli trials to construct an MLE for $p$, which will be $\hat{p} = \frac{n_k^{(g,\tilde{y})}}{n_k}$.
\end{proof} 

Lastly, we use Lemma \ref{lem:prior_term}, to bound $\Pr\brackets{F_M(\vx_{\text{test}}, y_{\text{test}}, g, \tilde{y})=1}$. The proof of Lemma \ref{lem:prior_term} uses the result of Lemma \ref{lem:F_M_bound}.

\begin{restatable}{lemma}{lemPriorTerm}\label{lem:prior_term}
For each client $k$, group $g$, and positive label $\tilde{y}$, we get the following interval bounds:
    \begin{equation}\textstyle
        L^{(g,\tilde y)} =  \sum\limits_{k=1}^{K} \gamma_k\frac{n_k^{(g,\tilde y)}}{n_k+1} \leq \Pr\brackets{F_M(\vx_{\text{test}}, y_{\text{test}}, g, \tilde{y})=1} \leq \sum\limits_{k=1}^{K} \gamma_k\frac{n_k^{(g,\tilde y)} + 1}{n_k+1} = U^{(g,\tilde y)}.
    \end{equation}
If the data satisfy the partial-IID assumption, then we can use an MLE point estimate, given by the following:
    \begin{equation}\textstyle
        \Pr\brackets{F_M(\vx_{\text{test}}, y_{\text{test}}, g, \tilde{y})=1} = \sum\limits_{k=1}^{K} \gamma_k\frac{n_k^{(g,\tilde y)}}{n_k} = \pi^{(g,\tilde y)}.
    \end{equation}
\end{restatable}
\begin{proof}
    To achieve this result, we first use the law of total probability to decompose $\Pr\brackets{F_M(\bm{x}_{\text{test}}, y_{\text{test}}, g, \tilde{y}) = 1}$ into terms known by the server and the client:  
    \begin{equation}\label{eq:lemma_intermediate_step}
    \Pr\brackets{F_M(\bm{x}_{\text{test}}, y_{\text{test}}, g, \tilde{y}) = 1} = \sum_{k=1}^K\Pr\brackets{F_M(\bm{x}_{\text{test}}, y_{\text{test}}, g, \tilde{y})=1~|~ E_k}\cdot \Pr\brackets{E_k}
    \end{equation} 

    Then, substituting the bounds for term {\scriptsize \circled{II}} (see Lemma \ref{lem:F_M_bound}), and $\gamma_k = \Pr\brackets{E_k}$ from term {\scriptsize \circled{III}} into Equation \ref{eq:lemma_intermediate_step}, we complete the proof. 
\end{proof}

On closer inspection, we observe that Terms {\scriptsize \circled{I}} and {\scriptsize \circled{II}} can be combined and bound together to provide a tighter lower bound. 

\begin{lemma}\label{lem:combined_1_2}
        Using the definitions of $\alpha_k^{(g,\tilde y); \lambda}$ and $n_k+1$ we have, 
        \begin{align}
            &\frac{\alpha_k^{(g,\tilde y); \lambda}}{n_k+1} \leq \Pr\brackets{s(\bm{x}_{\text{test}},\tilde{y})\leq \lambda~|~F_M(\bm{x}_{\text{test}}, y_{\text{test}}, g, \tilde{y})=1,~\bm{x}_{\text{test}} \overset{\text{exc.}}{\sim} k}\nonumber\\
            &~~~~~~~~~~~~~~~~~~~~~~~~~~~~~~~~\cdot\Pr\brackets{F_M(\bm{x}_{\text{test}}, y_{\text{test}}, g, \tilde{y})=1~|~ \bm{x}_{\text{test}} \overset{\text{exc.}}{\sim} k} \leq \frac{\alpha_k^{(g,\tilde y); \lambda}+1}{n_k+1}. \label{eq:term1_2_comb_bounds}
        \end{align}
\end{lemma}
\begin{proof}
    First, observe that, 
    \begin{align*}
        &{\Pr\brackets{s(\bm{x}_{\text{test}},\tilde{y})\leq \lambda~|~F_M(\bm{x}_{\text{test}}, y_{\text{test}}, g, \tilde{y})=1,~E_k}\cdot\Pr\brackets{F_M(\bm{x}_{\text{test}}, y_{\text{test}}, g, \tilde{y})=1~|~E_k}} \\ &=\Pr\brackets{s(\bm{x}_{\text{test}},\tilde{y})\leq \lambda,~F_M(\bm{x}_{\text{test}}, y_{\text{test}}, g, \tilde{y})=1~|~E_k}
    \end{align*}
    Then consider the following Bernoulli random variables (R.V) $\1\brackets{s(\bm{x},y) \leq \lambda}\cdot F_M(\bm{x}, y, g, \tilde{y})$ for all $(\bm{x},y) \in k \cup \{(\bm{x}_{\text{test}},y_{\text{test}})\}$ which form an exchangeable sequence (using the assumption $\bm{x}_{\text{test}} \overset{\text{exc.}}{\sim} k$). Additionally, observe $\alpha_k^{(g,\tilde y); \lambda} = \sum_{(\bm{x}_i,y_i)\in k}\1\brackets{s(\bm{x},y) \leq \lambda}\cdot F_M(\bm{x}, y, g, \tilde{y})$ is an equivalent definition of $\alpha_k^{(g,\tilde y); \lambda}$. The rest of the proof follows from the proof of Lemma \ref{lem:F_M_bound} by using  $\1\brackets{s(\bm{x},y) \leq \lambda}\cdot F_M(\bm{x}, y, g, \tilde{y})$ as the Bernoulli R.V instead of $F_M(\bm{x}, y, g, \tilde{y})$ and $\alpha_k^{(g,\tilde y); \lambda}$ in place of $n_k^{(g,\tilde y)}$. 
\end{proof}

Having proved the Lemmas, we can move on to proving Theorems \ref{thm:main_cov_thm} and \ref{thm:main_cov_thm_mle}:
\thmMainCovThm*
\begin{proof}
Substituting the bounds for terms {\scriptsize \circled{I}} and {\scriptsize \circled{II}} via Lemma \ref{lem:combined_1_2} and {\scriptsize \circled{IV}} which were established via Lemmas \ref{lem:F_M_bound} and \ref{lem:prior_term}, respectively, and the definition of {\scriptsize \circled{III}} into Equation \ref{eq:dissected_fair_cov} completes the proof. 
\end{proof}

\thmMainCovThmMLE*
\begin{proof}
Substituting the estimates for terms {\scriptsize \circled{I}} and {\scriptsize \circled{II}} via Lemma \ref{lem:combined_1_2} and {\scriptsize \circled{IV}} which were established via Lemmas \ref{lem:F_M_bound} and \ref{lem:prior_term}, respectively, and the definition of {\scriptsize \circled{III}} into Equation \ref{eq:dissected_fair_cov} completes the proof. 
\end{proof}

\subsection{Sufficient Data Counts}
Below, we consider the sufficient number of data points to get a non-degenerate solution (i.e., $\lambda_{\text{opt}; {y}} = \lambda_\text{max}$).
\thmIntervalWidthBound*
\begin{proof}
Define the interval width for a particular group label pair as, 
\begin{align}
\label{eq:client_return_form_iw}
    \textrm{IW}(\lambda, F_M, g, \tilde{y})&\coloneqq U_\text{cov}(\lambda, F_M, g, \tilde{y})-L_\text{cov}(\lambda, F_M, g, \tilde{y})\\
    &= \sum\limits_{k=1}^{K} \gamma_k \left(\frac{(\alpha_k^{(g,\tilde y); \lambda}+1)}{(n_k+1)L^{(g,\tilde y)}} - \frac{\alpha_k^{(g,\tilde y); \lambda}}{(n_k+1)U^{(g,\tilde y)}}\right),
\end{align}
where $U_\text{cov}$ and $L_\text{cov}$ are from Theorem \ref{thm:main_cov_thm}. We note that we can achieve non-vacuous bounds for each $\tilde{y} \in \mathcal{Y}^+$ if for each $g \in \mathcal{G}$, $\textrm{IW}(\lambda, F_M, g, \tilde{y}) < c$ for some $\lambda <\lambda_{max}$. 
Observe that if the interval width is greater than $c$ for all $\lambda$, then our algorithm cannot find non-vacuous bounds. Also, recall by definition, $L^{(g,\tilde y)} = N^{(g,\tilde y)}$ and $U^{(g,\tilde y)} = N^{(g,\tilde y)} + K$
Then consider,
\begin{align}
    \textrm{IW}(\lambda, F_M, g, \tilde{y})&=\sum\limits_{k=1}^{K} \gamma_k \left(\frac{(\alpha_k^{(g,\tilde y); \lambda}+1)}{(n_k+1)L^{(g,\tilde y)}} - \frac{\alpha_k^{(g,\tilde y); \lambda}}{(n_k+1)U^{(g,\tilde y)}}\right) \\
    &=\frac{1}{N+K}\sum\limits_{k=1}^{K} \left(\frac{(\alpha_k^{(g,\tilde y); \lambda}+1)}{L^{(g,\tilde y)}} - \frac{\alpha_k^{(g,\tilde y); \lambda}}{U^{(g,\tilde y)}}\right)\\
    &=\sum\limits_{k=1}^{K} \left(\frac{(\alpha_k^{(g,\tilde y); \lambda}+1)}{N^{(g,\tilde y)}} - \frac{\alpha_k^{(g,\tilde y); \lambda}}{N^{(g,\tilde y)} +K}\right) \label{eq:interm_step}
\end{align}
Given $\sum_{k=1}^{K}\alpha_k^{(g,\tilde y)} \le \sum_{k=1}^{K}n_k^{(g,\tilde y)} = N^{(g,\tilde y)}$, we have
\begin{equation}
    \left(\frac{N^{(g,\tilde y)} - \sum_{k=1}^{K}\alpha_k^{(g,\tilde y);\lambda}}{N^{(g,\tilde y)}} - \frac{N^{(g,\tilde y)} - \sum_{k=1}^{K}\alpha_k^{(g,\tilde y);\lambda}}{N^{(g,\tilde y)}+K}\right) \ge 0.
\end{equation}
Thus,
\begin{align}
    \eqref{eq:interm_step}&\le\scriptstyle\left[\sum\limits_{k=1}^{K} \left(\frac{(\alpha_k^{(g,\tilde y); \lambda}+1)}{N^{(g,\tilde y)}} - \frac{\alpha_k^{(g,\tilde y); \lambda}}{N^{(g,\tilde y)} +K}\right) + \left(\frac{N^{(g,\tilde y)} - \sum_{k=1}^{K}\alpha_k^{(g,\tilde y);\lambda}}{N^{(g,\tilde y)}} - \frac{N^{(g,\tilde y)} - \sum_{k=1}^{K}\alpha_k^{(g,\tilde y);\lambda}}{N^{(g,\tilde y)}+K}\right)\right]\\
    &=\left(\frac{N^{(g,\tilde y)} + K}{N^{(g,\tilde y)}} - \frac{N^{(g,\tilde y)}}{N^{(g,\tilde y)}+K}\right)\\
    &=\frac{2N^{(g,\tilde y)}K + K^2}{{N^{(g,\tilde y)}}^2 + N^{(g,\tilde y)}K} \le c \label{eq:deriv_final_result_bound}
\end{align}

Thus, we have the interval width, and we want to bound it by $c$ as in $\eqref{eq:deriv_final_result_bound} < c$. Solving for the quadratic form and taking the positive root, we arrive at, 
\begin{align}\label{eq:prelim_valid_bounds}
    N^{(g,\tilde y)} &\ge \frac{K}{2c} \left(2 - c(N+K) + \sqrt{c^2(N+K)^2 + 4}\right)
\end{align}
While \eqref{eq:prelim_valid_bounds} is a sufficient quantity for the minimal count, a tight overbound of \eqref{eq:prelim_valid_bounds} (for more interpretable bounds) follows as, 
\begin{align}\label{eq:prelim_interp_bounds}
    &\frac{K}{2c} \left(2 - c(N+K) + \sqrt{c^2(N+K)^2 + 4}\right)\\
    &\le \frac{K}{2c} \left(2 - c(N+K) + \sqrt{c^2(N+K)^2 +4c(N+K) +4}\right)\\
    &= \frac{K}{2c} \left(2 - c(N+K) + 2 + c(N+K)\right)\\
    &=\frac{2K}{c} \le N^{(g,\tilde y)}
\end{align}
\end{proof}


\section{Additional Experiment Details}
\label{app:fed_cf:experimental_details}
\subsection{Datasets} \label{app:datasets}
We present a summary of common dataset statistics in Table \ref{tab:fed_cf:datasets} and go into more details on each dataset in the following sections. Additionally, the Folktables and Fitzpatrick datasets are licensed under CC BY 4.0. The Pokec dataset in the SNAP repository is licensed under the BSD license. 

\begin{table}[ht]
    \small
    \centering
    \caption{Dataset Statistics. T refers to Tabular, G refers to Graph, and V refers to vision.\\$^*$ACS datasets have six (6) groups if using the \textit{continental} split schemes (see Section \ref{app:subsec:folktables}).\\\^{} Number of inputs after removing those with unknown group information}
    \label{tab:fed_cf:datasets}
    \begin{tabular}{cccccc}
        \toprule
         \textbf{Name} & \textbf{Type} & \textbf{Size} & \textbf{\# Labeled} & \textbf{\# Groups} & \textbf{\# Classes}\\
         \midrule
         ACSIncome & T & $1,664,500$  & ALL & race$(9)^*$ & $4$\\
         ACSEducation & T & $1,664,500$ & ALL & race$(9)^*$ & $6$\\
         Fitzpatrick & V & $16,012$\^{} & ALL & skin type$(6)$ & $9$\\
         \midrule
         \textbf{Name} & \textbf{Type} & \bm{$(|\gV|,|\gE|)$} & \textbf{\# Labeled} & \textbf{\# Groups} & \textbf{\# Classes}\\
         \midrule
         Pokec-\{n, z\} & G & $(133,138,~1,458,258)$ & $17,594$ & region$(2)$,~gender$(2)$ & $4$\\
         \bottomrule
    \end{tabular}
\end{table}

\subsection{Folktables Datasets}
\label{app:subsec:folktables}
In the fairness space, the American Community Services (ACS) datasets from the \texttt{Folktables} library are a widely used set of tabular data \citep{ding2021retiring}. The data is taken across the 51 U.S. states and territories. For our federated setup and each dataset below, we consider the following 6 partitioning schemes:
\begin{enumerate}[leftmargin=1.1cm]
    \item[(1.)] \textbf{All:} We consider each U.S. state and territory to be its own client
    \item[(2.)] \textbf{Large:} We follow the Bureau of Economic Analysis's division of the U.S. into New England, the Mideast, the Great Lakes, the Plains, the Southeast, the Southwest, the Rocky Mountain, and the Far West.
    \item[(3.)] \textbf{Small:} We follow the U.S. Census Bureau's division of the U.S. into the Northeast, the Midwest, the South, and the West
    \item[(4-6.)] \textbf{Continental All, Continental Large, Continental Small:} The same as 1 to 3, but we only consider the \textit{continental} U.S.--removing Alaska, Hawaii, and Puerto Rico.
\end{enumerate}

All Folktable datasets have a race attribute. When we partition the data using all the states and territories, we use the full version of race, which has 9 groups. However, when partitioning just with \textit{continental} U.S., we combine some demographic groups--primarily those from Alaska, Hawaii, and Puerto Rico--into the appropriate `Other' categories, resulting in a total of 6 groups.

\paragraph{ACSIncome:}  We used the standard ACSIncome dataset from Folktables; however, we divided the targets into four classes by evenly splitting the income into 4 brackets. The sensitive attribute in this case is race, resulting in either $9$ or $6$ groups.

\paragraph{ACSEducation:} This is a custom dataset. We used the ACSTravelTime data and selected Education Level as the target. The education level was divided into $6$ groups: \{did not complete high school, has a high school diploma, has a GED, started an undergrad program, completed an undergrad program, and completed graduate or professional school\}. ACSEducation also uses race as a sensitive attribute.

\subsection{Non-Tabular Datasets}
\paragraph{Pokec-\{n,z\}:} The Pokec-\{n, z\} dataset \citep{takac2012data} is a social network graph dataset collected from Pokec, a popular social network in Slovakia. Since several rows in the dataset are missing features, two commonly used subgraphs are the Pokec-z and Pokec-n datasets. The graphs have four labels corresponding to the fieldwork and two sensitive attributes: gender ($2$ groups) and region ($2$ groups). Our experiments consider each attribute individually as well as intersectional fairness by creating an attribute with $4$ groups. For our federated setup, we use each subgraph as a single client, resulting in $2$ clients.

\paragraph{Fitzpatrick:} The Fitzpatrick dataset \citep{groh2021evaluating} contains clinical images classified based on the depicted skin condition. There are several levels of granularity regarding the skin condition label. We use a version with 9 skin conditions: \{inflammatory, malignant epidermal, genodermatoses, benign dermal, benign epidermal, malignant melanoma, benign melanocyte, malignant cutaneous lymphoma, malignant dermal\}. There are 6 demographic groups based on the Fitzpatrick skin type. For our federated setup, we use a Dirichlet partitioner to split the data into $K\in \{2, 4, 8\}$ clients.

\subsection{Hyperparameters and Implementation}
To promote reproducibility, the source code for FedCF is provided in the supplementary material, along with the configuration files containing the hyperparameters used. All experiments were conducted on at least three seeds. The error bars were all too small to plot.

The project was written using the Flower AI Federated Learning framework~\citep{beutel2020flower} for both base model training and the FedCF framework. 

\paragraph{Compute Resources:} \label{par:fed_cf:compute_resources}Model training with the ACS and Pokec datasets was run on 1 A6000 GPU, and FedCF was run on 8 AMD EPYC 7313 CPUs. For the Fitzpatrick experiments, model training was done on 1 V100 GPU, and FedCF was run on 8 Dual Intel Xeon 8268 CPUs.

\subsection{Non-Conformity Scores} \label{app:scores}
\paragraph{Adaptive Prediction Sets (APS)} The most popular CP method for classification problems is APS~\citep{romano2020classification}. The scoring function first sorts the softmax logits in descending order and accumulates the class probabilities until the correct class is included. For tighter prediction sets, randomization is introduced through a uniform random variable. 

Formally, let $\hat{\pi}$ be a trained classification model with softmaxed output. If $\hat{\pi}(\vx)_{(1)}\geq \hat{\pi}(\vx)_{(2)}\geq \dots \geq \hat{\pi}(\vx)_{(C)}$, $u\sim U(0, 1)$, and $r_y$ is the rank of the correct label, then
\[
    s(\vx, y) = \brackets{\sum_{i = 1}^{r_y} \hat{\pi}(\vx)_{(i)}} - u\hat{\pi}(\vx)_y.
\]

APS has two major drawbacks that have led to it being surpassed by other methods in recent CP literature. First, APS tends to produce large (less efficient) prediction sets. Second, it does not account for structure in its formulation. To address these issues, alternatives like RAPS and DAPS have emerged\footnote{RAPS and DAPS have hyperparameters typically tuned on separate held-out data, but we fix them \textit{a priori} to preserve data for calibration and evaluation as well as to be consistent with what prior federated conformal prediction works have done.}.

\paragraph{Regularized Adaptive Prediction Sets (RAPS)} \citet{angelopoulos2022uncertaintysetsimageclassifiers} introduces a regularization approach for APS. Given the same setup and notation as APS, define $o(\vx, y) = \abs{\braces{c\in\mathcal{Y} : \hat{\pi}(\vx)_c\geq \hat{\pi}(\vx)_y}}$. Then,

\[
    s(\vx, y) = \brackets{\sum_{i = 1}^{r_y} \hat{\pi}(\vx)_{(i)}} - u\hat{\pi}(\vx)_y + \nu\cdot\max\braces{\parens{o(\vx, y) - k_{reg}}, 0},
\]
where $\nu$ and $k_{reg}\geq 0$ are regularization hyperparameters.

\paragraph{Diffusion Adaptive Prediction Sets (DAPS)} Graphs are rich with neighborhood information, with nodes often exhibiting homophily. This suggests that the non-conformity scores of connected nodes are likely to be related. To leverage this insight, DAPS~\cite{zargarbashi23conformal} incorporates a one-step diffusion update on the non-conformity scores. Formally, if $s(\vx, y)$ is a point-wise score function (e.g., APS), then the diffusion step yields a new score function
\[
    \hat{s}(\vv, \vx_v, y) = (1 - \delta) s(\vx_v, y) + \frac{\delta}{|
    \gN_\vv|} \sum\limits_{\vu \in \gN_\vv} s(\vx_u, y),
\]
where $\delta\in[0, 1]$ is a diffusion hyperparamter and $\gN_\vv$ is the $1$-hop neighborhood of $\vv$.
\section{On Using Group Information}\label{app:additional_theory}

\noindent\textbf{(Group-Classwise)-Formulation} A user can alter FedCF to use group information, such that $\lambda_\text{opt} \in \mathbb{R}^{|\mathcal{Y}|\times|\mathcal{G}|}$. Recall from Section \ref{subsubseq:improve_eff}, for a given label $y$ the threshold vector is $\lambda_{\text{opt};(y)} \in \mathbb{R}^{|\mathcal{G}|}$ and $\Lambda^{|\mathcal{G}|}$ defined as $|\mathcal{G}|$ cartesian products of $\Lambda$, the problem becomes,
\begin{equation}\label{eq:groupwise_cf_restate}
    \min_{\lambda\in \Lambda^{|\mathcal{G}|}}f(\lambda) \quad \text{subject to }\quad \text{cg}(\lambda, F_M, \tilde y, \mathcal{G}) - c \leq 0,
\end{equation}
where $f$ is the scalar objective, and we redfine $\text{cg}$ as,
\begin{equation}\label{eq:cg_vector_form}
    \textstyle\text{cg}(\lambda, F_M, \tilde y, \mathcal{G}):= \max\limits_{g_a \in \mathcal{G}}\braces{U_\text{cov}(\lambda_{(g_a)}, F_M, g_a, \tilde{y})}-\min\limits_{g_b \in \mathcal{G}}\braces{L_\text{cov}(\lambda_{(g_b)}, F_M, g_b, \tilde{y})}.
\end{equation}
to accommodate different thresholds for each group, such that $\lambda_{(g)}$ corresponds to the threshold for group $g$ for the fixed label $\tilde y$. For the results presented in Table \ref{tab:fedcf:classwise_and_groupwise}, approximately solve problem \eqref{eq:groupwise_cf_restate} when using $f(\lambda) = \lambda^{\top}1$--the element-wise sum of the $\lambda$ vector (this is discussed below in Section~\ref{subsec:group_info}).

We first note that the feasible set for the classwise formulation (Equation \ref{eq:classwise_formulation}) is a subset of the feasible set for the group-classwise formulation in Equation \ref{eq:groupwise_cf_restate}. We then further assume $f$ is elementwise monotone and separable, e.g. $f(\lambda) := f_{g_{1}}(\lambda_{g_1})+ \cdots +f_{g_{|\mathcal{G}|}}(\lambda_{g_{|\mathcal{G}|}})$ for $g_i \in \mathcal{G}$ where $f_{g_{i}}$ is a non-decreasing function($f(\lambda) = \lambda^{\top}1$ trivially satisfies this assumption). Under this assumption, we observe that an optimizer for the classwise formulation will have an objective cost greater than or equal to the group-classwise formulation. Formally, 

\begin{restatable}{lemma}{thmOptimizerBounds}\label{thm:optimizer_bounds}
    For a fixed label $\tilde y$, fairness metric $F_m$, and set of groups $\mathcal{G}$, let $\lambda_\text{opt}^L \in \mathbb{R}$ be the optimizer of the classwise formulation in Equation \ref{eq:classwise_formulation} and $\lambda_\text{opt}^G \in \mathbb{R}^{|G|}$. Then,
    \begin{equation}
    f(\lambda_\text{opt}^G) \le f(\underbrace{[\lambda_\text{opt}^L, \dots, \lambda_\text{opt}^L]}_{|\mathcal{G}|})
    \end{equation}
\end{restatable}
\begin{proof} By definition, the optimizer of the objective $f$, $\lambda_\text{opt}^G$, must have a smaller objective value than any other feasible solution, i.e., $\underbrace{[\lambda_\text{opt}^L, \dots, \lambda_\text{opt}^L]}_{|\mathcal{G}|}$. 
\end{proof}
\begin{restatable}{remark}{remRestrcuturedOpt}\label{rem:restructure_group_classwise}
    Lemma \ref{thm:optimizer_bounds} does not necessarily imply $\lambda_{\text{opt};(g)}^G \le \lambda_\text{opt}^L$ (where $\lambda_{\text{opt};(g)}^G$ is the threshold for class $g$). However, for the class of problems we consider, this is a simplifying assumption that is often true since $\lambda_\text{opt}^L$ is increased to cover the worst-case group, $g$. Thus, we instead solve the following problem, 
    \begin{subequations}\label{eq:restricted_problem}
        \begin{align}
        \min_{\lambda\in \Lambda^{|\mathcal{G}|}}&f(\lambda)\\
        \textrm{s.t}\quad& \text{cg}(\lambda, F_M, \tilde y, \mathcal{G}) - c \leq 0\\
        & \lambda_{(g)} \le \lambda_\text{opt}^L\quad \forall g\in\mathcal{G}.
        \end{align}
    \end{subequations}
\end{restatable}

To solve the problem in Equation \ref{eq:restricted_problem} we first run the FedCF optimization (Algorithm \ref{alg:descent_cf_algo_extended}) described in to solve $\lambda_\text{opt}^{L}$, for $N$ rounds (num\_rounds in the Algorithm). Following this, we run the Algorithm for $R$ iterations, where we admit a new solution as optimal if $\lambda_{(g)}^{(t+1)} \le \lambda_{(g)}^{(t)}~\forall g\in\mathcal{G}$ (using the element-wise montone property of $f$) and $\lambda$ satisfies the coverage gap requirement. While this leads to local minima, we use random restarts to escape them and find a better solution. By taking the described approach, we can more efficiently search in a smaller space of $\Lambda$ for the first $N$ steps and then perform $M$ searches in the larger space $\Lambda^{|\mathcal{G}|}$. For our implementation, we set {$N$ to 50, $R$ to 50, and recall that $f(\lambda) = \lambda^{\top}1$}.

\subsection{Quantifying the Efficiency Improvement} 
In the following section, we bound the efficiency gap between the classwise and (group, class)-wise formulation when solving the problem in Equation \ref{eq:restricted_problem}:
\begin{restatable}{theorem}{thmBoundsThm}\label{thm:bounds_thm}
    Let $\lambda^{L} \in \mathbb{R}^{|\mathcal{Y}|}$ and $\lambda^{G} \in \mathbb{R}^{|\mathcal{Y}|\times|\mathcal{G}|}$ be optimizers for the class-wise and (group, class)-wise formulations of (Fed)CF, respectively. Then let $x_{test}$ represent an unseen covariate, and $\bm{g}(x_{\textrm{test}})$ map the covariate to its group. Lastly, let $J_{(y,g)}(\lambda) = \Pr[s(x_\textrm{test},y) \leq \lambda| \bm{g}(x_{test}) = g]$ be the CDF function of the conditional score distribution and Lipschitz continuous, with Lipschitz constant $M_{(y,g)}$ over the interval $[\lambda_0, \lambda_\textrm{max}]$---that is for all $v,w \in [\lambda_0, \lambda_\textrm{max}]$, $|J_{(y,g)}(v) - J_{(y,g)}(w)| \le M_{(y,g)}|v-w|$. Then, the efficiency gap between the two (Fed)CF formulations, $\Delta_{\mathrm{eff}}$, satisfies, 
    \begin{equation}
        \Delta_{\mathrm{eff}} \le \sum_{y \in \mathcal{Y}^{+}}\sum_{g \in \mathcal{G}}M_{(y,g)}\left|(\lambda^{L}_{(y)}-\lambda^{G}_{(y,g)})\right|\Pr[\bm{g}(x_{\textrm{test}}) = g].
    \end{equation}
\end{restatable}
\begin{proof}
    Observe from the analysis of~\citep{maneriker2025conformal}, the efficiency difference of two prediction sets is quantified as the difference in likelihood of a non-conformity score being less than the threshold for each label. That is,
    \begin{equation}
        \Delta_{\mathrm{eff}} = \left|\underset{x_\textrm{test}\sim\mathcal{X}}{\mathbb{E}}\left[\sum_{y \in \mathcal{Y}}\left(\mathbf{1}[s(x_\textrm{test},y) \le \lambda^L_{(y)}]-\mathbf{1}[s(x_\textrm{test},y) \le \lambda^{G}_{(y,\bm{g}_\textrm{test})}]\right)\right]\right|,
    \end{equation}
    where $\lambda^L_{(y)}$ is the component for class $y$ in $\lambda^L$ and $\lambda^G_{y, g}$ is the component for $(y, g)$ in $\lambda^G$. 
    For notational brevity, let $\bm{g}_\textrm{test} \coloneqq \bm{g}(x_{\textrm{test}})$. Then,
    \begin{align}
        \Delta_{\mathrm{eff}} &= \left|\sum_{y \in \mathcal{Y}}\left(\Pr[s(x_\textrm{test},y) \le \lambda^L_{(y)}]-\Pr[s(x_\textrm{test},y) \le \lambda^{G}_{(y,\bm{g}_\textrm{test})}]\right)\right|\\
        &=\text{\footnotesize$\left|\sum_{y \in \mathcal{Y}}\sum_{g \in \mathcal{G}}\left(\Pr[s(x_\textrm{test},y) \le \lambda^L_{(y)}\mid \bm{g}_\textrm{test} = g]-\Pr[s(x_\textrm{test},y) \le \lambda^{G}_{(y,\bm{g}_\textrm{test})}\mid\bm{g}_\textrm{test} = g]\right)\Pr[\bm{g}_\textrm{test} = g]\right|$}\\
        &=\text{\footnotesize$\left|\sum_{y \in \mathcal{Y}}\sum_{g \in \mathcal{G}}\left(\Pr[s(x_\textrm{test},y) \le \lambda^L_{(y)}\mid \bm{g}_\textrm{test} = g]-\Pr[s(x_\textrm{test},y) \le \lambda^{G}_{(y,\bm{g}_\textrm{test})}\mid\bm{g}_\textrm{test} = g]\right)\Pr[\bm{g}_\textrm{test} = g]\right|\label{eq:prelim_step_cdf_eff}$}\\
        &= \left|\sum_{y \in \mathcal{Y}}\sum_{g \in \mathcal{G}}\left(J_{(y,g)}(\lambda^{L}_{(y)})-J_{(y,g)} (\lambda^{G}_{(y,g)})\right)\Pr[\bm{g}_\textrm{test} = g]\right|\quad\quad\text{(Definition of CDF)}\\
        &\le \sum_{y \in \mathcal{Y}}\sum_{g \in \mathcal{G}}\left|J_{(y,g)}(\lambda^{L}_{(y)})-J_{(y,g)}(\lambda^{G}_{(y,g)})\right|\Pr[\bm{g}_\textrm{test} = g]\quad\quad\quad\text{(Triangle Inequality)}\\
        &\le \sum_{y \in \mathcal{Y}}\sum_{g \in \mathcal{G}}M_{(y,g)}\left|(\lambda^{L}_{(y)}-\lambda^{G}_{(y,g)})\right|\Pr[\bm{g}_\textrm{test} = g].\quad\quad\quad\quad\text{(Lipschitz Continuity of $J_{(y,g)}$)}
    \end{align}
    Lastly, we note that since $\lambda^L$ and $\lambda^G$ are optimizers to their respective (Fed)CF formulation, for all $y \notin \mathcal{Y^+}$, $\lambda^{L}_{(y)}=\lambda^{G}_{(y,\bm{g}_\textrm{test})} = \lambda_0$. Thus, the bound reduces to,
    \begin{align}\label{eq:bound_thm_final_result}
        \Delta_{\text{eff}}\le \sum_{y \in \mathcal{Y}^{+}}\sum_{g \in \mathcal{G}}M_{(y,g)}\left|(\lambda^{L}_{(y)}-\lambda^{G}_{(y,g)})\right|\Pr[\bm{g}_\textrm{test} = g]
    \end{align}
\end{proof}

\section{FedCF vs FairFed + FCP}\label{app:fed_cf:fairfed}
FedCF takes a (black-box) federated machine learning model and applies the ideas from Conformal Fairness to produce fair federated conformal predictors; however, it is natural to ask whether using an underlying fair ML model would suffice. We conduct an experiment comparing FedCF vs FCP with FairFed~\citep{ezzeldin2023}. We report the worst-case fairness violations using Fitzpatrick with $8$ clients for Demographic Parity. For consistency, we use a Demographic Parity-based weight-scheme for FairFed and set $\beta\in\{0.5, 2.0\}$ in Table~\ref{tab:fed_cf:fairfed}, where $\beta$ is a parameter that controls the fairness budget in the weight update. If $\beta = 0$, then FairFed is equivalent to FedAvg. We observe that using FCP with FairFed is insufficient to achieve the desired fairness guarantee, demonstrating the utility and need for FedCF.

\begin{table}[!ht]
\centering
\caption{\textbf{Fitzpatrick, 8 clients, RAPS} We report the \textbf{fairness disparity} for different $c$ values and observe that simply using a fair base model is insufficient for controlling the fairness-specific coverage gap between groups.}
\label{tab:fed_cf:fairfed}
\begin{tabular}{lcccccc}
\toprule
\multirow{2}{*}{$\beta$} & \multicolumn{2}{c}{$c = 0.1$} & \multicolumn{2}{c}{$c = 0.15$} & \multicolumn{2}{c}{$c = 0.2$} \\
\cmidrule(lr){2-3} \cmidrule(lr){4-5} \cmidrule(lr){6-7}
 & FCP & FedCF & FCP & FedCF & FCP & FedCF \\
\midrule
0.5 & 0.302 & \textbf{0.045} & 0.302 & \textbf{0.140} & 0.303 & \textbf{0.197} \\
2.0 & 0.247 & \textbf{0.024} & 0.247 & \textbf{0.122} & 0.247 & \textbf{0.122} \\
\bottomrule
\end{tabular}
\end{table}

\section{FedCF with Enhanced Privacy}
\label{app:fed_cf:privacy}
Preserving data privacy is a fundamental pillar of FL mechanisms, as they typically interact with sensitive client data. In this vein, we formulate an \textit{enhanced privacy} version of FedCF.

\subsection{Enhanced Privacy} To better preserve privacy (compared to the \textit{communication efficient} approach), we can offload more of the computation to the client-side, making it harder for the server-side to reverse-engineer or infer distributional information from the sent quantities. Expanding Equation \ref{eq:cg_form_2}, we get
\begin{equation}
\label{eq:client_return_form_2}
    \text{\footnotesize$U_\text{cov}(\lambda, F_M, g_a, \tilde{y})-L_\text{cov}(\lambda, F_M, g_b, \tilde{y})
    = \sum\limits_{k=1}^{K} \gamma_k \underbrace{\braces{\frac{(\alpha_k^{(g_a,\tilde y); \lambda}+1)}{(n_k+1)L^{(g_a,\tilde y)}} - \frac{\alpha_k^{(g_b,\tilde y); \lambda}}{(n_k+1)U^{(g_b,\tilde y)}}}}_\text{Returned by the Client}$}.
\end{equation}

In this formulation, the client sends back the summand for each group, positive label pair except for $g_a \ne g_b$, making the space complexity of the client's message $\gO\big(\abs{\gG}^2\abs{\gY^{+}}\big)$--quadratic with respect to the number of groups and linear with respect to positive labels. Observe that for $g_a = g_b$, we are computing the interval width (Equation \ref{eq:client_return_form_iw}) for which we have a tighter overbound from the left hand side of Equation \ref{eq:deriv_final_result_bound} which can be computed using the prior terms sent to compute $L^{(g,\tilde y)}$ and $U^{(g,\tilde y)}$. This allows us to mitigate reconstruction attacks, which would enable one to recover $\alpha_k^{(g_a,\tilde y)}$ if the pairwise gap between $g_a$ and $g_b$ were sent. While this can lead to an overapproximation of the coverage gap, the bounds in Equation \ref{eq:prelim_interp_bounds} establish when this estimate is less than $c$ (i.e., $\frac{2K}{c} \le N^{(g,\tilde{y})}$), thus not hurting our approaches' formulation to find the smallest valid $\lambda_{\text{opt};(y)}$ for each label $y$. 

The data privacy improves with this approach compared to the \textit{communication efficient} version, since the data sent to the server is the difference of client-level summary statistics, which obfuscates individual distribution information from the server. However, unlike the \textit{communication efficient} approach, the upper-coverage term ($U_\text{cov}$) is not separable from the aggregated sum, thus preventing us from enforcing $U_\text{cov} < 1$. In limited data settings, this results in more conservative coverage gap estimates, which increases the prediction set size when using the \textit{enhanced privacy} approach.

\paragraph{Unknown $\gamma_k$.} In our work, we take $\gamma_k\propto (n_k + 1)$; however, if $\gamma_k$ is unknown, we can still apply FedCF using the Enhanced Privacy variant just discussed. This is because we can factor our $\gamma_k$ in Equation~\ref{eq:client_return_form_2} and apply the analysis done in ~\citet[Appendix B]{lu2023federated}.

Below we provide a side-by-side comparison of the \textit{communication efficient} and \textit{enhanced privacy} algorithms for computing the federated coverage gap on the client side (Figure~\ref{fig:client_cg_algorithms}). A unified server-side algorithm can be found in Algorithm~\ref{alg:server_unified}. 

\begin{figure}[ht!]
    \centering
    \begin{minipage}{.49\textwidth}
      \input{algorithms/client_cg_low_overhead}
    \end{minipage}
    \hfill
    \begin{minipage}{.49\textwidth}
      \input{algorithms/client_cg_more_private}
    \end{minipage}
    \caption{\textbf{Pseudocode for the two client-side protocols to compute the coverage gap.} The \textit{enhanced privacy} version (on the right) includes the pairwise computation step, which results in a larger space complexity compared to the more \textit{communication efficient} version (on the left). The \texttt{global} is used to convey that the priors are precomputed and available when needed.}
    \label{fig:client_cg_algorithms}
\end{figure}

\begin{algorithm}[H]
\caption{Full Server-side Aggregation for Coverage Gap}
\label{alg:server_unified}

\begin{algorithmic}[1]
\small
\FUNCTION{\textsc{ServerCG}($\lambda_0, F_M, \tilde y, \mathcal{G}, \texttt{formulations}$)}

\STATE $n\_list = [0]_{\gK}$

    \STATE $l\_list = [0]_{\gK \times \mathcal{G}}$, $u\_list = [0]_{\gK \times \mathcal{G}}$ \COMMENT{Used for comm. efficient formulations}
    \STATE $pw\_cg\_list = [0]_{\gK \times \mathcal{G} \times \mathcal{G}}$ \COMMENT{Used for private formulations}

\FOR{client $k \in \gK$ \textbf{in parallel}}
    \IF{$\texttt{formulations} == \texttt{COMM\_EFFICIENT}$}
        \STATE Receive $(l_k, u_k, n_k) = \textsc{ClientCG\_Comm\_Efficient}(k, \lambda_0, F_M, \tilde y, \mathcal{G})$
        \STATE $l\_list[k] \gets l_k$, $u\_list[k] \gets u_k$
    \ELSE
        \STATE Receive $(pw\_cg_k, n_k) = \textsc{ClientCG\_Private}(k, \lambda_0, F_M, \tilde y, \mathcal{G})$
        \STATE $pw\_cg\_list[k] \gets pw\_cg_k$
    \ENDIF
    \STATE $n\_list[k] \gets n_k$
\ENDFOR

\item[]
\STATE // Initialize final coverage variables
\STATE $N = \sum_{k \in \gK} n\_list[k]$, $K = \abs{\gK}$, $U_\text{cov} = [0]_\mathcal{G}$, $L_\text{cov} = [0]_\mathcal{G}$, $PW_\text{cov} = [0]_{\mathcal{G} \times \mathcal{G}}$
\STATE $\texttt{all\_comm\_efficient} = \texttt{all}\parens{\texttt{formulations}[k] == \texttt{COMM\_EFFICIENT}}$
\FOR{client $k \in \gK$}
    \STATE $\gamma_k = \parens{(n\_list[k] + 1)/(N + K)}$
    \IF{\texttt{all\_comm\_efficient}}
        \IF{\texttt{use\_mle}}
            \STATE $U_\text{cov} $ += $ \parens{\gamma_k/\pi^{(g,\tilde y)}} \cdot u\_list[k]$ \COMMENT{Standard operations are element-wise}
            \STATE $L_\text{cov} $ += $ \parens{\gamma_k / \pi^{(g,\tilde y)}} \cdot l\_list[k]$
        \ELSE
            \STATE $U_\text{cov} $ += $ \parens{\gamma_k/L^{(g,\tilde y)}} \cdot u\_list[k]$ \COMMENT{Standard operations are element-wise}
            \STATE $L_\text{cov} $ += $ \parens{\gamma_k / U^{(g,\tilde y)}} \cdot l\_list[k]$
        \ENDIF
    \ELSE
        \IF{$\texttt{formulations}[k] == \texttt{COMM\_EFFICIENT}$}
            \STATE $PW_\text{cov} $ += $ \gamma_k \cdot \parens{u\_list[k]~\ominus~l\_list[k]^{\top}}$ \COMMENT{$\ominus$ is pairwise differences between two vectors.}
        \ELSE
            \STATE $PW_\text{cov} $ += $ \gamma_k \cdot pw\_cg\_list[k]$
        \ENDIF
    \ENDIF
\ENDFOR

\item[]
\IF{\texttt{all\_comm\_efficient}}
    \STATE $U_\text{cov} = \text{element\_wise\_min}(U_\text{cov},[1]_\mathcal{G})$ \COMMENT{Limit upper coverage prior to coverage gap calculation}
    \STATE $\text{cov\_gap} =  \max\limits_{g \in \mathcal{G}} U_\text{cov}[g] - \min\limits_{g \in \mathcal{G}} L_\text{cov}[g]$
\ELSE
    \STATE $\text{cov\_gap} = \min\left\{\max\limits_{g_a, g_b \in \mathcal{G}} PW_\text{cov}[g_a, g_b], 1 \right\}$ \COMMENT{Limit Coverage Gap to 1}
\ENDIF

\STATE \textbf{return} \text{cov\_gap}
\ENDFUNCTION
\end{algorithmic}
\end{algorithm}

\subsection{Hybrid} In real-world scenarios, clients often have varying privacy and communication requirements. For example, clients in resource-constrained areas may not have the network bandwidth to send the necessary packets to the centralized server. In our proposed \textit{hybrid} approach, a client may elect to be \textit{communication efficient}, without preventing the remaining clients from using the \textit{enhanced privacy} protocol. We present the full server-side algorithm, which combines the \textit{communication efficient}, \textit{enhanced privacy}, and \textit{hybrid} protocols for the federated coverage gap, in Algorithm \ref{alg:server_unified}.

\subsection{Empirical Comparison}
We conduct two experiments using the Fitzpatrick dataset and $8$ clients, as well as the larger ACSIncome dataset with the \textit{continental\_all} partition scheme--$48$ clients--to test the \textit{communication efficient}, \textit{enhanced privacy}, and \textit{hybrid} protocols. For the \textit{hybrid} protocol, we randomly assign half the clients to each protocol. From Table \ref{tab:form_expt}, we observe that all configurations control the fairness disparity within the closeness criterion; however, if all clients agree upon the \textit{communication efficient} protocol, FedCF achieves a better efficiency with a slightly worse fairness disparity, albeit still within the closeness criterion. Though with more data, we observe that the efficiency gaps are smaller as seen in Table \ref{tab:form_expt_acs}.


\paragraph{}\hspace{-3.5mm}With these empirical results, note that under the hybrid setting, clients that optimize for communication efficiency still benefit from the fact that they can operate over a limited bandwidth network connection. The required bandwidth for a particular client undergoes a factor of $\approx \frac{|\mathcal{G}|}{2}$ reduction--i.e. $\mathcal{O}(|\mathcal{G}|^2|\mathcal{Y}^+) \to \mathcal{O}(2\cdot |\mathcal{G}||\mathcal{Y}^+)$, when a client selects the \textit{communication efficient} protocol while ensuring the remaining clients benefit from the \textit{enhanced privacy} protocol. For Fitzpatrick, this results in the communication overhead (in bytes) being reduced by a factor of three.
($|\mathcal{G}|/2=3$, for Fitzpatrick). For the ACS datasets using the small, large, or all client assignments, this reduction corresponds to $|\mathcal{G}|/2=4.5$
\raggedbottom
\begin{center}
\begin{table}[ht!]
    \centering
    \small
    \caption{\textbf{Fitzpatrick, $\bm{8}$ clients, APS.} Each entry is of the form, \textbf{efficiency/fairness disparity}. We bold the lower fairness disparity value for each comparison. We observe that the \textit{communication efficient} approach produces the most efficient prediction sets, while having a similar or higher fairness disparity. The \textit{enhanced privacy} approach and \textit{hybrid} approach have similar performance (w.r.t efficiency and fairness disparity), with minor differences stemming from the stochasticity of FedCF, as they default to the same coverage-gap aggregation protocol (see Algorithm \ref{alg:server_unified}). All methods improve upon the baseline fairness disparity and control for the closeness criterion.}
    \label{tab:form_expt}

    \subfloat[Enhanced Privacy]{%
    \begin{tabular}{lcccccc}
    \toprule
    \multirow{2}{*}{Metric} & \multicolumn{2}{c}{$c = 0.1$} & \multicolumn{2}{c}{$c = 0.15$} & \multicolumn{2}{c}{$c = 0.2$} \\
    \cmidrule(lr){2-3} \cmidrule(lr){4-5} \cmidrule(lr){6-7}
     & Base & Ours & Base & Ours & Base & Ours \\
    \midrule
    \textbf{Dem\_Parity} & 3.671 / 0.136 & 7.041 / \textbf{0.047} & 3.671 / 0.136 & 4.978 / \textbf{0.101} & 3.672 / 0.136 & 3.940 / \textbf{0.111} \\
    \textbf{Pred\_Eq}    & 3.676 / 0.134 & 7.042 / \textbf{0.047} & 3.675 / 0.134 & 4.765 / \textbf{0.094} & 3.672 / 0.134 & 3.880 / \textbf{0.106} \\
    \bottomrule
    \end{tabular}}

    \vspace{0.1cm}

    \subfloat[Hybrid (50-50)]{%
    \begin{tabular}{lcccccc}
    \toprule
    \multirow{2}{*}{Metric} & \multicolumn{2}{c}{$c = 0.1$} & \multicolumn{2}{c}{$c = 0.15$} & \multicolumn{2}{c}{$c = 0.2$} \\
    \cmidrule(lr){2-3} \cmidrule(lr){4-5} \cmidrule(lr){6-7}
     & Base & Ours & Base & Ours & Base & Ours \\
    \midrule
    \textbf{Dem\_Parity} & 3.674 / 0.136 & 6.871 / \textbf{0.066} & 3.670 / 0.136 & 4.967 / \textbf{0.103} & 3.670 / 0.136 & 3.939 / \textbf{0.111} \\
    \textbf{Pred\_Eq}    & 3.671 / 0.134 & 7.041 / \textbf{0.047} & 3.673 / 0.134 & 5.123 / \textbf{0.094} & 3.671 / 0.134 & 3.919 / \textbf{0.107} \\
    \bottomrule
    \end{tabular}}

    \vspace{0.1cm}

    \subfloat[Communication Efficient]{%
    \begin{tabular}{lcccccc}
    \toprule
    \multirow{2}{*}{Metric} & \multicolumn{2}{c}{$c = 0.1$} & \multicolumn{2}{c}{$c = 0.15$} & \multicolumn{2}{c}{$c = 0.2$} \\
    \cmidrule(lr){2-3} \cmidrule(lr){4-5} \cmidrule(lr){6-7}
     & Base & Ours & Base & Ours & Base & Ours \\
    \midrule
    \textbf{Dem\_Parity} & 3.674 / 0.137 & 6.053 / \textbf{0.104} & 3.674 / 0.136 & 4.890 / \textbf{0.103} & 3.672 / 0.136 & 3.935 / \textbf{0.111} \\
    \textbf{Pred\_Eq}    & 3.671 / 0.134 & 6.308 / \textbf{0.109} & 3.674 / 0.134 & 4.931 / \textbf{0.094} & 3.674 / 0.134 & 3.876 / \textbf{0.106} \\
    \bottomrule
    \end{tabular}}
\end{table}
\end{center}
\begin{table}[H]
    \centering
    \small
    \caption{\textbf{ACSIncome, Continental All, RAPS.} Each entry is of the form, \textbf{efficiency/fairness disparity}. We observe that with sufficient data, each protocol performs at a similar efficiency, and they all decrease the baseline fairness disparity and control it within the closeness criterion. Our fairness disparity values are bolded.}
    \label{tab:form_expt_acs}

    \subfloat[Enhanced Privacy]{%
    \begin{tabular}{lcccccc}
    \toprule
    \multirow{2}{*}{Metric} & \multicolumn{2}{c}{$c = 0.1$} & \multicolumn{2}{c}{$c = 0.15$} & \multicolumn{2}{c}{$c = 0.2$} \\
    \cmidrule(lr){2-3} \cmidrule(lr){4-5} \cmidrule(lr){6-7}
     & Base & Ours & Base & Ours & Base & Ours \\
    \midrule
    \textbf{Dem\_Parity} & 2.609 / 0.148 & 3.037 / \textbf{0.086} & 2.610 / 0.148 & 2.634 / \textbf{0.138} & 2.613 / 0.148 & 2.613 / \textbf{0.148} \\
    \textbf{Pred\_Eq}    & 2.607 / 0.160 & 3.294 / \textbf{0.063} & 2.610 / 0.161 & 2.661 / \textbf{0.138} & 2.609 / 0.161 & 2.609 / \textbf{0.161} \\
    \bottomrule
    \end{tabular}}

    \vspace{0.1cm}

    \subfloat[Hybrid (50-50)]{%
    \begin{tabular}{lcccccc}
    \toprule
    \multirow{2}{*}{Metric} & \multicolumn{2}{c}{$c = 0.1$} & \multicolumn{2}{c}{$c = 0.15$} & \multicolumn{2}{c}{$c = 0.2$} \\
    \cmidrule(lr){2-3} \cmidrule(lr){4-5} \cmidrule(lr){6-7}
     & Base & Ours & Base & Ours & Base & Ours \\
    \midrule
    \textbf{Dem\_Parity} & 2.608 / 0.148 & 3.039 / \textbf{0.085} & 2.609 / 0.148 & 2.633 / \textbf{0.138} & 2.596 / 0.148 & 2.596 / \textbf{0.148} \\
    \textbf{Pred\_Eq}    & 2.606 / 0.160 & 3.277 / \textbf{0.079} & 2.606 / 0.160 & 2.657 / \textbf{0.138} & 2.595 / 0.161 & 2.595 / \textbf{0.161} \\
    \bottomrule
    \end{tabular}}

    \vspace{0.1cm}

    \subfloat[Communication Efficient]{%
    \begin{tabular}{lcccccc}
    \toprule
    \multirow{2}{*}{Metric} & \multicolumn{2}{c}{$c = 0.1$} & \multicolumn{2}{c}{$c = 0.15$} & \multicolumn{2}{c}{$c = 0.2$} \\
    \cmidrule(lr){2-3} \cmidrule(lr){4-5} \cmidrule(lr){6-7}
     & Base & Ours & Base & Ours & Base & Ours \\
    \midrule
    \textbf{Dem\_Parity} & 2.608 / 0.148 & 3.037 / \textbf{0.086} & 2.611 / 0.148 & 2.634 / \textbf{0.138} & 2.601 / 0.149 & 2.601 / \textbf{0.149} \\
    \textbf{Pred\_Eq}    & 2.610 / 0.161 & 3.300 / \textbf{0.071} & 2.607 / 0.160 & 2.658 / \textbf{0.138} & 2.609 / 0.161 & 2.609 / \textbf{0.161} \\
    \bottomrule
    \end{tabular}}
\end{table}
\clearpage
\section{Differential Privacy in FedCF}
\label{app:fedcf:diff_privacy} As both the \emph{Communication-Efficient} and \emph{Enhanced Privacy} approaches can be targeted by reconstruction attacks from a malicious server, we extend
FedCF to formally consider $(\epsilon,\delta)$-differential privacy (DP) to maintain robustness under adversarial conditions. DP is a mathematically rigorous framework for data privacy~\citep{dwork2006differential}, where $\delta$ is the probability that $\epsilon$-DP is violated. We can embed DP within our framework via client shuffling and additive noise approaches. Client shuffling is a global DP approach that is performed after the client sends data. Before the server receives the data, it goes through a trusted, centralized shuffler to anonymize which client has sent what data \citep{erlingsson2019amplification}. Our framework can accommodate client shuffling due to its parallelism with client-side computation and its additive aggregation approach.

For additive noise, we propose augmenting the values each client sends back with Gaussian noise~\citep{dwork2014algorithmic,dong2022gaussian}, such that a client returns, 
\begin{equation}\label{eq:g_func}
h = {\frac{(\alpha_k^{(g_a,\tilde y); \lambda}+1)}{(n_k+1)L^{(g_a,\tilde y)}} - \frac{\alpha_k^{(g_b,\tilde y); \lambda}}{(n_k+1)U^{(g_b,\tilde y)}}} + X,
\end{equation}
where $X$ is a Gaussian random variable (R.V). For the \textit{communication efficient} approach, one would add a Gaussian R.V. to the upper coverage and lower coverage terms returned by the client. To ensure $(\epsilon,\delta)$-DP, we make $X\sim \mathcal{N}\bigg(0, \frac{2\ln(1.25/\delta)(\Delta h)^2}{\epsilon^2}\bigg)$, where $\Delta h$ is the sensitivty of $h$--or how much $h$ can change if one of the points in the client's dataset changes. For FedCF, $h$ can be affected by data changes in the covariates (or non-conformity scores), labels, and group memberships.

\subsection{\textbf{Example:} Differential Privacy Bounds for Enhanced Privacy Protocol}
Observe using the \textit{enhanced privacy} approach, $\Delta h \leq \frac{1}{n_k}\big(\frac{1}{L^{(g_a,\tilde y)}} + \frac{1}{U^{(g_b,\tilde y)}}\big)$. For the \textit{communication efficient} approach $\Delta h \leq \frac{1}{n_k}$ for the components of the upper and lower coverage terms sent by the client. The server will know each client's sensitivity and their choice of $\epsilon$ and $\delta$.

To demonstrate how the server can estimate the coverage gap, we will consider an example using the \textit{enhanced privacy} approach. The result from server aggregation is, 
\begin{align}
    &\text{cov\_gap\_est}(\lambda, F_M,g_a,g_b,\tilde y)\nonumber\\
    &~~~\coloneqq\sum\limits_{k=1}^{K} \gamma_k \underbrace{\braces{\frac{(\alpha_k^{(g_a,\tilde y); \lambda}+1)}{(n_k+1)L^{(g_a,\tilde y)}} - \frac{\alpha_k^{(g_b,\tilde y); \lambda}}{(n_k+1)U^{(g_b,\tilde y)}}+X_k^{(g_a,g_b,\tilde y)}}}_\text{Returned by the Client},
\end{align}
where $X_k^{(g_a,g_b,\tilde y)} \sim \mathcal{N}(0, {\sigma^2}_{k;(g_a,g_b,\tilde y)})$ such that ${\sigma^2}_{k;(g_a,g_b,\tilde y)}$ provides $(\epsilon_k,\delta_k)$-DP for the client. Then observe, 

\begin{align}
    &\text{cov\_gap\_est}(\lambda, F_M,g_a,g_b,\tilde y)\nonumber\\
    &~~~~~~~~~~= \underbrace{\sum\limits_{k=1}^{K} \gamma_k \braces{\frac{(\alpha_k^{(g_a,\tilde y); \lambda}+1)}{(n_k+1)L^{(g_a,\tilde y)}} - \frac{\alpha_k^{(g_b,\tilde y); \lambda}}{(n_k+1)U^{(g_b,\tilde y)}}}}_\text{true coverage gap} + \underbrace{\sum\limits_{k=1}^{K} \gamma_kX_k^{(g_a,g_b,\tilde y)}}_\text{Guassian R.V}\\
    &~~~~~~~~~~=\text{cov\_gap}(\lambda, F_M,g_a,g_b,\tilde y) + X,~~X \sim \mathcal{N}\bigg(0,\sum\limits_{k=1}^{K} \gamma_k^2{\sigma^2}_{k;(g_a,g_b,\tilde y)}\bigg) 
\end{align}

Using a prespecified probability $\beta$ we can accept or reject the statement $\text{cov\_gap\_est}(\lambda, F_M,g_a,g_b,\tilde y) \leq c$. In other words, we can check whether,

\begin{align*}
&\text{cov\_gap}(\lambda, F_M,g_a,g_b,\tilde y) + X \leq c \implies X \leq c - \text{cov\_gap}(\lambda, F_M,g_a,g_b,\tilde y) \\
&\implies\boxed{\underbrace{\frac{X}{\sqrt{\sum\limits_{k=1}^{K} \gamma_kX_k^{(g_a,g_b,\tilde y)}}}}_{\text{Standard Normal RV}} \leq \frac{c - \text{cov\_gap}(\lambda, F_M,g_a,g_b,\tilde y)}{\sqrt{\sum\limits_{k=1}^{K} \gamma_kX_k^{(g_a,g_b,\tilde y)}}}}.
\end{align*}

Then, if $\Phi\parens{\frac{c - \text{cov\_gap}(\lambda, F_M,g_a,g_b,\tilde y)}{\sqrt{\sum\limits_{k=1}^{K} \gamma_kX_k^{(g_a,g_b,\tilde y)}}}} > \beta$, where $\Phi$ is the CDF of the standard normal distribution, we can accept the coverage gap as being less than $c$. In other words, with probability $\beta$, the closeness criterion is satisfied with $\lambda$. 

While using Gaussian noise results in a PAC-style guarantee, one could instead add strictly positive noise via an exponential mechanism \cite{dwork2014algorithmic}, where the noise $X\sim \text{exp}(\frac{\epsilon}{2\Delta h})$ is selected to satisfy $\epsilon$-DP, i.e., $(\epsilon,0)$-DP. This would result in an overestimate of the actual coverage gap. If the overestimate satisfies the closeness criterion, then the server would assert that the exact coverage gap also satisfies the closeness criterion--thus restoring the strict (non-PAC) guarantee in FedCF.   
\section{FedCF for Auditing}
\label{app:fed_cf:audit}
Auditing tools are vital for regulatory bodies to ensure ML models comply with fairness and safety standards \citep{maneriker2023}. In this regard, we present how FedCF can be used to determine if a federated conformal predictor is \textit{fair} according to the regulator's specification of fairness and closeness criterion, $c$~\citep{eeoc1979, fair_chance_act, aedt_nyc, euaiact}. 

To assess compliance, FedCF can use the global threshold ($\lambda$) values used by the previously trained conformal-predictor and provide it to each client. Then, the client should send the sufficient values calculated via Algorithm \ref{alg:client_form_1} (or Algorithm \ref{alg:client_form_2}) to compute the federated coverage gap. The server would aggregate these values using Algorithm \ref{alg:server_unified}. If the calculated coverage gap is below $c$, then the server can assert that the conformal predictor is fair. 

Our auditing approach does not require all clients to provide data for auditing. As discussed in Section \ref{subsec:fed_cov_gap}, our guarantees hold assuming that the test-point, $(\bm x_\text{test}, y_\text{test}) \sim \sum_{k=1}^{K}\gamma_k P_k$, is sampled from a mixture of client distributions where $\gamma_k$ is the probability the test point is sampled from $P_k$, or equivilantly is exchangeable with data from client $k$. Thus, if a subset of clients used to train the original federated conformal predictor provides auditing data, then the audit guarantees will hold assuming that $(\bm x_\text{test}, y_\text{test})$ are sampled from a mixture consisting of the subset of clients used for auditing. This result allows clients to \textit{independently} decide if they would like to submit data for auditing.

The auditing tool provided by FedCF can also be used to ascertain the \textit{marginal} fairness with respect to each client. Using the auditing procedure described above with data from one client, FedCF can determine if the global, federated conformal predictor maintains fairness with respect to data from a single client. If the computed coverage gap is less than $c$, then the fairness guarantees hold with regard to $(\bm x_\text{test},y_\text{test}) \sim P_k$, i.e., the client's marginal distribution.

\clearpage
\section{More Results}\label{app:fedcf:more_results}
Here, we provide additional results for the ACS and Pokec-\{n,z\} datasets. Recall, in each figure, we use a \textbf{solid} line to represent the \textit{average} efficiency of the \textbf{base federated conformal predictors} across different thresholds and a \textbf{dashed} line to represent the corresponding \textit{average} worst-case fairness disparity. The bar plot shows the efficiency and worst-case fairness disparity using FedCF, while the \textbf{dots} indicate the \textit{desired} fairness disparity. We report the average base performance for clarity and readability

\subsection{Coverage Results}\label{app:coverage_results}
A key property of FedCF is that it ensures the base coverage guarantee is satisfied by bounding the threshold search space by FCP's conformal quantile, $\hat{q}(\alpha)$ from \eqref{eq:fedcf:fedcp_quantile}. In every case, FedCF's coverage will be greater than or equal to FCP's coverage.
\begin{table}[htbp!]
    \centering
    \footnotesize
\captionsetup{font=footnotesize,labelfont=footnotesize}
    \caption{{Various datasets using \textbf{Demographic Parity}. We compare the \textbf{marginal coverage} of the Base (FCP) approach against FedCF over different closeness criteria ($c$).}}
    \label{tab:fedcf_demographic_parity}
    \begin{tabular}{l ccc ccc ccc}
        \toprule
        & \multicolumn{3}{c}{\textbf{Pokec-\{n, z\} (DAPS)}} & \multicolumn{3}{c}{\textbf{Fitzpatrick (RAPS)}} & \multicolumn{3}{c}{\textbf{ACSEducation (RAPS)}} \\
        \cmidrule(lr){2-4} \cmidrule(lr){5-7} \cmidrule(lr){8-10}
        $c$ & 0.1 & 0.15 & 0.2 & 0.1 & 0.15 & 0.2 & 0.1 & 0.15 & 0.2 \\
        \midrule
        Base (FedCP) & 0.894 & 0.894 & 0.894 & 0.895 & 0.895 & 0.895 & 0.932 & 0.932 & 0.932 \\
        FedCF        & {0.904} & {0.901} & 0.894 & {0.927} & {0.909} & {0.896} & {0.977} & {0.973} & {0.952} \\
        \bottomrule
    \end{tabular}
\end{table}


\raggedbottom
\subsection{Impact of Data Heterogeneity on ACSEducation: US vs Continental US}

\begin{figure}[ht!]
    \centering
    \begin{subfigure}{0.475\textwidth}
        \begin{subfigure}{\textwidth}
        \centering
            \includegraphics[width=\linewidth]{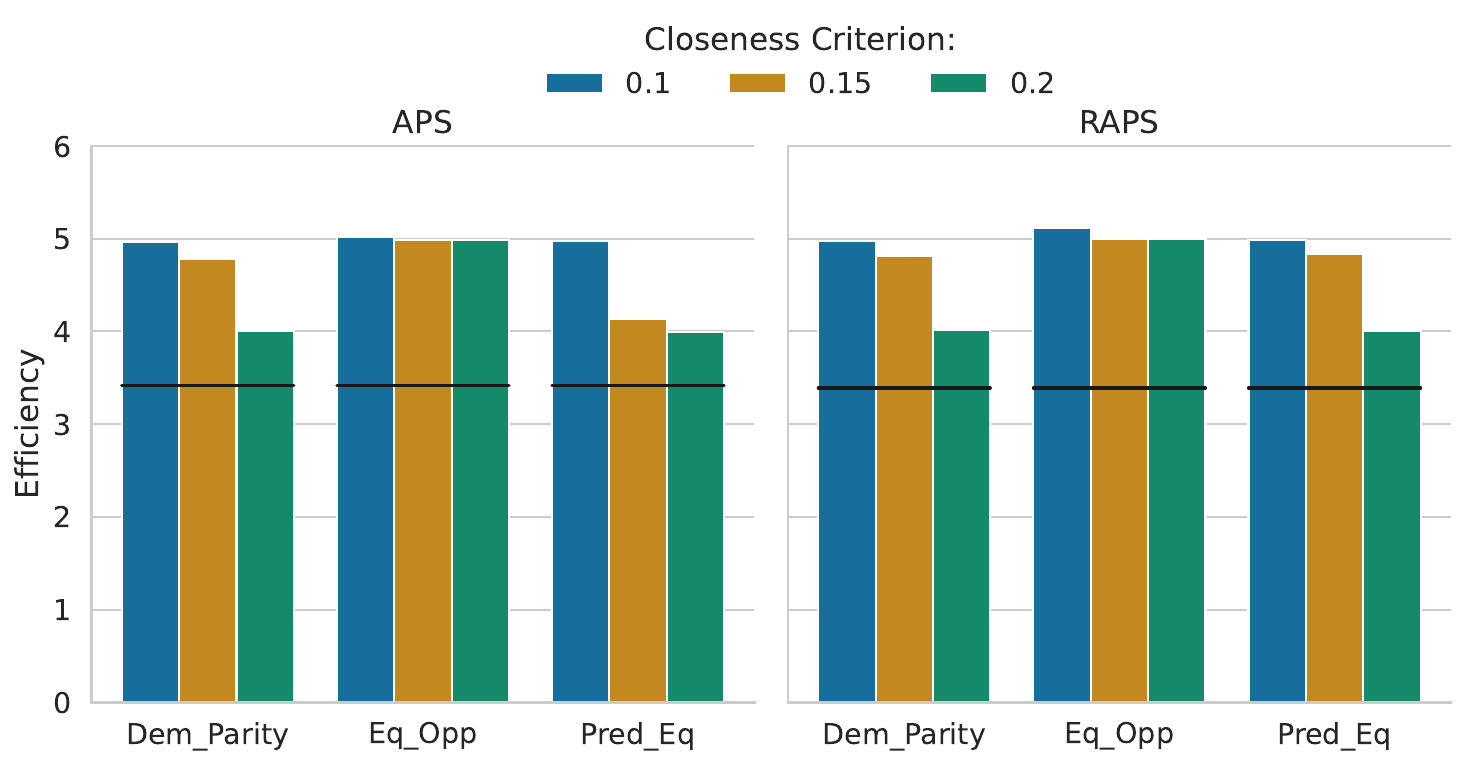}
        \end{subfigure}
        \begin{subfigure}{\textwidth}
            \centering
            \includegraphics[width=\linewidth]{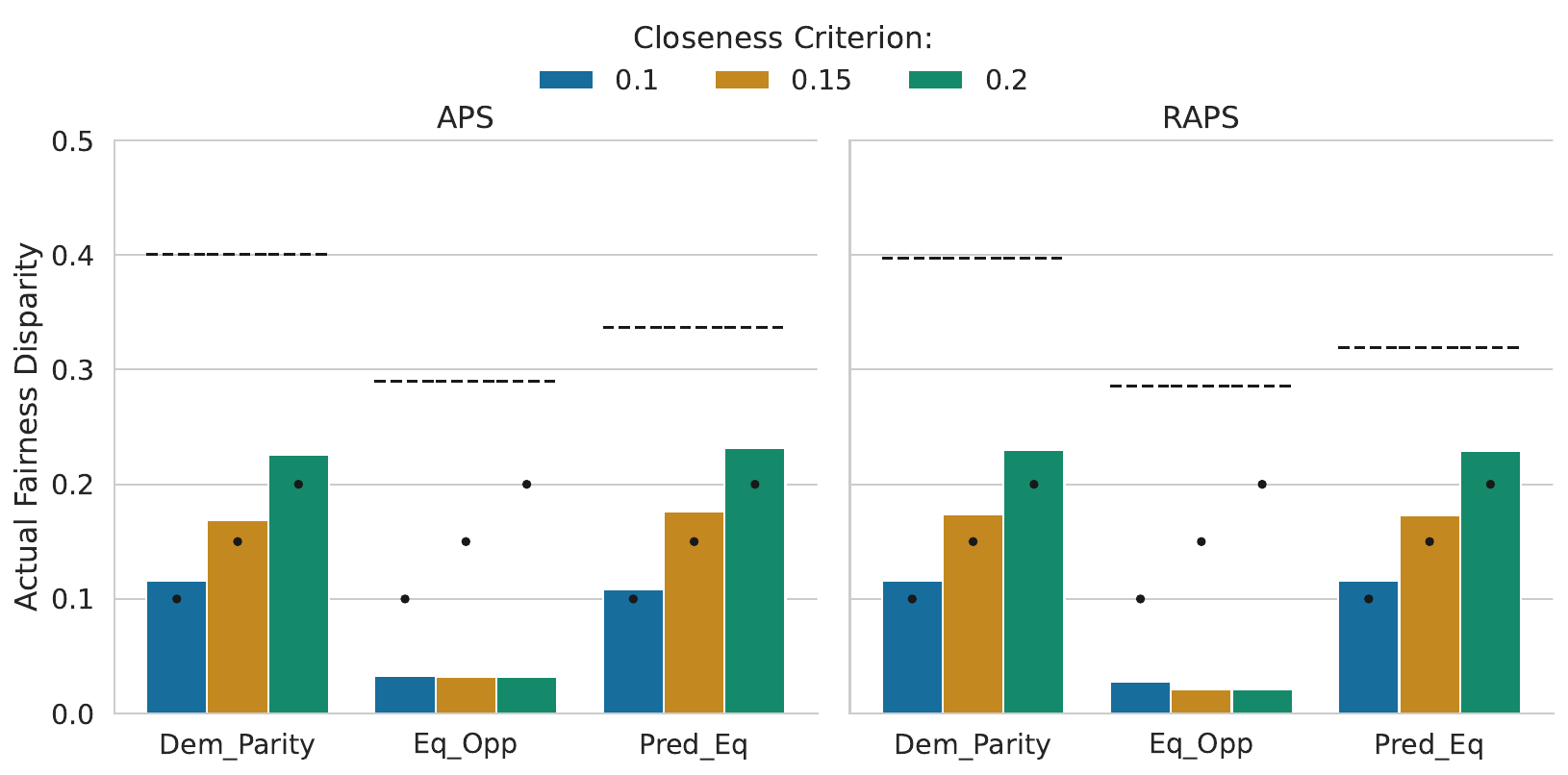}
        \end{subfigure}
        \caption{\textbf{Communication Efficient.}}
    \end{subfigure}%
    \hfill
    \begin{subfigure}{0.475\textwidth}
    \begin{subfigure}{\textwidth}
        \centering
        \includegraphics[width=\linewidth]{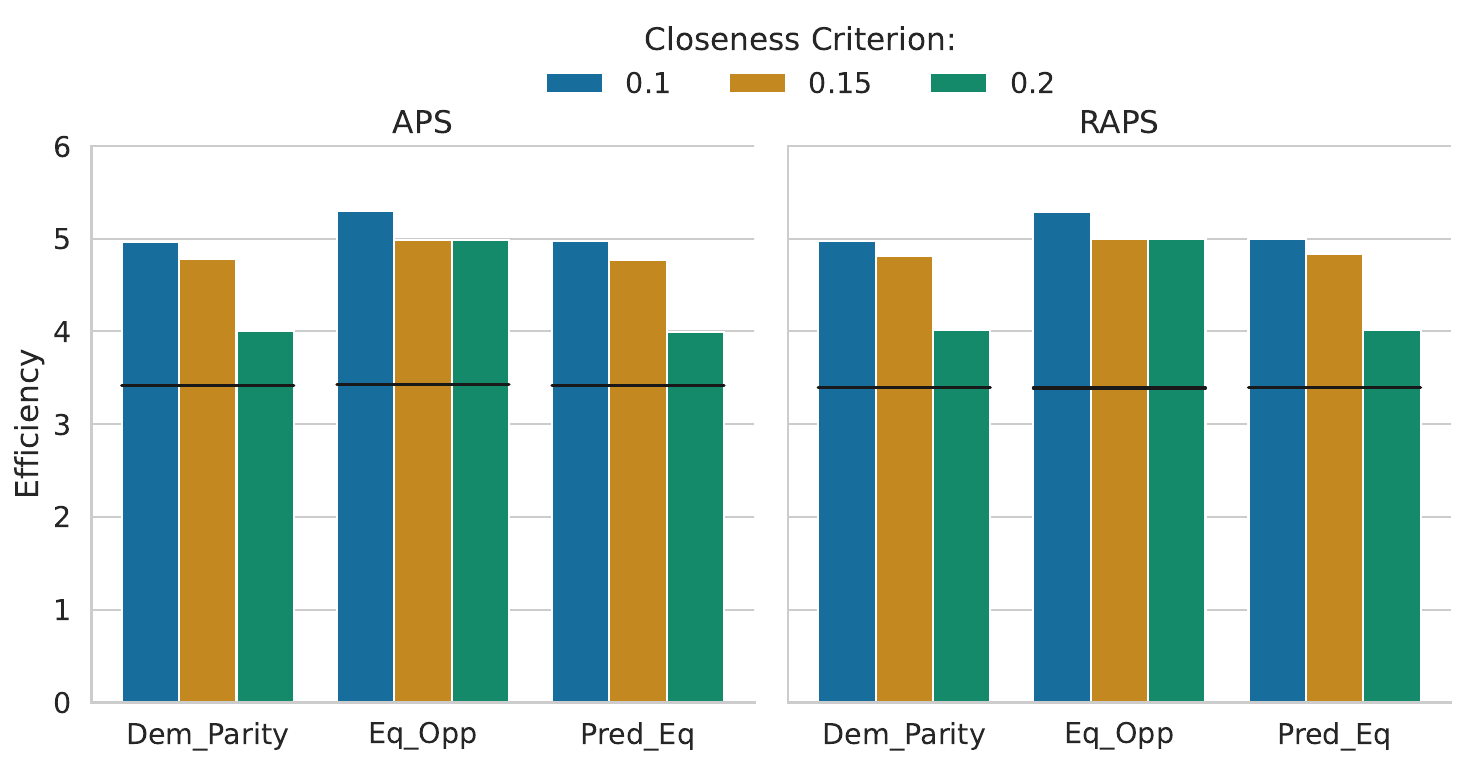}
    \end{subfigure}
    \begin{subfigure}{\textwidth}
        \centering
        \includegraphics[width=\linewidth]{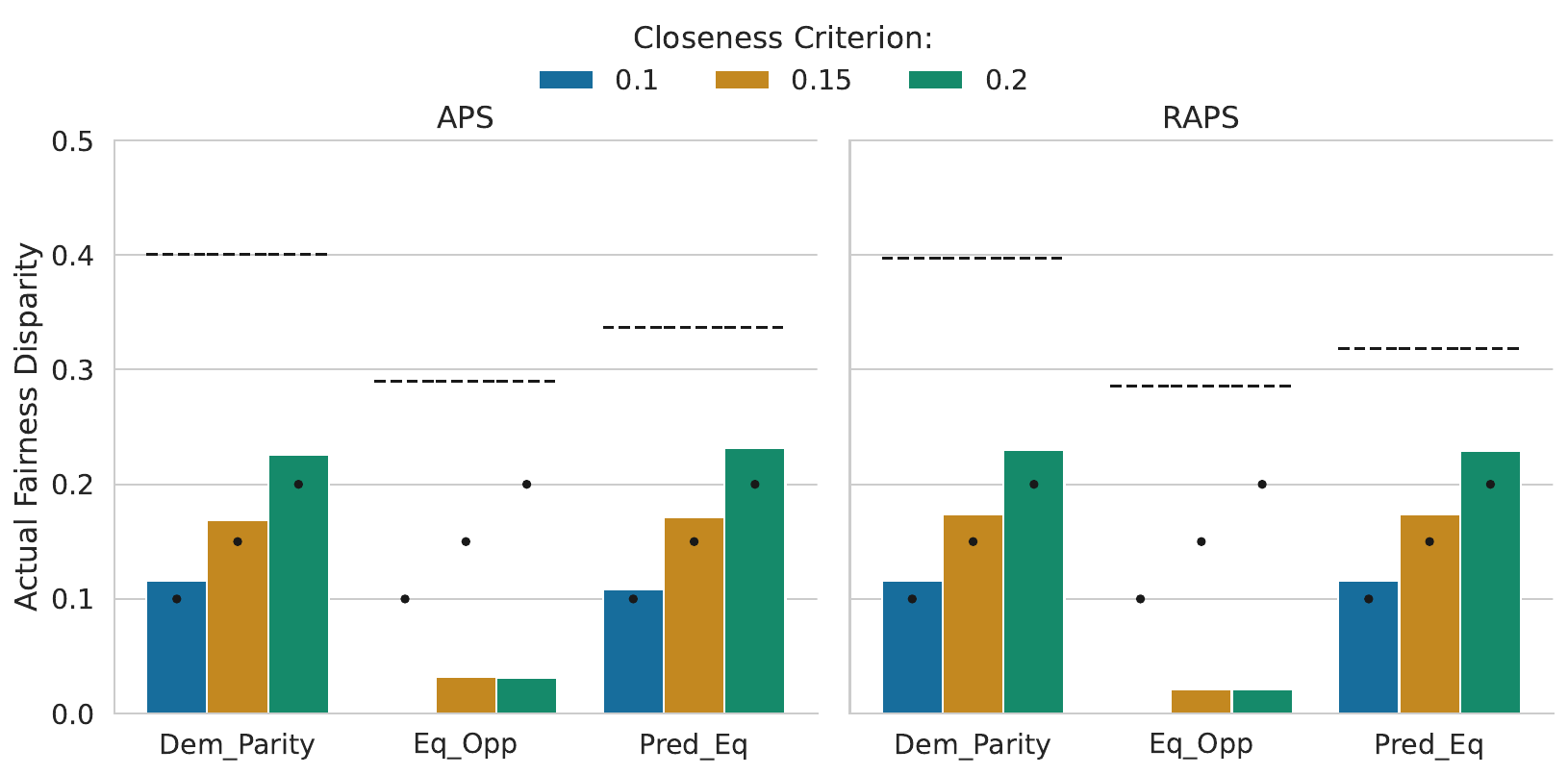}
    \end{subfigure}
    \caption{\textbf{Enhanced Privacy.}}
    \end{subfigure}
    \caption{\label{fig:acs_education_small_appendix}\textbf{ACSEducation, Small, Interval Bounds.} The plots in the top row indicate the efficiency with the corresponding fairness disparity plots in the bottom row. We observe that when all US states are included (and Puerto Rico), the closeness criterion is satisfied. However, the efficiency for Equal Opportunity is high for all closeness criterion values, especially compared to the continental US version of ACSEducation in Figure \ref{fig:acs_education_continental_small}. This result stems from a conservative coverage gap estimate during calibration due to limited covariate representation for some groups.}
\end{figure}

\begin{figure}[ht!]
    \centering
    \begin{subfigure}{0.475\textwidth}
        \begin{subfigure}{\textwidth}
        \centering
            \includegraphics[width=\linewidth]{figures/ACSEducation_small_False_Communication_Efficient_efficiency.pdf}
        \end{subfigure}
        \begin{subfigure}{\textwidth}
            \centering
            \includegraphics[width=\linewidth]{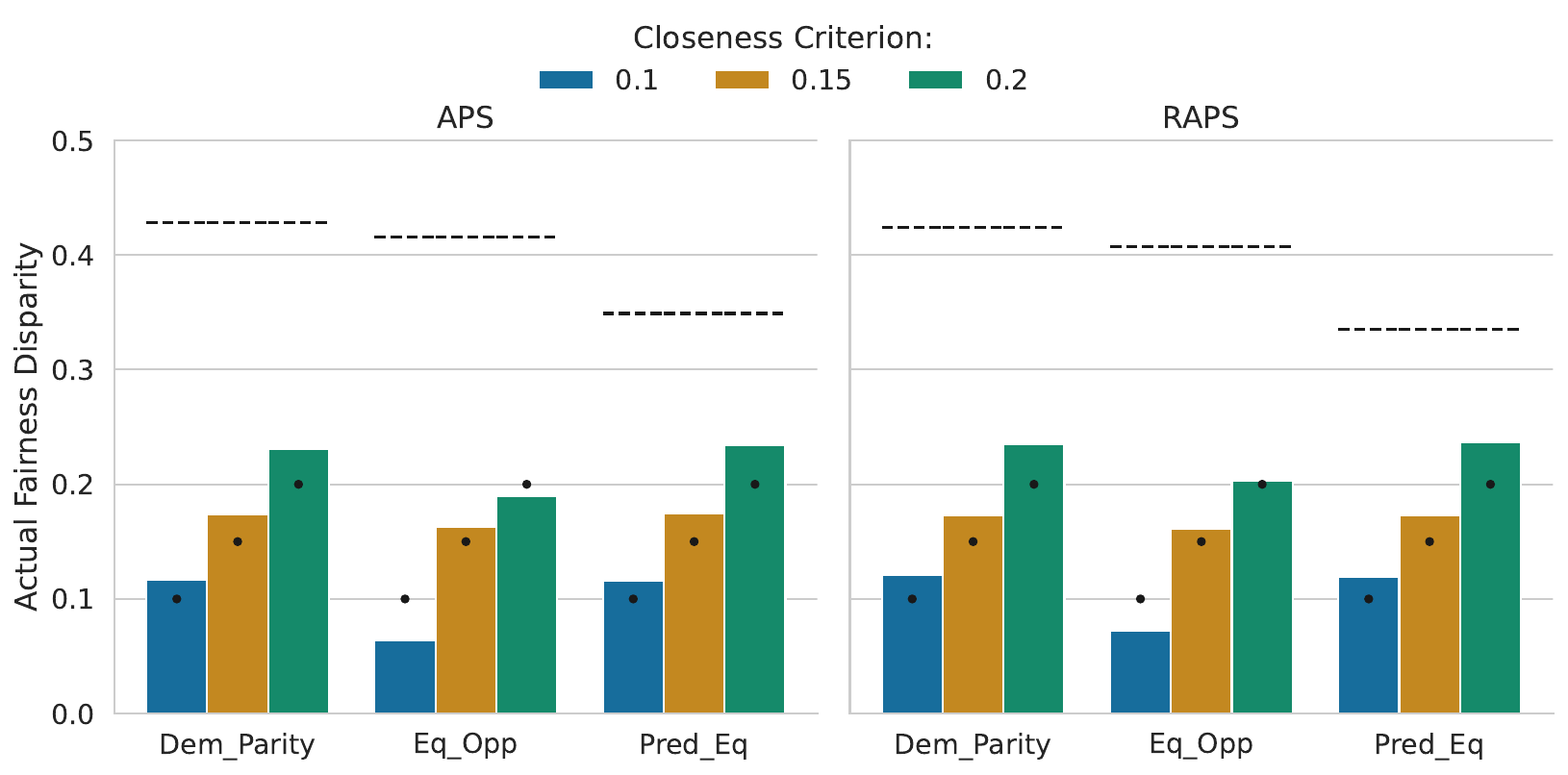}
        \end{subfigure}
        \caption{\textbf{Communication Efficient.}}
    \end{subfigure}%
    \hfill
    \begin{subfigure}{0.475\textwidth}
    \begin{subfigure}{\textwidth}
        \centering
        \includegraphics[width=\linewidth]{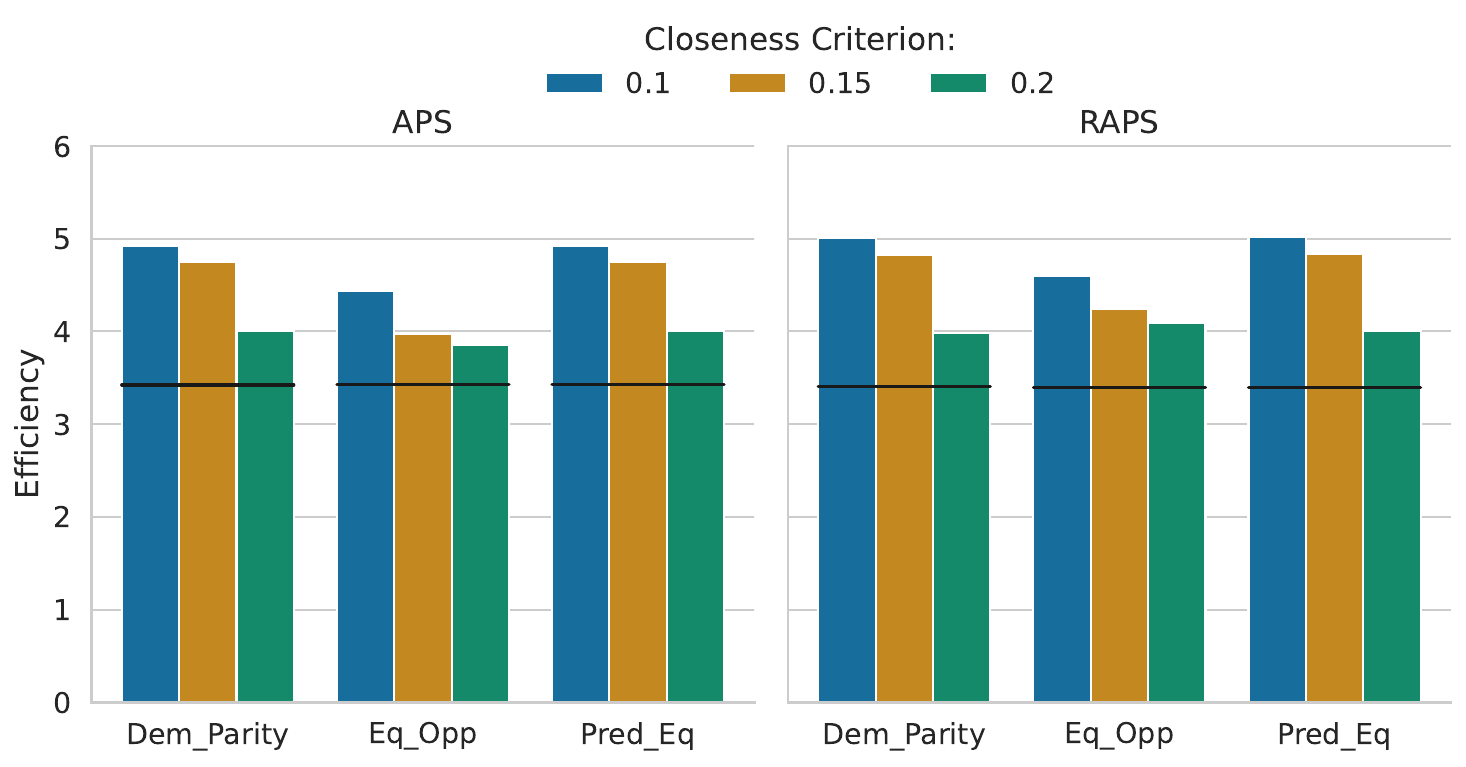}
    \end{subfigure}
    \begin{subfigure}{\textwidth}
        \centering
        \includegraphics[width=\linewidth]{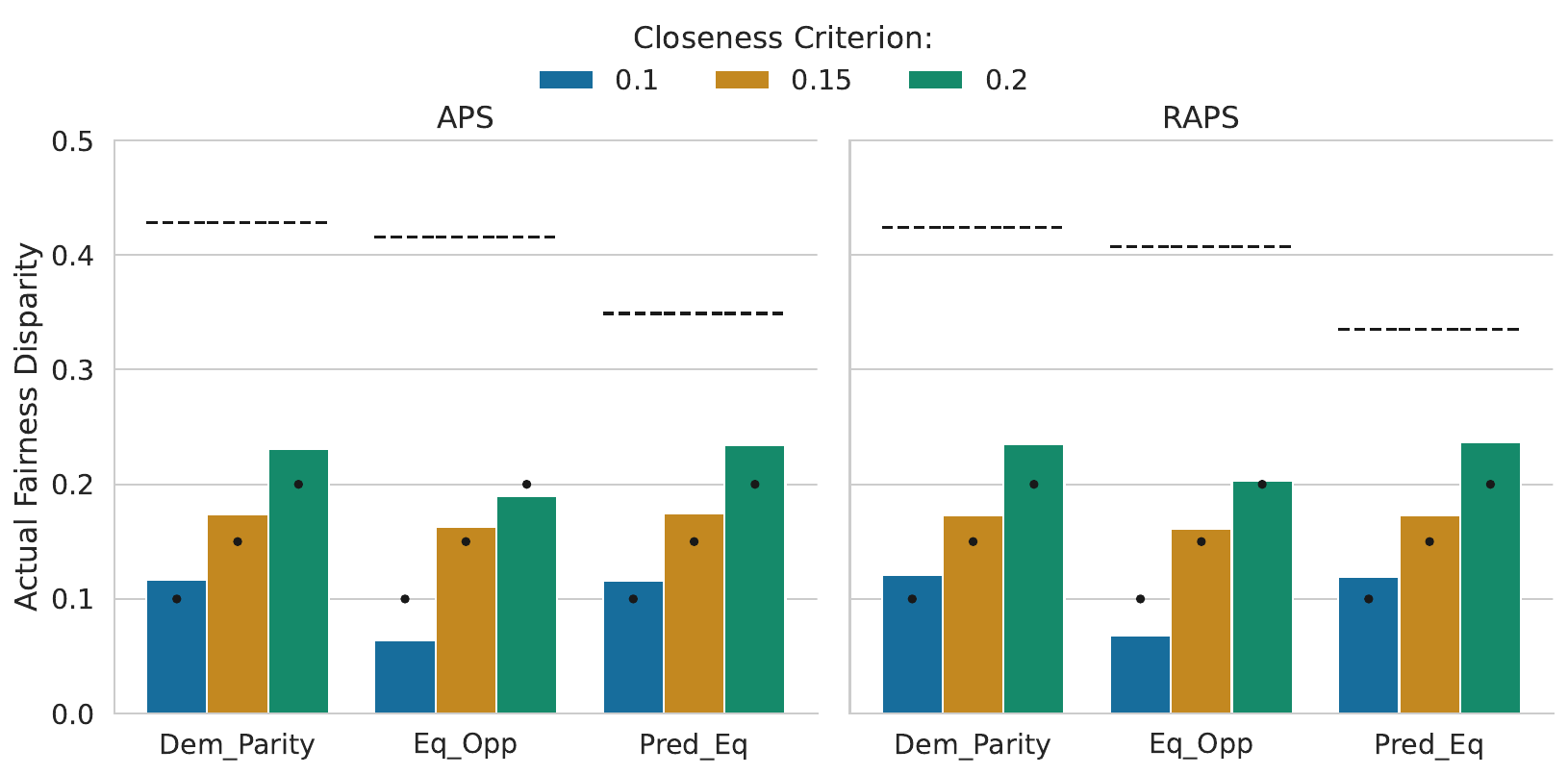}
    \end{subfigure}
    \caption{\textbf{Enhanced Privacy.}}
    
    \end{subfigure}
    \caption{\label{fig:acs_education_continental_small}\textbf{ACSEducation, Continental Small, Interval Bounds.} The top row demonstrates the efficiency of FedCF when using the continental version of ACSEducation, and its fairness disparity on the bottom row. Compared to Figure \ref{fig:acs_education_small_appendix}, the efficiencies improved (particularly for Equal Opportunity using RAPS), due to increased covariate representation for all sensitive groups.}
\end{figure}

\begin{figure}[ht!]
    \centering
    \begin{subfigure}{0.475\textwidth}
        \begin{subfigure}{\textwidth}
        \centering
            \includegraphics[width=\linewidth]{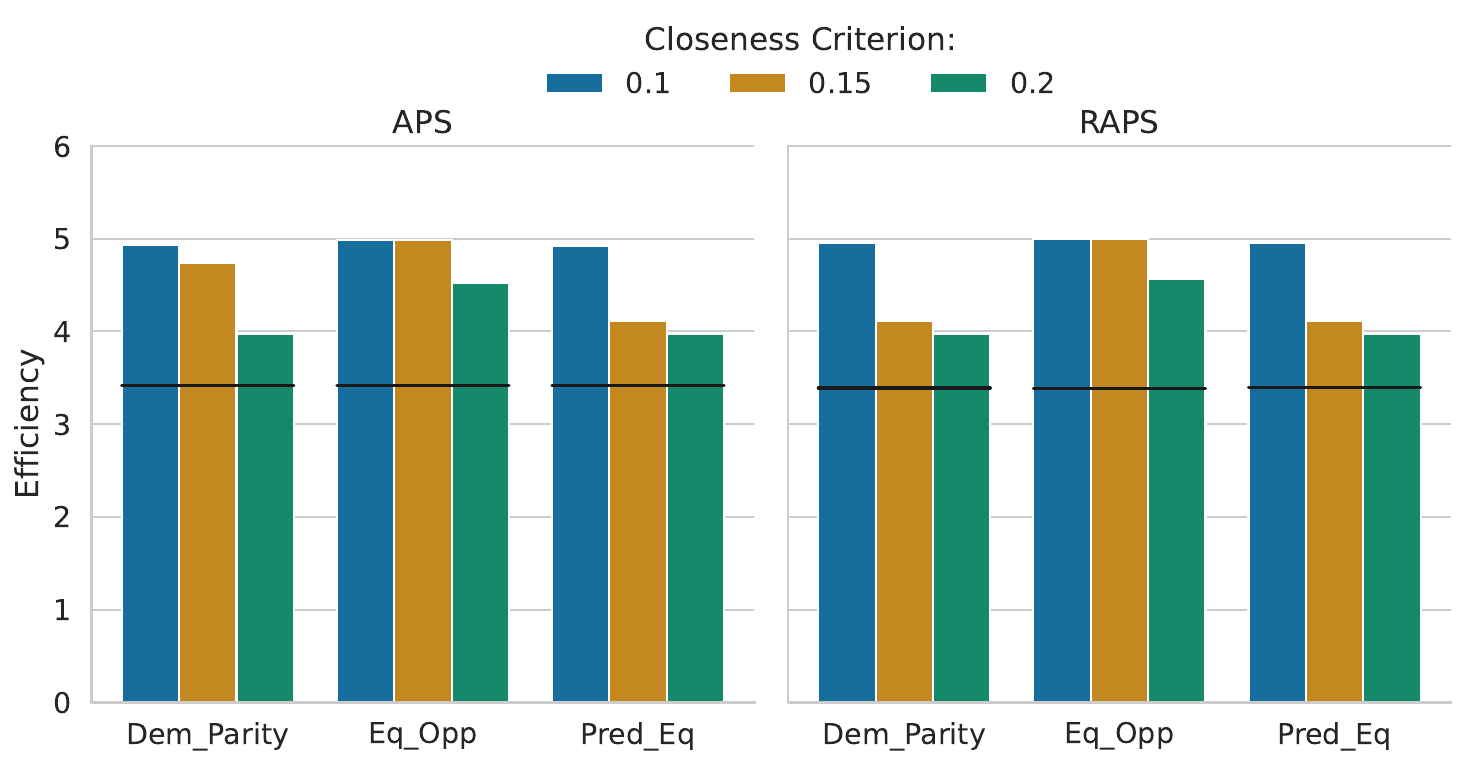}
        \end{subfigure}
        \begin{subfigure}{\textwidth}
            \centering
            \includegraphics[width=\linewidth]{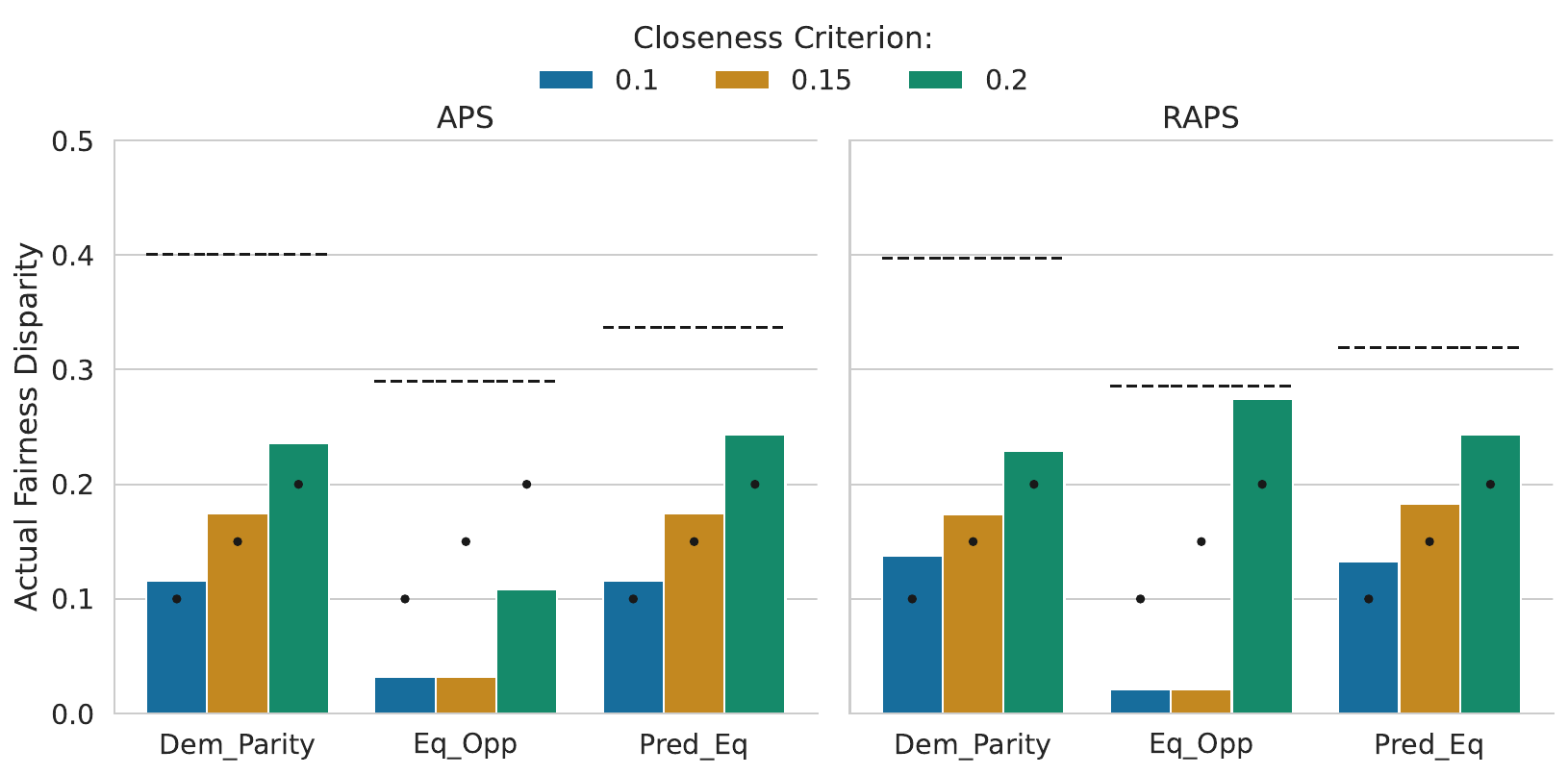}
        \end{subfigure}
        \caption{\textbf{Communication Efficient.}}
    \end{subfigure}%
    \hfill
    \begin{subfigure}{0.475\textwidth}
    \begin{subfigure}{\textwidth}
        \centering
        \includegraphics[width=\linewidth]{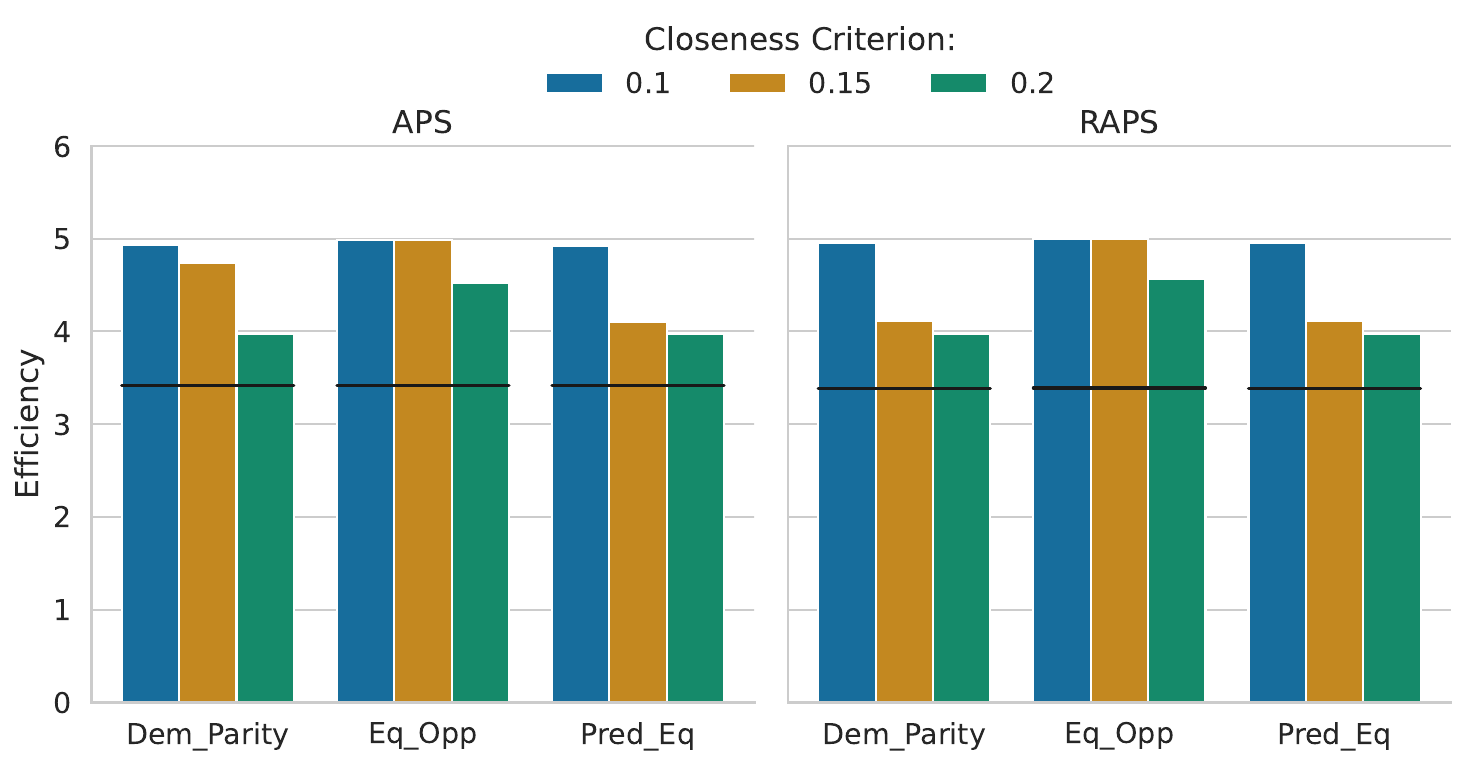}
    \end{subfigure}
    \begin{subfigure}{\textwidth}
        \centering
        \includegraphics[width=\linewidth]{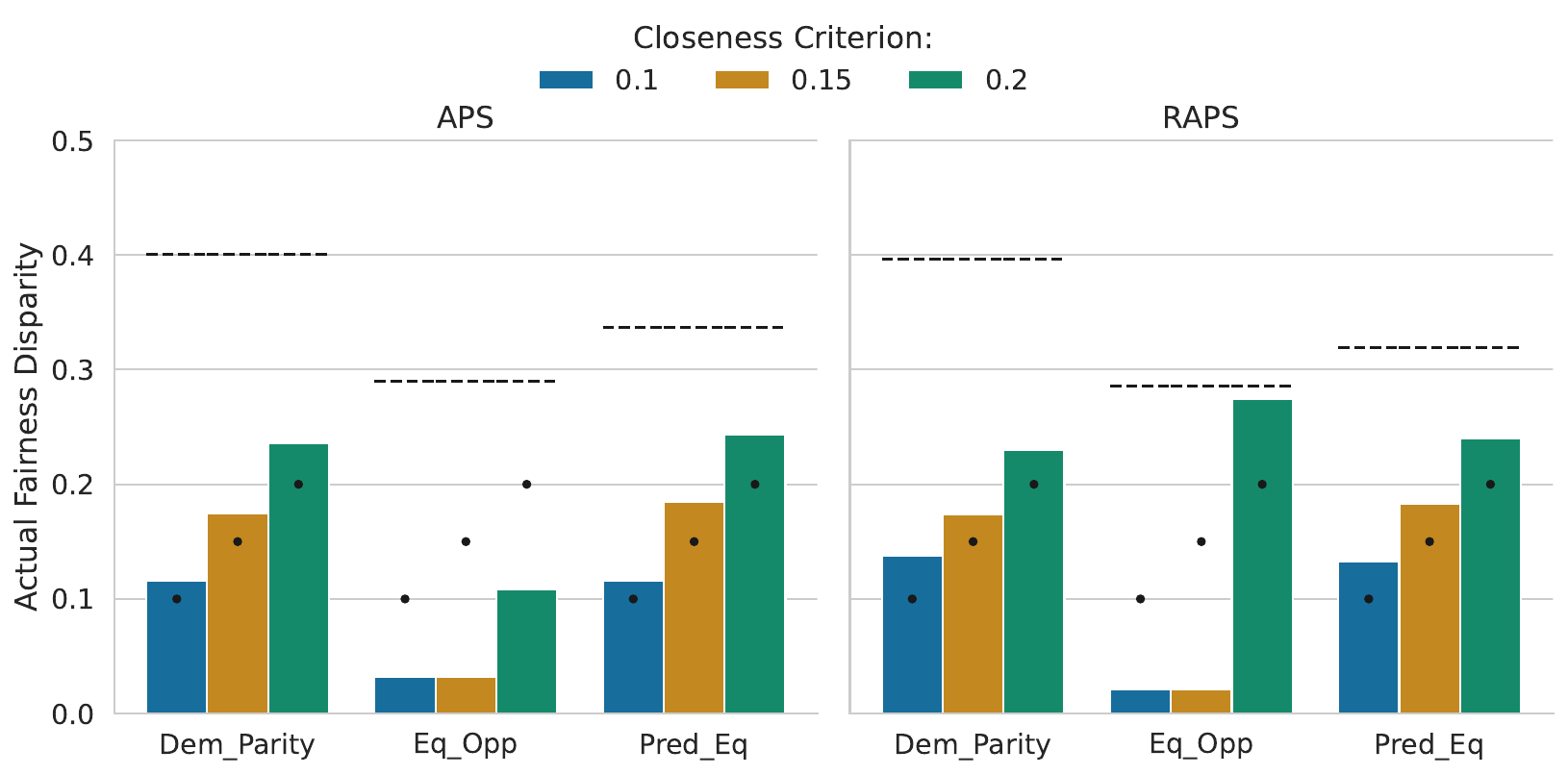}
    \end{subfigure}
    \caption{\textbf{Enhanced Privacy.}}
    \end{subfigure}
    \caption{\textbf{ACSEducation, Small, Point Estimates} The plots in the top row indicate the efficiency with the corresponding fairness disparity plots in the bottom row. We observe that using point estimates will result in a similar or lower efficiency than using the interval bounds approach in Figure \ref{fig:acs_education_small_appendix}, at the cost of a similar or higher fairness violation. Because the MLE does not provide a finite sample guarantee, the violation can exceed the desired closeness criterion, but will be lower than the baseline federated conformal predictor.}
\end{figure}

\begin{figure}[ht!]
    \centering
    \begin{subfigure}{0.475\textwidth}
        \begin{subfigure}{\textwidth}
        \centering
            \includegraphics[width=\linewidth]{figures/ACSEducation_small_True_Communication_Efficient_efficiency.pdf}
        \end{subfigure}
        \begin{subfigure}{\textwidth}
            \centering
            \includegraphics[width=\linewidth]{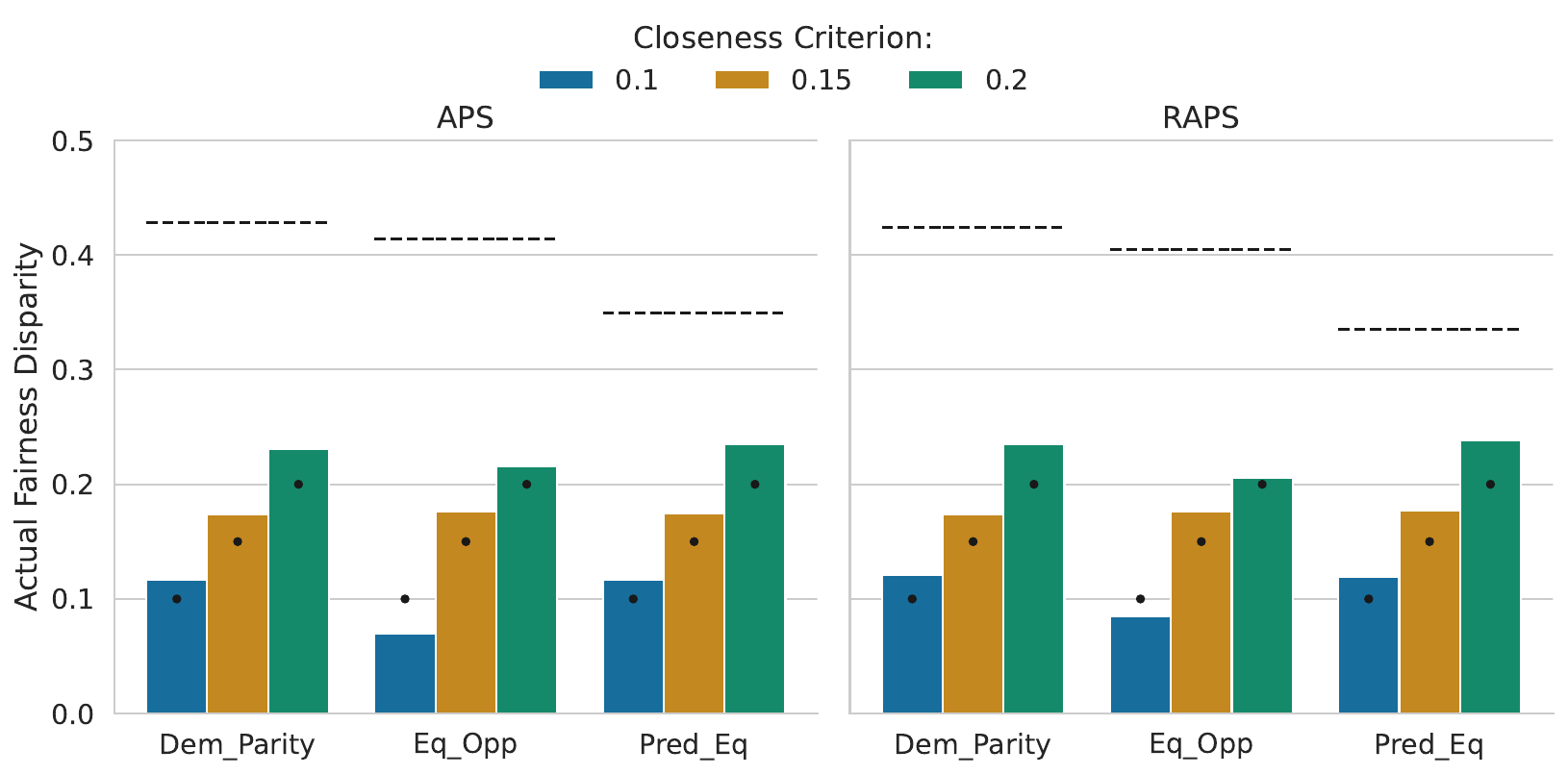}
        \end{subfigure}
        \caption{\textbf{Communication Efficient.}}
    \end{subfigure}%
    \hfill
    \begin{subfigure}{0.475\textwidth}
    \begin{subfigure}{\textwidth}
        \centering
        \includegraphics[width=\linewidth]{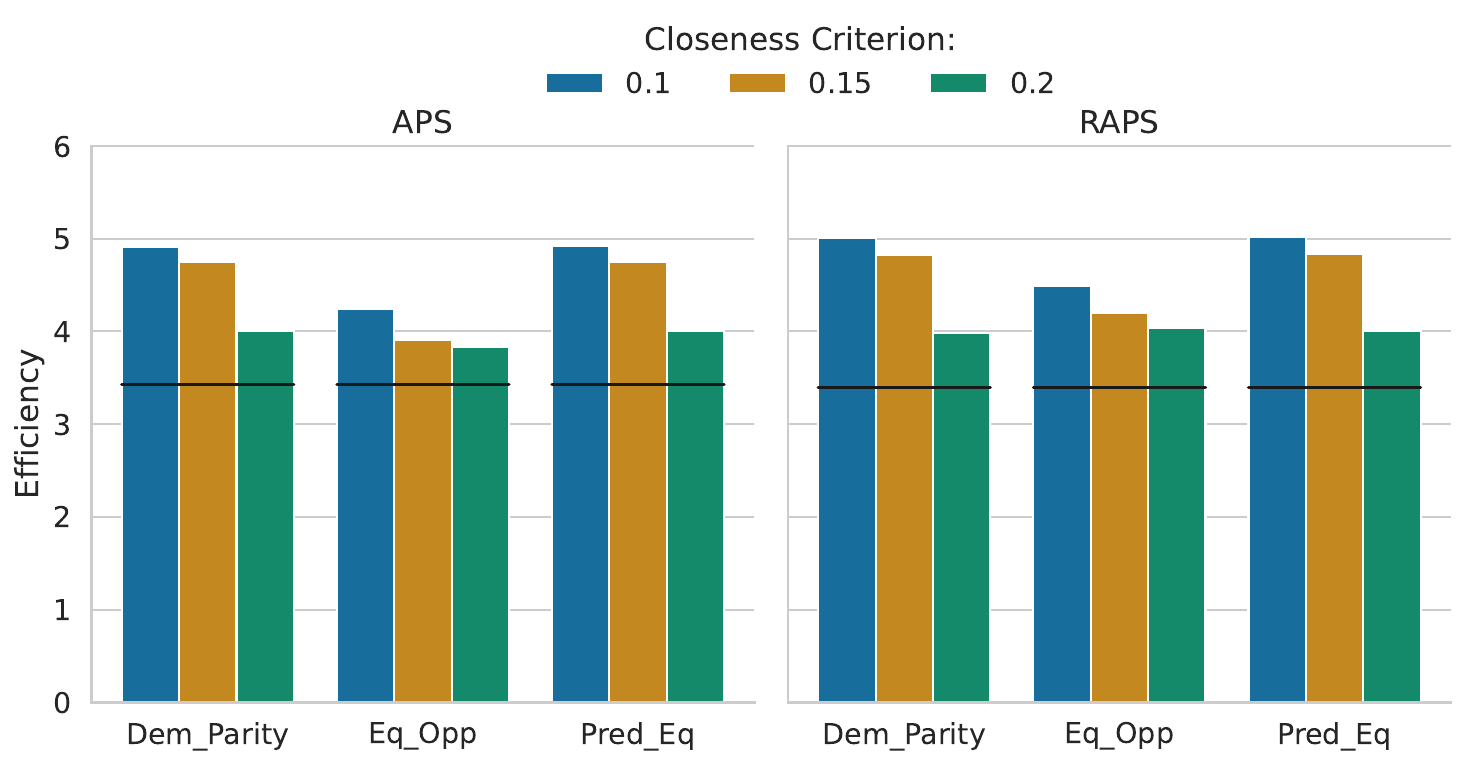}
    \end{subfigure}
    \begin{subfigure}{\textwidth}
        \centering
        \includegraphics[width=\linewidth]{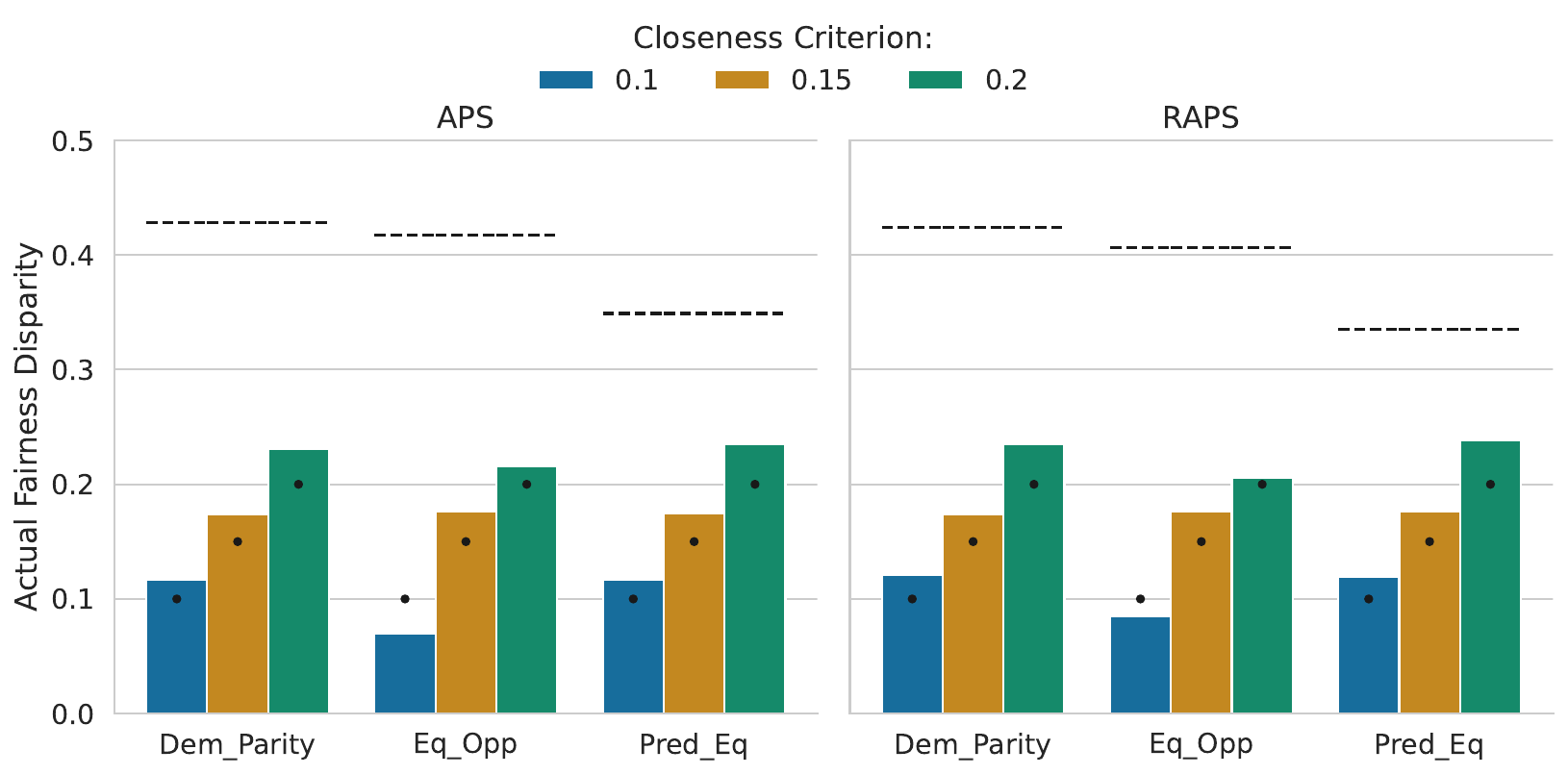}
    \end{subfigure}
    \caption{\textbf{Enhanced Privacy.}}
    \end{subfigure}
    \caption{\textbf{ACSEducation, Continental Small, Point Estimates.} The plots in the top row indicate the efficiency with the corresponding fairness disparity plots in the bottom row. We observe that using point estimates will result in a similar or lower efficiency than using the interval bounds approach in Figure \ref{fig:acs_education_continental_small}, at the cost of a similar or higher fairness violation. Because the MLE does not provide a finite sample guarantee, the violation can exceed the desired closeness criterion, but will be lower than the baseline federated conformal predictor.}
\end{figure}

\clearpage
\subsection{Impact of different sensitive attributes for Pokec-\{n,z\}}

\begin{figure}[ht!]
    \centering
    \begin{subfigure}{0.8\textwidth}
        \begin{subfigure}{\textwidth}
        \centering
            \includegraphics[width=\linewidth]{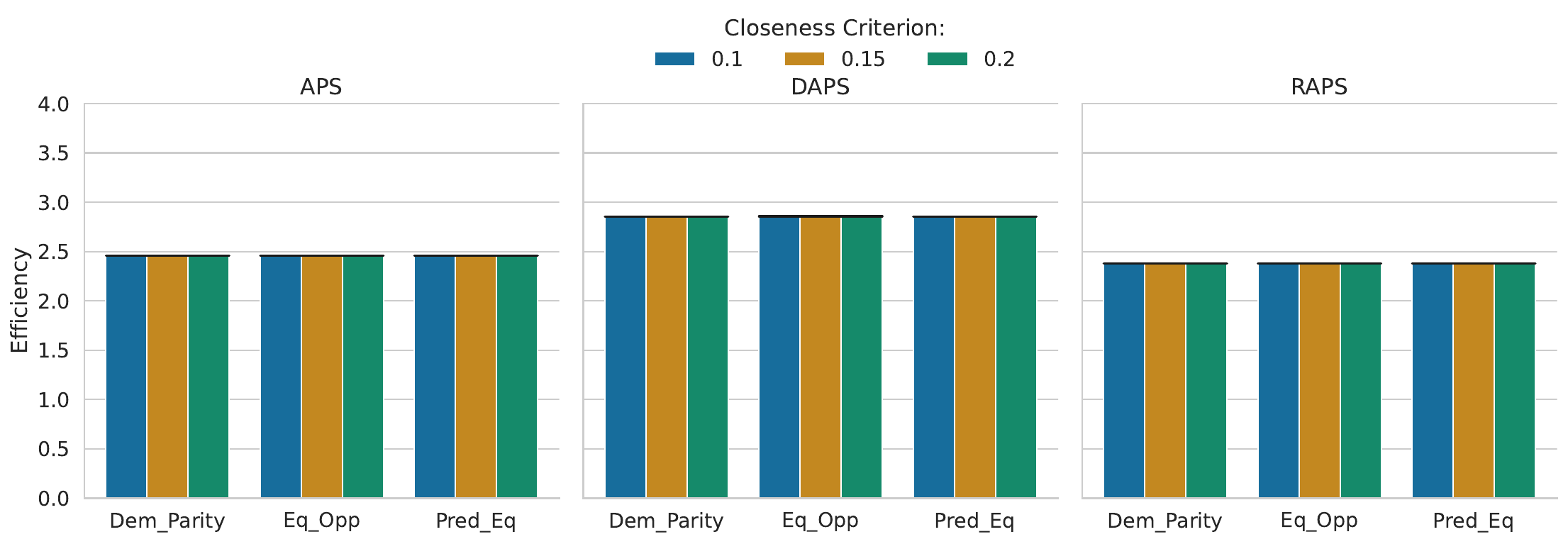}
        \end{subfigure}
        \begin{subfigure}{\textwidth}
            \centering
            \includegraphics[width=\linewidth]{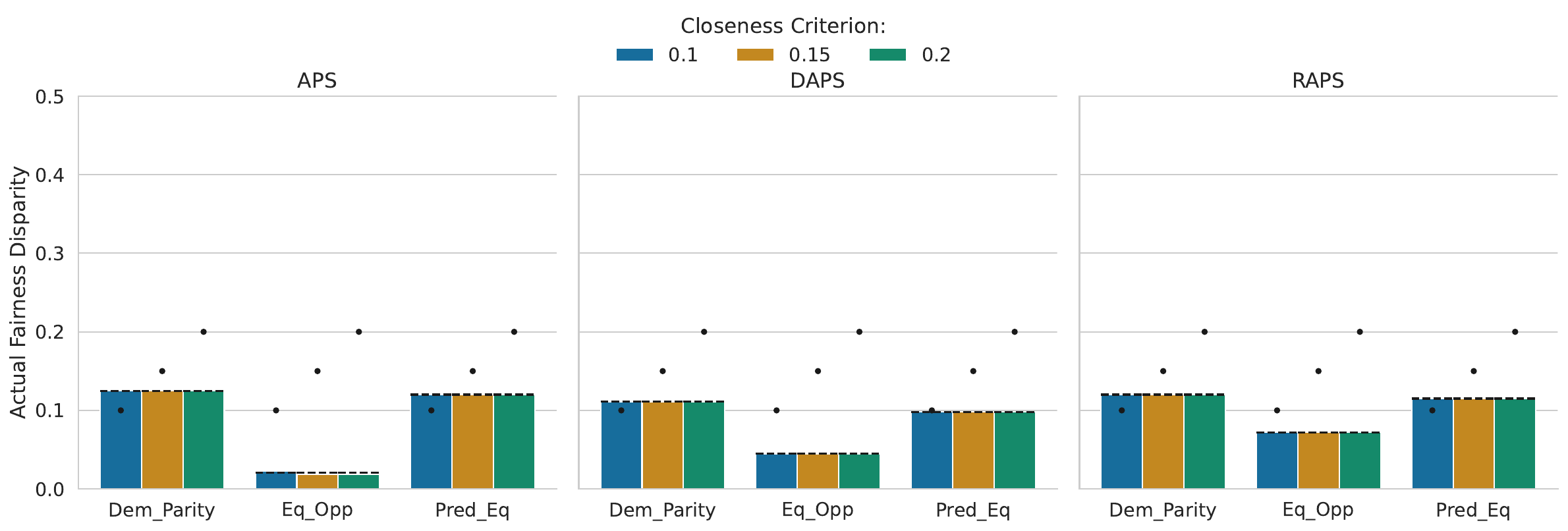}
        \end{subfigure}
        \caption{\textbf{Communication Efficient.}}
    \end{subfigure}
    \begin{subfigure}{0.8\textwidth}
        \begin{subfigure}{\textwidth}
            \centering
            \includegraphics[width=\linewidth]{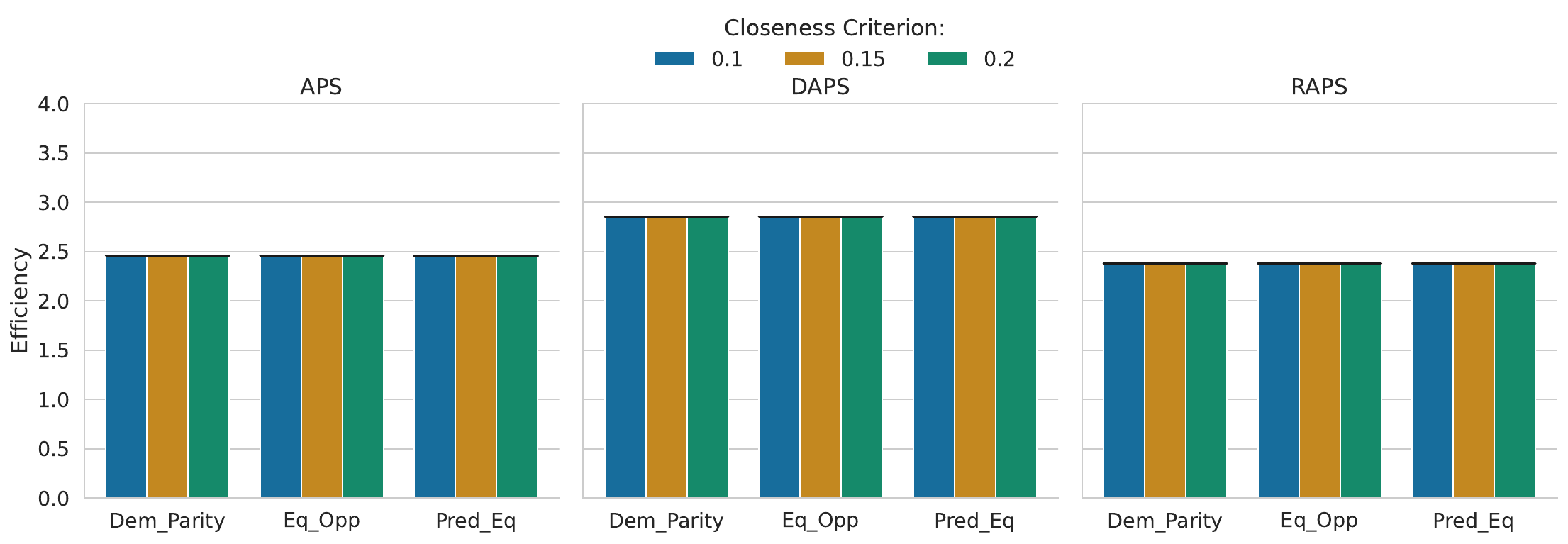}
        \end{subfigure}
        \begin{subfigure}{\textwidth}
            \centering
            \includegraphics[width=\linewidth]{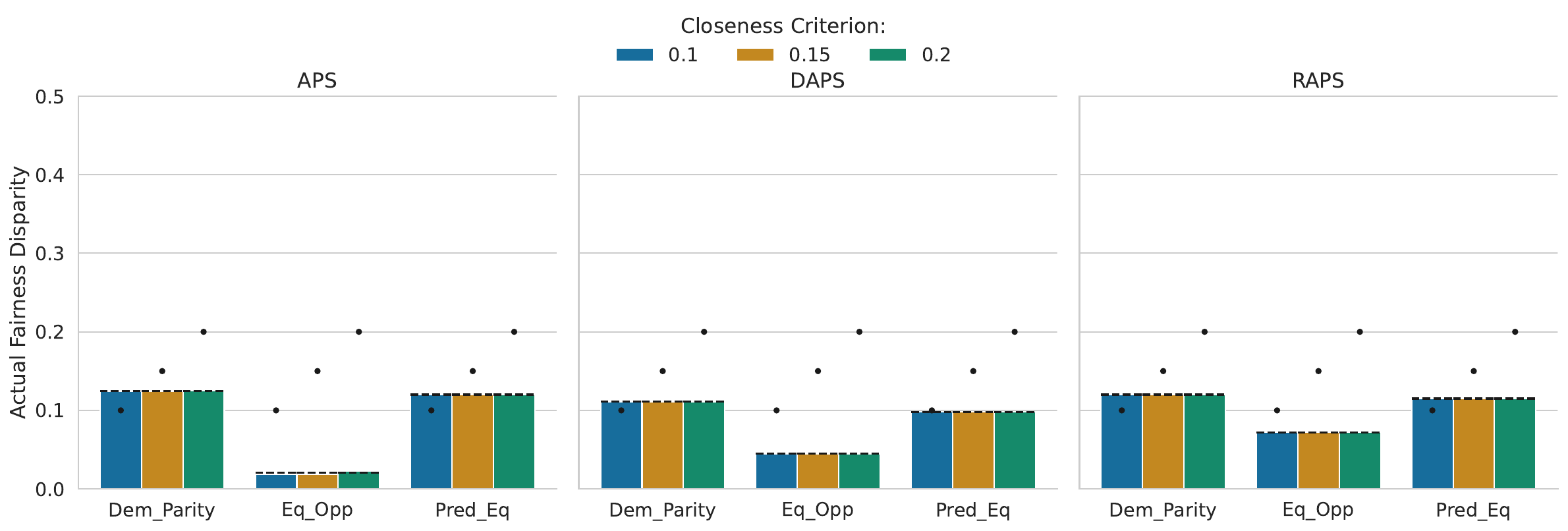}
        \end{subfigure}
        \caption{\textbf{Enhanced Privacy.}}
    \end{subfigure}
    \caption{\textbf{Pokec-\{n,z\}, \textit{gender}.} For each plot (a) and (b), the top plots are for the efficiency, and the bottom plots are for the fairness disparity. The baseline disparity is within the closeness criterion, so we see no changes in efficiency when using FedCF. This is the case when using either the \textit{communication efficient} and \textit{enhanced privacy} protocols.}
\end{figure}

\begin{figure}[ht!]
    \centering
    \begin{subfigure}{0.8\textwidth}
        \begin{subfigure}{\textwidth}
        \centering
            \includegraphics[width=\linewidth]{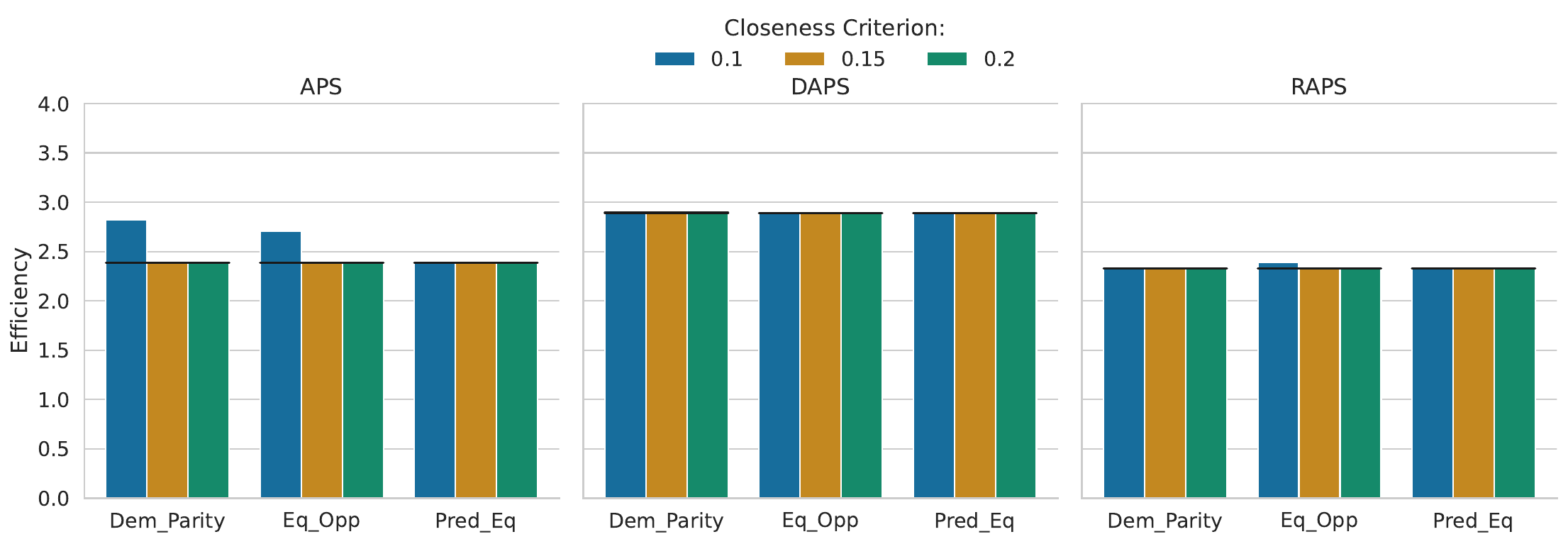}
        \end{subfigure}
        \begin{subfigure}{\textwidth}
            \centering
            \includegraphics[width=\linewidth]{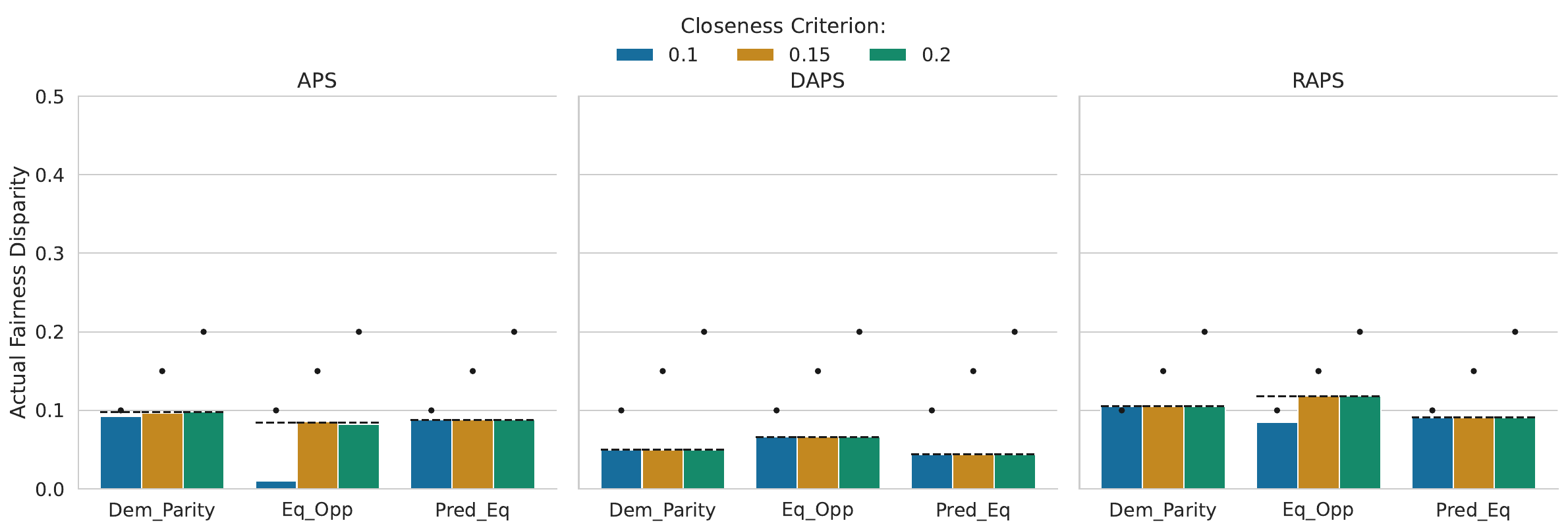}
        \end{subfigure}
        \caption{\textbf{Communication Efficient.}}
    \end{subfigure}
    \begin{subfigure}{0.8\textwidth}
        \begin{subfigure}{\textwidth}
            \centering
            \includegraphics[width=\linewidth]{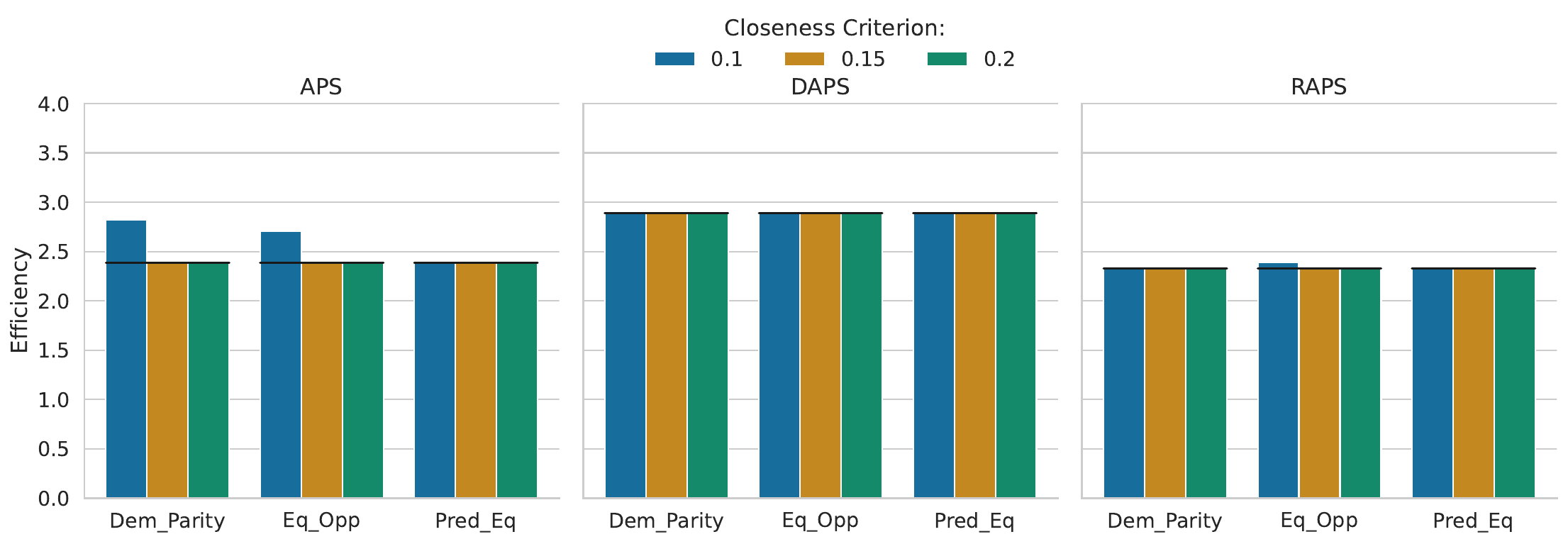}
        \end{subfigure}
        \begin{subfigure}{\textwidth}
            \centering
            \includegraphics[width=\linewidth]{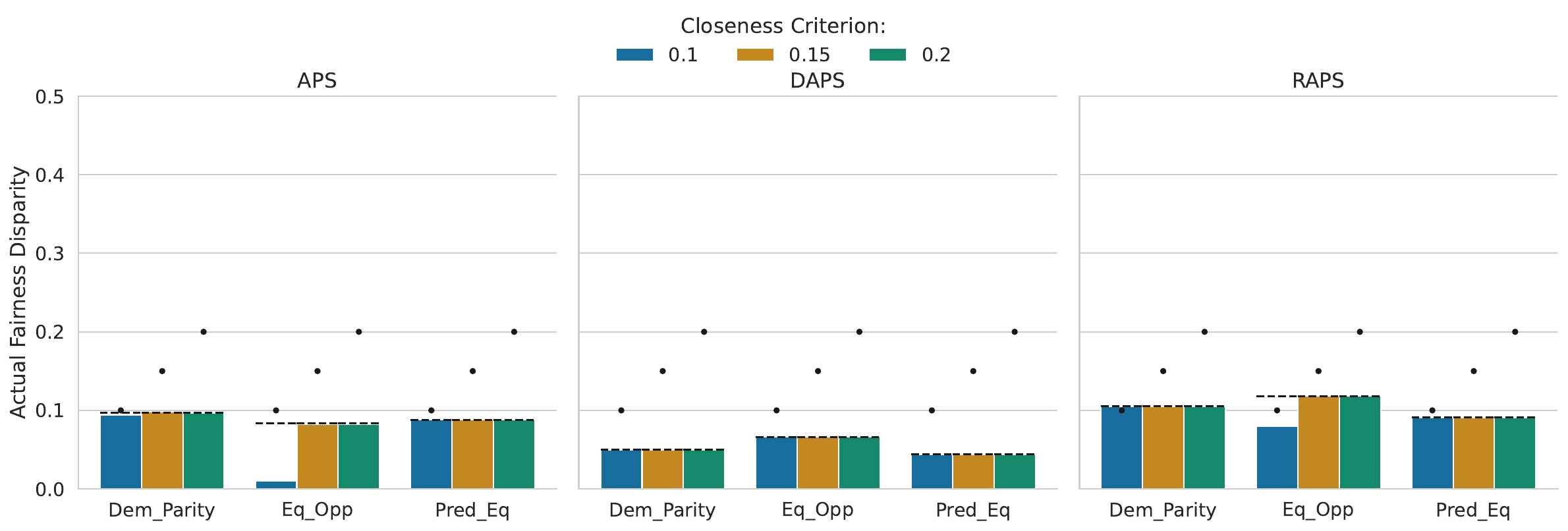}
        \end{subfigure}
        \caption{\textbf{Enhanced Privacy.}}
    \end{subfigure}
    \caption{\textbf{Pokec-\{n,z\}, \textit{region}.} For each plot (a) and (b), the top plots are for the efficiency, and the bottom plots are for the fairness disparity. Note that while the baseline disparity is within the closeness criterion for the test set, the finite-sample guarantee from using the interval bounds ensures FedCF looks for a better threshold, resulting in a smaller violation with a small cost to efficiency. This is the case when using either the \textit{communication efficient} and \textit{enhanced privacy} protocols.}
\end{figure}

\begin{figure}[ht!]
    \centering
    \begin{subfigure}{0.8\textwidth}
        \begin{subfigure}{\textwidth}
        \centering
            \includegraphics[width=\linewidth]{figures/Pokec_region_gender_2_clients_False_Communication_Efficient_efficiency.pdf}
        \end{subfigure}
        \begin{subfigure}{\textwidth}
            \centering
            \includegraphics[width=\linewidth]{figures/Pokec_region_gender_2_clients_False_Communication_Efficient_violation.pdf}
        \end{subfigure}
        \caption{\textbf{Communication Efficient.}}
    \end{subfigure}
    \begin{subfigure}{0.8\textwidth}
        \begin{subfigure}{\textwidth}
            \centering
            \includegraphics[width=\linewidth]{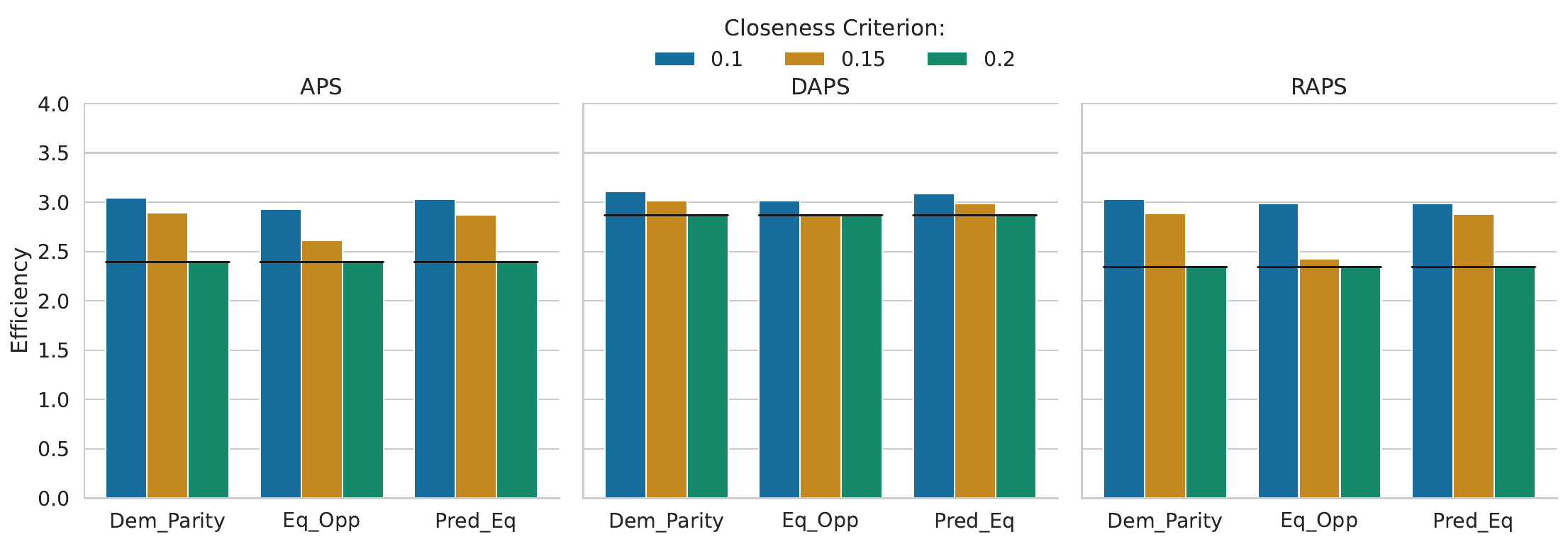}
        \end{subfigure}
        \begin{subfigure}{\textwidth}
            \centering
            \includegraphics[width=\linewidth]{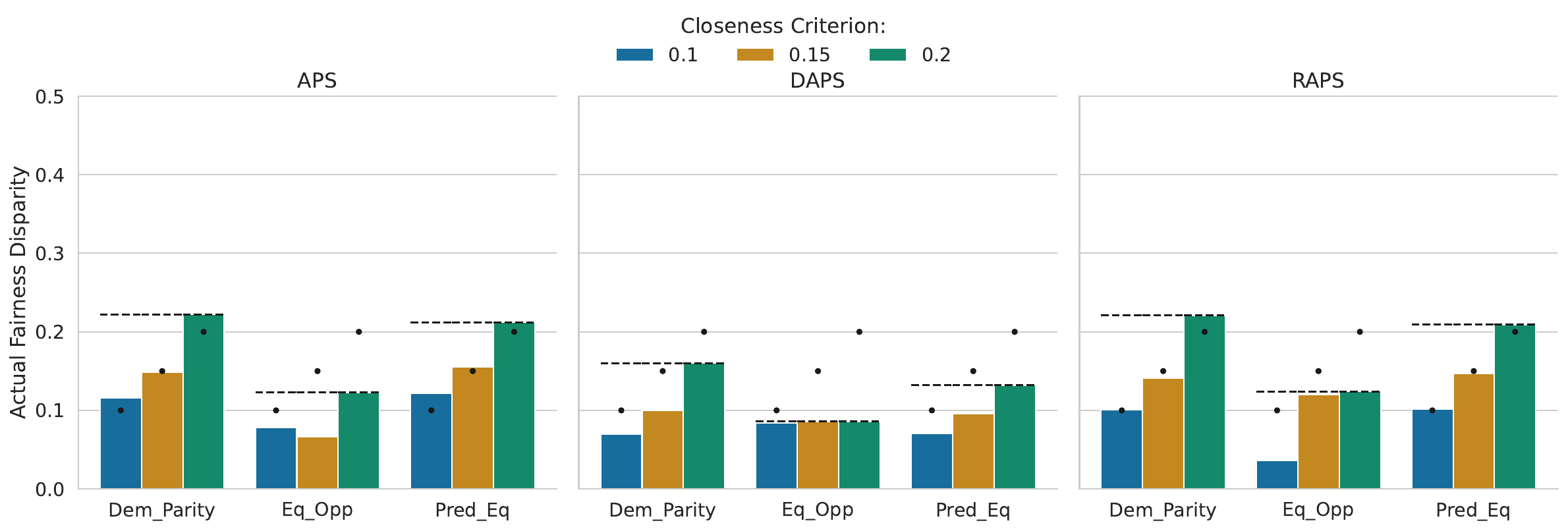}
        \end{subfigure}
        \caption{\textbf{Enhanced Privacy.}}
    \end{subfigure}
    \caption{\textbf{Pokec-\{n,z\}, \textit{region} and \textit{gender}.} For each plot (a) and (b), the top plots are for the efficiency, and the bottom plots are for the fairness disparity. In the case of intersectional fairness, since there are more groups, the violation will be worse than considering a single sensitive attribute. We observe that in all cases, FedCF produces a threshold that satisfies the closeness criterion, at a slight cost to efficiency. This is the case when using either the \textit{communication efficient} and \textit{enhanced privacy} protocols.}
\end{figure}
\clearpage
\section*{NeurIPS Paper Checklist}

\begin{enumerate}

\item {\bf Claims}
    \item[] Question: Do the main claims made in the abstract and introduction accurately reflect the paper's contributions and scope?
    \item[] Answer: \answerYes{} 
    \item[] Justification: {We provide the theoretical justification and implementation details in Section \ref{sec:methods} and empirical results in Section \ref{sec:experiments}. Further theory and proofs are available in Appendix \ref{app:fed_cf:descent_analysis}, \ref{app:fed_cf:proofs}, and \ref{app:additional_theory} with more results in Appendix \ref{app:fedcf:more_results}.}

\item {\bf Limitations}
    \item[] Question: Does the paper discuss the limitations of the work performed by the authors?
    \item[] Answer: \answerYes{} 
    \item[] Justification: We have a Limitations paragraph in the conclusion.

\item {\bf Theory assumptions and proofs}
    \item[] Question: For each theoretical result, does the paper provide the full set of assumptions and a complete (and correct) proof?
    \item[] Answer: \answerYes{} 
    \item[] Justification: {All theorems that are presented in Section \ref{sec:methods} have proofs in Appendix \ref{app:fed_cf:proofs} and \ref{app:additional_theory}, with assumptions explicilty detailed.}

    \item {\bf Experimental result reproducibility}
    \item[] Question: Does the paper fully disclose all the information needed to reproduce the main experimental results of the paper to the extent that it affects the main claims and/or conclusions of the paper (regardless of whether the code and data are provided or not)?
    \item[] Answer: \answerYes{} 
    \item[] Justification: {All details needed to reproduce results are presented in Section \ref{sec:experiments} and Appendix \ref{app:fed_cf:experimental_details}. The code used for the experiments is available in the supplemental material.}

\item {\bf Open access to data and code}
    \item[] Question: Does the paper provide open access to the data and code, with sufficient instructions to faithfully reproduce the main experimental results, as described in supplemental material?
    \item[] Answer: \answerYes{} 
    \item[] Justification: {The code is provided in the supplementary material, while all datasets are available publicly (details on procuring data can be found in the code).}

\item {\bf Experimental setting/details}
    \item[] Question: Does the paper specify all the training and test details (e.g., data splits, hyperparameters, how they were chosen, type of optimizer) necessary to understand the results?
    \item[] Answer: \answerYes{} 
    \item[] Justification: {This information is provided in Section \ref{sec:experiments} and Appendix \ref{app:fed_cf:experimental_details}.}

\item {\bf Experiment statistical significance}
    \item[] Question: Does the paper report error bars suitably and correctly defined or other appropriate information about the statistical significance of the experiments?
    \item[] Answer: \answerNo{} 
    \item[] Justification: {We ran the experiments with different seeds for random data splitting of the calibration and test sets, and found the deviations to be too small to be legibly plotted in the figures.}

\item {\bf Experiments compute resources}
    \item[] Question: For each experiment, does the paper provide sufficient information on the computer resources (type of compute workers, memory, time of execution) needed to reproduce the experiments?
    \item[] Answer: \answerYes{} 
    \item[] Justification: We include this information in Appendix~\ref{app:fed_cf:experimental_details} in Section \ref{par:fed_cf:compute_resources}.
    
\item {\bf Code of ethics}
    \item[] Question: Does the research conducted in the paper conform, in every respect, with the NeurIPS Code of Ethics \url{https://neurips.cc/public/EthicsGuidelines}?
    \item[] Answer: \answerYes{} 
    \item[] Justification: {The research adheres to the code of ethics.}

\item {\bf Broader impacts}
    \item[] Question: Does the paper discuss both potential positive societal impacts and negative societal impacts of the work performed?
    \item[] Answer: \answerYes{} 
    \item[] Justification: {We discuss the impacts in Appendix \ref{app:fedcf:impact_statement}.}
    
\item {\bf Safeguards}
    \item[] Question: Does the paper describe safeguards that have been put in place for responsible release of data or models that have a high risk for misuse (e.g., pre-trained language models, image generators, or scraped datasets)?
    \item[] Answer: \answerNA{} 
    \item[] Justification: {The paper and described methods do not have such risks.}

\item {\bf Licenses for existing assets}
    \item[] Question: Are the creators or original owners of assets (e.g., code, data, models), used in the paper, properly credited and are the license and terms of use explicitly mentioned and properly respected?
    \item[] Answer: \answerYes{} 
    \item[] Justification: {We provide citations and licenses for the datasets we used in Appendix \ref{app:datasets}, and adhered to their respective licenses.}

\item {\bf New assets}
    \item[] Question: Are new assets introduced in the paper well documented and is the documentation provided alongside the assets?
    \item[] Answer: \answerNA{} 
    \item[] Justification: {We have not released any assets. Once our code is made public, we will release it under the Apache-2.0 license.}

\item {\bf Crowdsourcing and research with human subjects}
    \item[] Question: For crowdsourcing experiments and research with human subjects, does the paper include the full text of instructions given to participants and screenshots, if applicable, as well as details about compensation (if any)? 
    \item[] Answer: \answerNA{} 
    \item[] Justification: N/A

\item {\bf Institutional review board (IRB) approvals or equivalent for research with human subjects}
    \item[] Question: Does the paper describe potential risks incurred by study participants, whether such risks were disclosed to the subjects, and whether Institutional Review Board (IRB) approvals (or an equivalent approval/review based on the requirements of your country or institution) were obtained?
    \item[] Answer: \answerNA{} 
    \item[] Justification: N/A

\item {\bf Declaration of LLM usage}
    \item[] Question: Does the paper describe the usage of LLMs if it is an important, original, or non-standard component of the core methods in this research? Note that if the LLM is used only for writing, editing, or formatting purposes and does \emph{not} impact the core methodology, scientific rigor, or originality of the research, declaration is not required.
    \item[] Answer: \answerNA{} 
    \item[] Justification: {LLMs were \textbf{not} used during development.}

\end{enumerate}

\end{document}